%% file: draft.tex
\documentclass[11pt]{article}
\usepackage[margin=1in]{geometry}

\usepackage{preamble}
\usepackage{verbatim}
\usepackage{marginnote}
\usepackage{url}
\title{Kalman Filtering with Adversarial Corruptions}
\author{Sitan Chen\thanks{Email: \texttt{sitanc@berkeley.edu} This work was supported in part by NSF Award 2103300, NSF CAREER Award CCF-1453261, NSF Large CCF-1565235 and Ankur Moitra's ONR Young Investigator Award. Part of this work was completed while visiting the Simons Institute for the Theory of Computing.} \\
UC Berkeley
 \and 
Frederic Koehler\thanks{Email: \texttt{fkoehler@berkeley.edu}. This work was supported in part by NSF CAREER Award CCF-1453261, NSF Large CCF-1565235, Ankur Moitra's ONR Young Investigator Award, and E. Mossel's Vannevar Bush Faculty Fellowship ONR-N00014-20-1-2826. Part of this work was completed while visiting the Simons Institute for the Theory of Computing.} \\
UC Berkeley
 \\\and
Ankur Moitra\thanks{Email: \texttt{moitra@mit.edu} This work was
supported in part by a Microsoft Trustworthy AI Grant, NSF CAREER Award CCF-1453261, NSF Large CCF1565235, a David and Lucile Packard Fellowship and an ONR Young Investigator
Award.}\\
MIT
\and 
Morris Yau\thanks{Email: \texttt{morrisy@mit.edu}}\\
MIT}
\date{\today}

\newcommand{\polylog}{\mathsf{polylog}}
\newcommand{\OPT}{\mathsf{OPT}}
\newcommand{\epsgeo}{\epsilon_{\mathsf{geo}}}
\newcommand{\pE}{\wt{\mathbb{E}}}

\renewcommand{\t}{t}
\renewcommand{\epsilon}{\varepsilon}
\newcommand{\avgT}{\frac{1}{T}\sum^{T-1}_{i=0}}

\begin{document}

\maketitle
\thispagestyle{empty}

\begin{abstract}
    \normalsize
    Here we revisit the classic problem of linear quadratic estimation, i.e. estimating the trajectory of a linear dynamical system from noisy measurements. The celebrated Kalman filter gives an optimal estimator when the measurement noise is Gaussian, but is widely known to break down when one deviates from this assumption, e.g. when the noise is heavy-tailed. Many ad hoc heuristics have been employed in practice for dealing with outliers. In a pioneering work, Schick and Mitter \cite{schick1989robust,schick1994robust} gave provable guarantees when the measurement noise is a known infinitesimal perturbation of a Gaussian and raised the important question of whether one can get similar guarantees for large and unknown perturbations.
    
    In this work we give a truly robust filter: we give the first strong provable guarantees for linear quadratic estimation when even a \emph{constant} fraction of measurements have been adversarially corrupted. This framework can model heavy-tailed and even non-stationary noise processes.  
    Our algorithm robustifies the Kalman filter in the sense that it competes with the optimal algorithm that knows the locations of the corruptions. Our work is in a challenging Bayesian setting where the number of measurements scales with the complexity of what we need to estimate. Moreover, in linear dynamical systems past information decays over time. We develop a suite of new techniques to robustly extract information across different time steps and over varying time scales.
\end{abstract}

\clearpage

\tableofcontents
\addtocontents{toc}{\protect\thispagestyle{empty}}
\pagenumbering{gobble}

\clearpage
\pagenumbering{arabic} 

\input{intro}

\input{overview}
\input{prelims}

\input{TlogT}
\input{TloglogT}

\input{twostage}

\bibliographystyle{alpha}
\bibliography{biblio}

\appendix
\input{wiener}

\input{examples}

\input{defer}

\end{document}

%% file: intro.tex
\section{Introduction}

\subsection{Background}

In this work, we revisit the classic problem of linear quadratic estimation, i.e. estimating the trajectory of a linear dynamical system from noisy measurements. First we review the setup:

\begin{enumerate}
    \item[(1)] There are known matrices $A\in\R^{d\times d}$ and $B\in\R^{m\times d}$, and there is an unknown initial state $x^*_0\in\R^d$ drawn from $\mathcal{N}(0, R^2\, \Id)$. 
    \item[(2)] The trajectory $\brc{x^*_0,\ldots,x^*_{T-1}}$ and the observations $\brc{y^*_0,\ldots,y^*_{T-1}}$ are generated according to the following model:
    \begin{align*}
    	x^*_i &= Ax^*_{i-1} + w^*_i  \\
    	y^*_i &= Bx^*_i + v^*_i
    \end{align*}
\noindent where $\brc{w^*_i}$ and $\brc{v^*_i}$ are \emph{process noise} and \emph{observation noise} and are vectors in dimensions $d$ and $m$ respectively. Furthermore $\brc{w^*_i}$ and $\brc{v^*_i}$ are all independently drawn from $\mathcal{N}(0, \sigma^2\, \Id)$ and $\mathcal{N}(0, \tau^2 \,\Id)$ respectively.
    
\end{enumerate}

The goal is to estimate the trajectory from the observations in either an offline or online sense, and to minimize the sum of squares of the error. The celebrated Kalman smoother and Kalman filter solve these two problems optimally. The main idea is that when the initialization and noise distributions are all Gaussian, at any time the posterior distribution on the trajectory given the observations is a Gaussian process. It suffices to estimate the mean of the posterior distribution and this can be done by finding the least squares solution to a structured regression problem that depends on $A$ and $B$ and the observations. It turns out that there is an even more compact formulation in terms of equations that define the Kalman filter. 

The Kalman filter \cite{kalman1959general, kalman1961new} is one of the crowning achievements in control theory. It has wide-ranging applications in robotics, navigation, signal processing and econometrics. It is also a key building block in algorithms for estimating the model parameters of a linear dynamical system, as well as in change-point detection and building optimal controllers. Famously, the Kalman filter provided navigation estimates that helped guide the landing of the Apollo 11 lunar module in the Sea of Tranquility. Rudolf Kalman was awarded the National Medal of Science in 2008. 

{\em But how brittle is the Kalman filter to assumptions of Gaussianity?} This is by no means a new question. If we relax the distributional assumptions but instead restrict the disturbances $\brc{w^*_i}$ and $\brc{v^*_i}$ to have bounded norm, then the minimax optimal filter can be found by dynamic programming. The solution is called the $H_\infty$ filter and has wide-ranging applications in its own right \cite{kailath2000linear}. However in many settings the assumption that the disturbances are bounded in norm is not reasonable either. In such cases, it has often been repeated that the Kalman filter can fail catastrophically. This is an unfortunate state-of-affairs because it means even though we can find the optimal filter when the noise is nice and Gaussian, it can break down badly with even a single badly outlying observation. 

There are many natural heuristics that have been employed for dealing with outliers. However we are aware of only one work that gives rigorous guarantees in the presence of outliers. In 1994, Schick and Mitter \cite{schick1994robust} studied a model where the observation noise is drawn from a distribution
\begin{equation}
(1-\eta)\mathcal{N}(0, \sigma^2 \Id) + \eta \mathcal{H} \label{eq:sm}
\end{equation}
where $\mathcal{H}$ is a sufficiently regular distribution, but is allowed to be heavy-tailed. This is essentially the Huber contamination model. They derived provable guarantees but under a number of strong assumptions: First, they assumed that the distribution $\mathcal{H}$ is known to the filter designer. Thus the filter can use information about $\mathcal{H}$ to correct for the exact ways in which the noise is non-Gaussian. Second, their guarantees are asymptotic in nature in the sense that they only hold in the limit as $\eta \rightarrow 0$. As Schick and Mitter \cite{schick1994robust} discussed, as $T \rightarrow \infty$ for fixed $\eta > 0$, the number of outliers (i.e. timesteps where the observation noise is sampled from $\mathcal{H}$) goes to infinity. However their analysis relies on the exponential asymptotic stability of the Kalman filter, whereby outside of a window around the current timestep, the influence of older observations is significantly attenuated. Thus as $\eta \rightarrow 0$ the number of outliers in any window can be bounded, even if the total number of outliers cannot. In fact, while their estimator is nonlinear, as is necessary for handling heavy-tailed noise, it is constructed through a bank of Kalman filters. Each filter ignores one of the observations, assuming that it is the lone outlier. The filters are then combined in a natural way. 

\paragraph{In this work we seek a truly robust filter.} We want to build a filter without any knowledge of $\mathcal{H}$. Moreover we want the outliers to not merely be sampled from a heavy-tailed distribution, but allow for their values to be chosen, possibly adaptively, by an unbounded adversary. For instance, this captures situations where the process generating the outliers is non-stationary. Moreover we want to prove guarantees that hold for fixed noise rates, as opposed to guarantees that only hold in the $\eta\to 0$ limit. Finally, we will want our filter to be a robustification of the Kalman filter itself in the sense that when $\eta = 0$ we want to achieve the same exact guarantees as before. Thus our filter competes with the Kalman filter in a strong sense, but gracefully degrades in performance as we move away from the precise distributional assumptions underlying the classic theory.

\subsection{Our Results}

\paragraph{The Corruption Model.} Here we explain our corruption model  (see Section~\ref{sec:generative} for a formal description). Let $0 \le \eta < 1/2$ be the \emph{corruption fraction}. We will assume that for every timestep $i$, with probability $\eta$ the observation $y^*_i$ falls under the control of an all-powerful adversary. The adversary is allowed to replace all of the observations under his control with arbitrary values. Now let $\brc{y_0,\ldots,y_{T-1}}$ denote the sequence of observations that the learner ultimately receives. We emphasize that exactly which of these have been corrupted is unknown to the learner.

Note that the corrupted timesteps are \emph{randomly chosen}, just like in the Huber contamination model. As a special case this captures the setting \eqref{eq:sm} studied by Schick and Mitter. Moreover, because the adversary gets to coordinate his corruptions, our model also allows dependencies and captures situations where the observation noise is non-stationary over time. As we show in Appendix~\ref{subsec:srw}, in the stronger corruption model where the adversary gets to choose which timesteps to corrupt, there are strong impossibility results. Thus our corruption model seems to be one of the strongest where we can still hope for meaningful guarantees. 

\paragraph{The Objective.}
As we discussed above, the Kalman filter can be thought of as computing the mean of the posterior distribution on the trajectory given the observations so far. When the noise is non-Gaussian, it is no longer true that the posterior distribution is itself a Gaussian process. It can be much more complex. {\em So how can we even define an optimization problem that generalizes that solved by Kalman filtering, if the posterior is non-Gaussian and, even worse, depends on $\mathcal{H}$ which is unknown to the filter and possibly changing over time?} 

Our main idea is to compete with a strong oracle that knows which measurements are corrupted and which are not. Let $a^*_i\in\brc{0,1}$ denote the indicator variable for whether round $i$ is clean in the sense that its measurement error came from a Gaussian, rather than coming from $\mathcal{H}$ or being chosen by an adversary. When the $a^*_i$'s are known to the filter, the optimal estimator in a Bayesian sense is to estimate the mean of the posterior using only information from the uncorrupted observations. This leads us to the following objective:
\begin{equation} 
L(\wh{x}) = \frac{1}{T} \left(\sum_{i = 0}^{T-1} (a^*_i \|B \wh{x}_i - y_i\|^2/\tau^2 + \norm{\wh{w}_i}^2/\sigma^2) + \|\wh{x}_0\|^2/R^2\right)  \label{eqn:loss}
\end{equation}
where the steps $\wh{w}_i$ are defined in terms of the trajectory $\wh{x}_i$, i.e. $\wh{x}_i = A\wh{x}_{i-1} + \wh{w}_i$ for all $i > 0$. 
This is the {\em clean posterior negative log likelihood}.
The best possible error we can achieve is given by 
\begin{equation}
    \OPT = \min_{\brc{\wh{x}_i}} L(\wh{x}) 
    \label{eq:OPT_def_pre}
\end{equation}
which is attained by the maximum a posteriori (MAP) estimator. This is the same as the posterior mean and can be explicitly computed from a Kalman smoother that knows the observations that come from clean rounds, i.e. $y_i$ for the rounds $i$ for which $a^*_i = 1$. For general $\brc{\wh{x}_i}$, we refer to $L(\wh{x}) - \OPT$ as the \emph{excess risk}.

We take a moment to explain the differences in our approach compared to the usual approaches in algorithmic robust statistics, for example in robust mean estimation \cite{huber1964robust,diakonikolas2017being,lai2016agnostic}. Usually in a robust statistics problem, the objective of the estimator is to recover the \emph{ground truth}, e.g. the true mean in the example of mean estimation, and the goal in the robust setting is to recover the ground truth perfectly as $\eta \to 0$ while optimizing the dependence on the corruption fraction $\eta$. This makes sense for mean estimation because when $\eta = 0$, the mean can always be estimated consistently by taking more samples. It does not make as much sense in the case of Kalman filtering: even without corruptions, we only ever get one observation per timestep, so we cannot hope to recover the ground truth trajectory arbitrarily well. Uncertainty about the true trajectory is unavoidable because the complexity of the trajectory grows with the number of observations we get to make, and also because information about the past is washed away over time. For this reason, we need to pick our measure of success carefully.
From a Bayesian perspective, the clean posterior mean represents the \emph{best possible} estimate we can make of the ground truth given the clean observations {\em given additional information about which rounds have been corrupted}. Thus it gives a natural way to quantify the distance of our estimator from optimality.

\paragraph{Main Results.} We show how to design an estimator which is robust to corruptions and competes with the optimum in the clean posterior log likelihood. In the clean case $\eta = 0$, our estimator gets asymptotically optimal posterior log likelihood, including the correct constant factor, and so its guarantee essentially matches the posterior mean/Kalman smoother. In fact, with high probability our estimator will be exactly the same as the Kalman smoother (see Remark~\ref{rmk:uncorrupted-case}).
\begin{theorem}[Informal, see Theorem~\ref{thm:main_TloglogT}]\label{thm:sos_informal}
    For $\eta \le 0.49$, and for a uniformly stable and completely observable dynamical system,\footnote{Uniform stability and complete observability are standard assumptions from the control theory literature, which we introduce in Section~\ref{sec:ctrlbasics}.} there is a polynomial-time algorithm which takes as input the corrupted observations $\brc{y_i}$ and outputs a trajectory $\brc{\wh{x}_i}$ achieving excess risk \begin{equation}
        L(\wh{x}) - \mathsf{OPT} \le C_d \,\eta\log(1/\eta)\left(m + d(\sigma^2/\tau^2)\log\log T\right) + o(1), \label{eq:Lbound}
    \end{equation} with high probability, where the steps $\wh{w}_i$ are defined by $\wh{x}_i = A\wh{x}_{i-1} +\wh{w}_i$, where $C_d$ is a constant which is polynomially bounded in $\log d$ and the parameters of the system (see Section~\ref{sec:generative}), and where $o(1)$ is a quantity vanishing polynomially in $T$.
\end{theorem}
\noindent 
To compare, we remind the reader of the performance of the obvious baselines: the Kalman filter can have unbounded error if there is even a single corruption, and oblivious outlier removal makes error $\Theta(T)$, see Appendix~\ref{subsec:srw}.
Recall that $\sigma^2$ is the variance of the process noise and $\tau^2$ is the variance of the observation noise. The dependence on the ratio $\sigma^2/\tau^2$ is unavoidable (Appendix~\ref{sec:sigbytau}), and the dependence on the dimensions $d$ and $m$ in \eqref{eq:Lbound} of Theorem~\ref{thm:sos_informal} is also unavoidable (Appendix~\ref{app:ddepend}). 


Notably, we can obtain strong provable guarantees for any $\eta < 1/2$ (we wrote $\eta < 0.49$ above only to simplify the statement). Thus our estimator has an information-theoretically optimal breakdown point. Also, our result can handle the case where the eigenvalues of $A$ are on or near the unit circle, e.g. $A = \Id$, a situation where the system is marginally stable but not strictly stable (see e.g. \cite{simchowitz2018learning} for discussion of this terminology). This is an important distinction, because when the eigenvalues of $A$ are all small, a relatively simple method based on truncating the Kalman filter can work (we prove this in Appendix~\ref{apdx:wiener}), but the performance of such a heuristic will degrade badly as the eigenvalues approach the unit circle, whereas the algorithm of Theorem~\ref{thm:sos_informal} will still work (see Section~\ref{sec:overview} and Appendix~\ref{apdx:examples} for more discussion)

The above result works in an offline setting. But what happens when our measurements come in an online fashion and we need to estimate the position at time $i$ from only the observations up to that point? In this sense we want the filter to be {\em causal}. 
Fortunately, at a small loss in our overall guarantees, we are able to make our approach online too. We need to change the definition of the objective slightly to handle the fact that predictions are made online: for the online case, we look at the suboptimality in predicting the next state compared to the oracle Kalman filter.
\begin{theorem}[Informal version of Theorem~\ref{thm:online}]\label{thm:online_informal}
There exists a polynomial time and causal estimator $\hat{\hat{x}}_{i + 1 | i}$ which satisfies the following guarantee with probability at least $1 - \delta$. With $\hat{x}_{i + 1 | x}$ denoting the output of the oracle Kalman filter, we have
\begin{equation*} 
\frac{1}{T}\sum_{i = 1}^T \|\hat{\hat{x}}_{i + 1 | i} - x^*_{i + 1}\|^2
-  \frac{1}{T} \sum_{i = 1}^T \|\hat{x}_{i + 1 | i} - x^*_{i + 1}\|^2 \le \frac{r''}{1 - \delta} (\eta + O(\sqrt{\eta \log(2/\delta)/T} + \log(2/\delta)/T))
\end{equation*}
where $r'' = C_d\, \eta \left(m + d(\sigma^2/\tau^2)\log T\right) + o(1)$ and $C_d$ is a constant which is polylogarithmic in $d,m$ and allowed to depend on other system parameters. 
\end{theorem}
Our analysis combines our convex programming approach with standard properties of the Kalman filter, such as its exponential stability, building on an idea of Schick and Mitter \cite{schick1994robust}.

\subsection{Further Related Work}
\label{subsec:related}
\noindent\textbf{Robust Statistics and Sum-of-Squares.} Our main algorithm for robust filtering is based on the Sum of Squares hierarchy \cite{parrilo2000structured}, which has broad applications to both control theory (see e.g. \cite{prajna2005sostools,henrion2005positive}) and to algorithmic robust statistics (see e.g. \cite{klivans2018efficient, hopkins2018mixture, kothari2018robust,bakshi2020outlier,diakonikolas2020robustly,cherapanamjeri2020optimal}). It also builds upon a line of recent work in algorithmic robust statistics using both SoS and non-SoS methods (see e.g. \cite{diakonikolas2017being,diakonikolas2018list,lai2016agnostic,charikar2017learning,chinot2020erm,dalalyan2019outlier}). One of the techniques we use, introducing a positive semi-definiteness constraint utilizing matrix concentration bounds (see Technical Overview), is conceptually related to the main technical ingredient in \cite{chen2020online} for robust linear regression. The recent work \cite{banks2021local} also considers robust statistics in a Bayesian setting, namely community detection in the stochastic block model: Their method recovers the optimal detection threshold as the corruption level goes to zero, but not the optimal performance of the Bayes estimator (they only establish results for detection and not recovery).

\vspace{0.5em}
\noindent\textbf{Practical Approaches.} Algorithms based on minimizing loss functions that are less sensitive to outliers, e.g. the Huber loss, have been widely applied in practical works on robust filtering (see e.g. \cite{karlgaard2015nonlinear,mattingley2010real}) but thus far lack strong theoretical guarantees. 
This is because, unlike in other contexts like robust/heavy-tailed mean estimation and regression \cite{huber1964robust,huber1973robust,chinot2020erm,mendelson2018learning}, the observations come from a non-stationary generative model and the length of the trajectory and number of observations are linked, so the existing proof techniques do not apply. 
There are a number of similar ad hoc methods which have been applied in practice to handle outliers, especially by downweighting the Kalman filter updates when the innovation exceeds some threshold. 

\vspace{0.5em}
\noindent\textbf{Other Notions of Robustness for Kalman Filtering} A popular variant of the standard Kalman filtering setup is to allow the variables $y_i$ to be dropped independently with some probability $p$ (analogous to our $\eta$); see e.g. \cite{sinopoli2004kalman,mo2008characterization,mo2011kalman}. This is like having an adversarial corruption model where the location of corruptions are known (and hence can be ignored); in this version of the problem, the Kalman filter remains optimal and the focus has been on understanding aspects of its behavior, e.g. in unstable systems. 

We mention some other works in the direction of making the Kalman filter more robust. Some take the approach of assuming a particular parametric model for the noise that is non-Gaussian, such as a Student's $t$-distribution \cite{roth2013student} or L\'{e}vy distribution \cite{sornette2001kalman}, with the corresponding caveat that their results are limited to the model they assume. A major line of work in the control theory and filtering literature deals with robustness to uncertainty in the parameters of the system (e.g. $A$ and $B$), which is quite different from the problem of handling outliers. For example \cite{poor1980robust}, studies a version of this question in the context of the Wiener filter: the goal there is to choose the best linear filter given the uncertainty set, and so it cannot handle outlying observations where a nonlinear filter is required. Similarly, in recent works such as \cite{tsiamis2020online,simchowitz2020improper} the authors study online methods for control theory problems where the goal is to compete with the best policy in a certain class (e.g. compete with the Kalman filter when the system is unknown, or compete with the best of the $H_{\infty}$ controller and the $H_2$ controller without knowing the disturbance model), but in our setting the only filters which perform well are nonlinear and so competing with a class of linear filters like the Kalman filter is not sufficient. 

\paragraph{Roadmap} 

In Section~\ref{sec:overview} we provide an overview of the key ingredients in our proof of Theorems~\ref{thm:sos_informal} and \ref{thm:online_informal}. In Section~\ref{sec:prelims} we give technical preliminaries including a formal description of the generative model we consider as well as sum-of-squares basics. In Section~\ref{sec:logT} we give an algorithm that recovers a confidence band around the true trajectory, and in Section~\ref{sec:TloglogT} we use this to obtain a refined estimate achieving the excess risk bound of Theorem~\ref{thm:sos_informal}. In Section~\ref{sec:online} we prove the online guarantee of Theorem~\ref{thm:online_informal}. In Appendix~\ref{apdx:wiener}, we show that a simple variant of the Kalman filter works when $A$ has eigenvalues in the interior of the unit disk. In Appendix~\ref{apdx:examples} we present a number of examples elucidating why various components of our final bound and our analysis are necessary.

%% file: overview.tex

\section{Technical Overview}\label{sec:overview}

In this section, we first review basic concepts from control theory like observability and stability that will play an important role in the subsequent discussion. We then give an overview of the main challenges of handling adversarial corruptions when estimating a linear dynamical system and explain the techniques we develop to obtain our main results.

\subsection{Control Theory Basics}
\label{sec:ctrlbasics}

\paragraph{Observability.} Before we describe how to estimate a linear trajectory from corrupted observations, we first review how to do so in the absence of corruptions. Note that without additional assumptions on $A$ and $B$, this may not be possible \emph{a priori}. For instance, $A$ might only act nontrivially on some subspace of $\R^d$, and the rows of $B$ might simply be completely orthogonal to this subspace, in which case we can't hope to recover the trajectory. 

In control theory, the standard way to ensure that the linear dynamical system at hand is not degenerate in this fashion is to assume it is \emph{observable}, as originally defined by Kalman in \cite{kalman1959general}. Formally, given a parameter $s\in\mathbb{N}$, define the \emph{observability matrix}
\begin{equation}
    \calO_s \triangleq \sum^{s-1}_{i=0} (A^i)^{\top}B^{\top}B A^i.
\end{equation}
To motivate this object, suppose momentarily that there were no observation or process noise, so that the trajectory is given by $x^*_i = Ax^*_{i-1}$ and the observations are given by $y^*_i = Bx^*_i$. Then our observations up to the $s$-th timestep are given by $y^*_i = BA^i x^*_0$ for $0\le i < s$. Now note that
\begin{equation}
    {x^*_0}^{\top}\calO_s x^*_0 = \sum^{s-1}_{i=0} \norm{BA^i x^*_0}^2 = \sum^s_{i=0} \norm{y^*_i}^2.
\end{equation}
In particular, if $\calO_s$ had nonzero kernel and $x^*_0$ were an element of this kernel, then all of the observations up to time $s$ would be zero, and we would get no information about $\brc{x^*_0,\ldots,x^*_{s-1}}$. Conversely, if $\calO_s$ were full rank, then one can recover $\brc{x^*_0,\ldots,x^*_{s-1}}$ given $\brc{y^*_0,\ldots,y^*_{s-1}}$ by solving the appropriate linear system. In other words, non-degeneracy of the observability matrix $\calO_s$ is a \emph{necessary and sufficient condition} for being able to recover the trajectory up to time $s$ from observations regardless of where the trajectory started.

More generally when there is observation and process noise, the natural quantitative analogue of non-degeneracy of $\calO_s$ is an upper bound on its condition number (see e.g. \cite{muller1972analysis,fetzer1975observability}):

\begin{assumption}[Complete observability\--- informal, see Assumption~\ref{assume:obs}] \label{assume:obs-no-noise}
	For some $s\in\mathbb{N}$, $\calO_s \triangleq \sum^{s-1}_{i=0} (A^i)^{\top}B^{\top}B A^i$ is well-conditioned.
\end{assumption}

\paragraph{Stability}
We will focus on models which satisfy the following weak stability assumption, often made in the control theory literature (including the work of Schick and Mitter \cite{schick1994robust}):
\begin{assumption}[Uniform stability\--- informal, see Assumption~\ref{assume:uniform}]\label{assume:uniform_overview}
    There is a constant $\rho \ge 1$ such that for any $i\in\mathbb{N}$, $\norm{A^i} \le \rho$ (here $\norm{\cdot}$ denotes the operator norm).
\end{assumption}
Intuitively, uniform stability ensures that the system initialized at any point will not eventually blow up at some time in the future. In contrast, if $A$ has an eigenvalue larger than one, the system is called \emph{explosive} or \emph{unstable} and the state will blow up at an exponential rate. Although Kalman filtering has also been studied in the case where $A$ is unstable, we know from the work on \emph{intermittent Kalman filtering} (see Section~\ref{subsec:related}) that even the oracle Kalman filter, which knows the location of the corruptions, will diverge if the corruption level in the unstable case is above some critical value \cite{mo2008characterization}. Since the setting we consider is strictly more difficult, there is no hope of closely tracking the trajectory in our setting.

\subsection{Our Techniques}
\label{sec:techniques}


\paragraph{Corruptions Degrade Observability.}

The first complication that arises in our setting is that corruptions can degrade the observability of the system. To see this, again consider the setup where there is no process or observation noise, but now some of the observations have been corrupted. We're essentially given a linear system $\brc{y_i = BA^i x^*_0}_{0<i<s}$ where some unknown subset $S_{\mathsf{bad}}\subseteq\brc{0,\ldots,s-1}$ of equations have been altered adversarially. If we knew $S_{\mathsf{bad}}$, then we could remove the corresponding equations and try solving for $x^*_0$ with the rest. Then the matrix that we need to be non-degenerate is no longer $\calO_s$ but rather 
\begin{equation}
		\calO'_s\triangleq \sum_{i\not\in S_{\mathsf{bad}}} (A^i)^{\top}B^{\top}B A^i. \label{eq:subO}
\end{equation}
Of course, $\calO_s$ being non-degenerate does not guarantee $\calO'_s$ is as well (see Appendix~\ref{subsec:switching}). One might wonder then whether Assumption~\ref{assume:obs-no-noise} must be significantly strengthened to ensure that the trajectory can be recovered from corrupted observations. As we illustrate in Appendix~\ref{subsec:switching}, this is indeed the case if the corruptions arrived at arbitrary timesteps. But if the corruptions arrive in a random fashion, we will demonstrate that no additional assumptions need to be made.

\paragraph{Corruptions Subsample the Observability Matrix.}

To get a sense for how this could be possible, note that if $S_{\mathsf{bad}}$ comprises a random $\eta$ fraction of the indices up to time $s$, then the expectation of $\calO'_s$ is exactly $(1 - \eta)\cdot \calO_s$. If we could argue that $\calO'_s$ also concentrates around its expectation, then because $\calO'_s$ is spectrally close to $(1 - \eta)\cdot \calO_s$ and $\calO_s$ is non-degenerate/well-conditioned by assumption, the corruptions don't actually impact the observability of the dynamical system in the presence of Huber contamination.

As the summands $(A^i)^{\top} B^{\top} B A^i$ are bounded in norm by Assumption~\ref{assume:uniform_overview}, we can carry through this matrix concentration argument as long as $s$ is sufficiently large. In Assumption~\ref{assume:obs-no-noise} however, we make no assumptions about how large $s$ is. Instead, we note that regardless of how large $s$ in Assumption~\ref{assume:obs-no-noise} is, by observing the system $t$ steps at a time rather than $s$ steps at a time for a large multiple $t$ of $s$, we find that Assumption~\ref{assume:obs-no-noise} also holds for $t$ in place of $s$.

More precisely, if we consider the observability matrix $\calO_t \triangleq \sum^{t-1}_{i=0} (A^i)^{\top}B^{\top} B A^i$ for some moderately large $t$, then 
$\calO_t\succeq \calO_s$, and one can also easily check that $\norm{\calO_t} \le (t\rho^2/s)\norm{\calO_s}$. In other words, the condition number of $\calO_t$ is at most $t\rho^2/s$ worse than that of $\calO_s$. So even if $s$ from Assumption~\ref{assume:obs-no-noise} is too small for $\calO'_s$ to concentrate sufficiently around its expectation, it would appear that we can simply take $t$ large enough that $\calO'_t$ concentrates sufficiently around its expectation.


\paragraph{Key Complication: Observable/Unobservable Subspaces.}

There is one essential wrinkle in the above argument: the fluctuations for matrix concentration for $\calO'_t$ are of order $\sqrt{t}$, so at some point they may exceed the smallest singular value of $\calO_t$. In particular, it would appear that because of Huber contamination, we lose all control over how well we can estimate the component of the trajectory that lives in the span of the small singular vectors of $\calO_t$, and so no matter what value of $t$ we choose, matrix concentration alone fails to show we can estimate the state successfully.


We now sketch a way to get around this problem. 
First note that in spite of the issue posed by the small singular vectors of $\calO_t$, the preceding discussion on matrix concentration does ensure that the projection of $\calO'_t$ to the \emph{large} singular vectors of $\calO_t$ is sufficiently nondegenerate with high probability (see Lemma~\ref{lem:anticonc}). One might therefore hope to be able to estimate the trajectory within this subspace. For this reason, we will refer to the span of the large singular vectors of $\calO_t$ as the \emph{observable subspace} and its orthogonal complement as the \emph{unobservable subspace}; denote projectors to these subspaces by $\Pi$ and $\Pi^{\perp}$ respectively. (The unobservable subspace should \emph{not} be thought of as some kind of ``invisible'' subspace which doesn't affect the observations; based on the discussion above, it represents directions which are difficult to recover locally based on partially corrupted observations.)

So what do we do about the unobservable subspace? Here is the key idea: while we do lose control of the trajectory's component inside the unobservable subspace \emph{within any fixed window of $t$ steps}, we can consolidate information \emph{across windows} to learn what goes on in the unobservable subspace. Before we can discuss how to implement this, we need to introduce our estimator.

\paragraph{An Inefficient Estimator.} 
To motivate the design of our estimator, we will construct a system of constraints capturing salient features of the model. We begin by introducing vector-valued variables $\brc{x_0,\ldots,x_{T-1}}$ corresponding to our estimate for the trajectory, as well as variables $\brc{w_1,\ldots,w_{T-1}}$ and $\brc{v_0,\ldots,v_{T-1}}$ corresponding to our estimates for the process and observation noise. We also introduce Boolean variables $a_0,...,a_{T-1} \in \{0,1\}$, where $a_i$ corresponds to our guess for whether the $i$-th observation was uncorrupted. Of course, the true values of the quantities that these variables represent are unknown to us, but there are a number of basic constraints they must satisfy. Firstly, because each observation is independently corrupted with probability $\eta$, we know by Chernoff that
\begin{equation}
    \frac{1}{T}\sum^{T-1}_{i=0} a_i \geq (1 - \eta) - o(1).
\end{equation}
Secondly, we know that the trajectory is given by a linear dynamical system, and in any uncorrupted timestep $i$, the observation $y_i$ is a noisy linear measurement of the trajectory at that time, so
\begin{equation}
    x_i = A x_{i - 1} + w_i \label{eq:follow_dynamics}
\end{equation}
\begin{equation}
  a_i (y_i - B x_{i - 1} - v_i) = 0. \label{eq:fit_clean_ys}  
\end{equation}
Additionally, we know the dynamics and observation noise is bounded with high probability:
\begin{equation}
    \|w_i\|^2 = O(\sigma^2 d), \quad \|v_i\| = O(\tau^2 m). \label{eq:noise_bound_informal}
\end{equation}
Thus far the constraints have been fairly straightforward. We now describe a key constraint capturing the preceding discussion on matrix concentration.
Recall that because the corruptions arrive in a random fashion, by matrix concentration the uncorrupted timesteps will ``subsample'' the observability matrix $\calO_t$ in each window. In other words, in every window $\brc{\ell t,\ldots,(\ell + 1)t - 1}$ of $t$ timesteps, the following spectral lower bound holds with high probability
\[\sum^{t - 1}_{j=0} (1 - a_{\ell t + j}) (A^j)^{\top} B^{\top} B A^j \preceq \eta\cdot \calO_{\t} + O(\sqrt{t}) \cdot \Id\]

Finally, with all of our constraints in place, it is straightforward to define an objective to optimize, in light of \eqref{eqn:loss}. We want to minimize the ``clean'' negative log-likelihood achieved by $\brc{x_i}$, where ``clean'' is defined with respect to the variables $\brc{a_i}$ instead of the true indicators $\brc{a^*_i}$:
\begin{equation} 
\min \frac{1}{T} \left(\sum_{i = 1}^T (a_i \|Bx_i - y_i\|^2/\tau^2 + \norm{w_i}^2/\sigma^2) + \|x_0\|^2/R^2\right).  \label{eqn:program_loss}
\end{equation}
Based on the discussion above, we might guess that $\|\Pi(x_i - x_i^*)\|^2$, the error from the observable subspace, can be bounded in the above program based on some kind of matrix concentration argument. It turns out this question itself is subtle, because we would need to rule out cancellations between the observable and unobservable subspace. However, even if we did argue that by itself, this is definitely not enough to make our ultimate objective \eqref{eqn:loss}, the clean negative posterior likelihood, small: for some of the clean observations ($i$ with $a^*_i = 1$), they will be dependent on the information in the unobservable subspace, and the true state $x_i^*$ can be very large in the unobservable subspace (see Appendix~\ref{sec:3dexample} for an illustrative example). The key difficulty to making the objective \eqref{eqn:loss} small will be to argue that $\|x_i - x^*_i\|$, i.e. the error in the \emph{whole} space, can be bounded. 

\paragraph{Decomposing the Error.} 
Roughly speaking, we will prove this by showing that the effect of errors in the past decays exponentially fast for our estimator. Intuitively, this argument will show that we do not make large errors in the unobservable subspace because the estimator will successfully propagate information from the past. Formally, we prove the following key inequality:
\begin{equation}\label{eq:error-contraction}
    \| x_{i+t} - x_{i+t}^*\|^2 \leq \frac{1}{2}\| x_{i} - x_{i}^*\|^2 + [\text{noise}],
\end{equation}
where here and throughout the rest of this overview, we use $[\text{noise}]$ to denote a small quantity that is polynomially bounded in $t$ and in the variance of the observation and process noise (see Lemma~\ref{lem:error_decay}).

The proof of \eqref{eq:error-contraction} starts with the following decomposition of the error:
\begin{equation} 
x_{i+t} - x_{i+t}^* = \left[(x_{i + t} - x_{i + t}^*) - A^t(x_i - x_i^*)\right] +  A^t \Pi(x_{i} - x_i^*) + A^t \Pi_{V_\perp}(x_{i} - x_i^*) \label{eqn:error-decomposition}
\end{equation}
The first term is the amount of new noise introduced between steps $i$ and $i + t$: it will be small because we are only taking $t$ steps, the process noise $\brc{w^*_i}$ is small, and the corresponding program variables $\brc{w_i}$ are also constrained to be small. 

The remaining two terms in \eqref{eqn:error-decomposition} account for the error propagated from the past: the second term $A^t \Pi(x_{i} - x_i^*)$ represents the error propagated from the observable error in the past, and the third term $A^t \Pi^{\perp}(x_{i} - x_i^*)$ represents the error propagated from the unobservable error in the past. We now show how both of those terms can be bounded, starting with the last of these terms (unobservable error).

\paragraph{Unobservable Error from the Past.}
We identify a simple but critical fact about observable linear dynamical systems: any vector in the unobservable subspace decreases in norm when it evolves forward by $t$ timesteps. More formally, we can show that for any vector $x\in\R^d$, 
\begin{equation}
    \norm{A^t\Pi^{\perp}x}^2 \le c\cdot \norm{\Pi^{\perp}x}^2, \label{eq:unobservable_decay}
\end{equation}
for some small constant $0 < c < 1$ provided $t$ is sufficiently large relative to $s$ (see Lemma~\ref{lem:unobservable-decay}).

In particular, by applying this to the vector given by the unobservable error in the past, we conclude that the error $\norm{A^t\Pi(x_i - x^*_i)}^2$ from propagating the past unobservable error $\norm{\Pi(x_i - x^*_i)}^2$ can be upper bounded by a small fraction of $\norm{\Pi(x_i - x^*_i)}^2$.


\begin{remark}
    We must caution the reader that \eqref{eq:unobservable_decay} does \emph{not} mean that the observable component any vector $x$ decays after $x$ evolves over $t$ timesteps, which would instead correspond to a bound of the form
    \begin{equation}
    		\norm{\Pi_{V_\perp} A^{t} x}^2 \le C\cdot \norm{\Pi_{V_\perp} x}^2. \label{eq:point_decay_informal}
    \end{equation} for some $C < 1$. 
    Of course, if \eqref{eq:point_decay_informal} were true, it would make life much easier: by taking $x$ to be any iterate in the trajectory, we would include that over time, the trajectory barely lives in the unobservable subspace at all! Unfortunately, this is not the case, as it is easy to construct linear dynamical systems 
    with a significant portion of the state in the unobservable subspace (see Appendix~\ref{subsec:unobservableunbounded}).
\end{remark}

\paragraph{Observable Error from the Past.}
We now discuss how to handle the error $\norm{A^t\Pi(x_i - x^*_i)}^2$ propagated from the observable error in the past. It is tempting to try the same approach as for the unobservable error here, namely to argue that $\norm{A^t\Pi(x_i - x^*_i)}^2$ is a small fraction of the observable error in the past. But this is too much to hope for: together with \eqref{eq:unobservable_decay} this would imply that $\norm{A^t} < 1$, whereas we only assume that $\norm{A^t} \le \rho$ for some $\rho \ge 1$ (as we show in Appendix~\ref{apdx:wiener}, it is quite easy to handle linear dynamical systems for which the former holds). 

In the absence of an analogue of \eqref{eq:unobservable_decay}, our key insight is that we can instead relate $\norm{A^t\Pi(x_i - x^*_i)}^2$ to the \emph{unobservable} error in the past! Informally, any large errors in estimating the observable part of the state $x^*_i$ must be explained by a large amount of interference from the unobservable part of the state. 
Essentially, in Lemma~\ref{lem:observable-case-analysis} we show that for some small constant $0 < c' < 1$,
\begin{equation}\label{eq:overview-observable}
\|\Pi(x_i - x_i^*)\|^2 \leq c'\cdot \| \Pi^{\perp}(x_i - x_i^*) \|^2 + [\text{noise}].
\end{equation}

The proof is rather involved but can be distilled into two main ingredients. Firstly, as alluded to earlier, the fact that the random corruptions subsample the observability matrix allows us to relate the error in estimating the state to the error in fitting the observations. Secondly, we can bound the latter using the following inequality (see Lemma~\ref{lem:pi-piperp-ineq}):
\begin{equation}\label{eqn:pi-piperp-intro}
a^*_i a_i\| BA^i \Pi (x_{0} - x^*_{0})\|^2 \leq 4a^*_i a_i\| BA^i \Pi^\perp (x_0 - x^*_0)\|^2 + [\text{noise}].
\end{equation}
Inequality \eqref{eqn:pi-piperp-intro} formalizes the following idea: because our estimator must match the observation at time $i$, any errors coming from the past are constrained to cancel between the observable and unobservable parts of the space, so if one is large the other is large as well. 

Unfortunately, combining the two ingredients above naively would establish \eqref{eq:overview-observable} with too large of a constant on the right hand side. 
It turns out however that we can greatly improve the constant by appealing to the aforementioned decay property of the unobservable subspace from \eqref{eq:unobservable_decay}. 
The details here are a bit subtle, and we refer to the proof of Lemma~\ref{lem:error_decay} for the formal argument.


\paragraph{Contraction of Error over Time.}
From \eqref{eq:unobservable_decay} and \eqref{eq:overview-observable}, it is straightforward to deduce \eqref{eq:error-contraction} and conclude that the error incurred by the filter is decaying exponentially quickly. Recalling that $\norm{(x_{i + t} - x_{i + t}^*) - A^t(x_i - x_i^*)}^2$ will be some small noise term, we have by triangle inequality that
\begin{align*} 
\| x_{i+t} - x_{i+t}^*\|^2 
&\le [\text{noise}] + 3\| A^t \Pi(x_{i} - x_i^*)\|^2 + 3\|A^t \Pi^{\perp}(x_{i} - x_i^*)\|^2 \\
&\le [\text{noise}] + 3c'\|\Pi^{\perp}(x_{i} - x_i^*)\|^2 + 3c\| \Pi^{\perp}(x_{i} - x_i^*)\|^2 < [\text{noise}] + \frac{1}{2}\|x_{i} - x_i^*\|^2,
\end{align*}
which proves \eqref{eq:error-contraction}. So we have established that over the course of $t$ timesteps, the error of our estimate essentially decays geometrically. We remark that this doesn't give us any control over our error on the iterates \emph{within} these timesteps, and in particular it remains to be shown how to use all of the tools we've introduced to compete with the Bayes-optimal predictor. It turns out however that by virtue of the program constraints ensuring that our estimate of the trajectory follows the linear dynamics \eqref{eq:follow_dynamics} and tries to fit the clean observations \eqref{eq:fit_clean_ys} while minimizing the objective \eqref{eqn:program_loss}, the excess risk introduced \emph{within} any window of $t$ is polynomially bounded by $t$ and the variance of the noise (see Lemma~\ref{lem:outer}). In this overview, we do not delve into the details of this as this particular argument is reminiscent of existing analyses in the robust statistics literature.



\paragraph{Confidence Bands and Achieving $\log\log T$ Excess Risk.}

There is an important catch in the above discussion which we now need to address: We assumed that every noise vector $u^*_i$ and $w^*_i$ is bounded, and that subsampling holds in every window of size $t$. 
The trouble is in order to ensure that these events hold across the entire trajectory of length $T$, we would need to take $t$ to be logarithmic in $T$. This would translate into incurring a logarithmic overhead in the excess risk bound. {\em We are able to avoid paying the full price for this union bound.} 

First, it turns out that our analysis only requires that the observation and process noise are bounded \emph{on average} over the trajectory in an appropriate sense (see Constraints~\ref{constraint2:totalnoise}-\ref{constraint2:measurements_noise_corrupted}). The more serious issue is: How can we avoid assuming that subsampling holds in each window? This is a key component to being able to integrate information across different time windows. As a starting point, we observe that our initial estimator actually meets a stronger guarantee: It actually outputs a trajectory which is \emph{pointwise} close to the true trajectory by a distance that scales polynomially in variance of the noise and polylogarithmically in $T$ (see Corollary~\ref{cor:warmstart}). In other words, it allows us to form a \emph{confidence band} around the estimator of radius, and this logarithmic scaling is essentially optimal for any guarantee of this form.



We show how to exploit this confidence band. In particular, we engineer a second system of constraints which incorporates the output of our initial estimator and refines it to achieve our final $\log\log T$ excess risk bound in Theorem~\ref{thm:sos_informal}. The main idea is because the noise that accumulates over a window scales polynomially in the window size, we can consider windows over \emph{shorter timescales} scaling doubly logarithmically rather than logarithmically in $T$. The goal is to achieve higher accuracy on shorter windows whenever subsampling holds. This might seem counterproductive: By shortening the windows, we are only making it more likely that subsampling will fail in any given window! Indeed, if we are now taking windows of doubly logarithmic length, then in roughly a $1/\polylog(T)$ fraction of the windows, the random corruptions will fail to properly subsample the observability matrix. But this is where the confidence band comes in: Over these bad windows, we already know how to estimate those iterates pointwise to error $\polylog(T)$, so the total contribution of the bad windows to the excess risk scales as $(1/\polylog(T))\cdot \polylog(T) = O(1)$ !

The key complication is that the algorithm designer doesn't actually know which windows subsampling failed in. Instead, we will set up a system of constraints similar to the one for our earlier estimator but with additional Boolean variables, one per window, corresponding to our guess for whether the random corruptions in that window successfully subsampled the observability matrix (see Program~\ref{program:sosv2}). We show how to integrate information \emph{across the windows on which we correctly guessed that subsampling succeeded} to achieve our final guarantee. As the argument here is rather involved and we defer the details to Section~\ref{sec:decay} and Lemma~\ref{lem:error-first-iterate-sos2}.

\paragraph{Efficient Algorithm via Sum-of-Squares.} 
While the estimators we have described appear to be inefficient as they require solving certain systems of polynomial constraints, our proofs that the solutions to these systems satisfy the guarantees of Theorem~\ref{thm:sos_informal} are simple in the sense that they can be implemented in the degree four sum-of-squares proof system \cite{Shor87,parrilo2000structured}. So instead of solving these polynomial systems which would \emph{a priori} incur an exponential runtime, it suffices to output a \emph{pseudo-distribution} over solutions and round it to an integral solution in a straightforward way. These tools have become a mainstay in algorithm design in robust statistics more broadly \cite{hopkins2018mixture, kothari2018robust, bakshi2020outlier, diakonikolas2018robustly, bakshi2021robust, liu2021settling, bakshi2020robustly}. We will explain the basics of the sum-of-squares proof system in Section~\ref{sec:sos_basics}.

\paragraph{Two-stage filter.} Thus far we have been discussing the offline problem, where the estimator at some time instant is allowed to depend on the entire observation sequence $y_1,\ldots,y_T$ \--- in other words, we designed a robust Kalman {\em smoother}. Our techniques can be used to solve the online filtering problem, where the prediction for $x_t$ is only allowed to depend on $y_1,\ldots,y_{t - 1}$. It turns out that the transformation to online guarantees is simple: We use the offline smoother developed above, in particular the confidence band it outputs, as an outlier removal algorithm and combine it with the standard Kalman filter run on the observations $\tilde{y}_1,\ldots,\tilde{y}_T$ which are not deleted by outlier removal. After the outlier removal, only small corruptions remain, so following an idea of Schick and Mitter \cite{schick1989robust,schick1994robust}, we can then use stability estimates for the Kalman filter to establish accuracy guarantees on its prediction for the next state.  

%% file: prelims.tex

\section{Preliminaries}
\label{sec:prelims}

Given matrix $M$, let $\lambda_{\min}(M)$ and $\lambda_{\max}(M)$ denote its bottom and top singular values, and let $\kappa(M) \triangleq \frac{\lambda_{\min}(M)}{\lambda_{\max}(M)}$. We will sometimes also denote $\lambda_{\max}(M)$ by $\norm{M}$.

\subsection{Generative Model}
\label{sec:generative}

For known matrices $A\in\R^{d\times d}$, $B\in\R^{m\times d}$, the underlying trajectory $\brc{x^*_i}$ and uncorrupted observations $\brc{y^*_i}$ are given by $x^*_0 \sim\calN(0,R^2\cdot\Id_d)$ and
\begin{align}
	x^*_i &= Ax^*_{i-1} + w^*_i \ \text{for all} \ i>0 \label{eq:dynamics} \\
	y^*_i &= Bx^*_i + v^*_i \ \text{for all} \ i\ge 0\label{eq:measure}
\end{align}
where the \emph{dynamics noise} $w^*_i$ is independently sampled from $\calN(0,\sigma^2 \cdot \Id_d)$ and and the \emph{observation noise} $v^*_i$ is independently sampled from $\calN(0,\tau^2 \cdot \Id_m)$, i.e. both types of noise in the system are isotropic Gaussian up to scaling. The assumption that the noise is isotropic simplifies notation greatly, and is largely without loss of generality in the following sense: if the noise covariance matrices are full rank, a change of basis will make the noise isotropic. Our results can also be extended to handle the case where the noise covariance is rank degenerate.

We now describe how the corrupted observations are formed. After the trajectory $\brc{x^*_i}$ and observations $\brc{y^*_i}$ have generated, an independent $\mathrm{Ber}(1 - \eta)$ coin is flipped for every $0\le i < T$; let $a^*_i\in\brc{0,1}$ denote the outcome at time $i$. For all $i$ for which $a^*_i = 1$, define $y_i = y^*_i$. For all $i$ for which $a^*_i = 0$, a computationally unbounded adversary is allowed to set $y_i$ arbitrarily. We can assume this adversary has full knowledge of the system, e.g. the full trajectory $\brc{x^*_0,\ldots,x^*_{T_1}}$, the full sequence of true observations $\brc{y^*_0,\ldots,y^*_{T-1}}$, etc.


The baseline estimation error which we will try to achieve approximately is
\begin{equation}
    \OPT = \min_{\brc{\wh{x}_i}} \frac{1}{T} \left(\sum_{i = 0}^{T-1} (a^*_i \|B \wh{x}_i - y_i\|^2/\tau^2 + \norm{\wh{w}_i}^2/\sigma^2) + \|\wh{x}_0\|^2/R^2\right) \label{eq:OPT_def}
\end{equation}
where the minimum ranges over all trajectories $\wh{x}_i$ with steps $\wh{w}_i$ satisfying $\wh{x}_i = A\wh{x}_{i-1} + \wh{w}_i$ for all $i > 0$.
Note that the objective function here corresponds to the negative log-density of the (Gaussian) posterior of the trajectory given the clean observations, up to additive constants and multiplicative factors.
In particular the minimum in \eqref{eq:OPT_def} (i.e. the posterior MAP as well as the posterior mean, since the posterior is Gaussian) is achieved by iterates $\wh{x}_i$ and steps $\wh{w}_i$ given by running the offline Kalman filter (a.k.a. \emph{Kalman smoother}) on the part of the trajectory indexed by $i$'s for which $a^*_i = 1$. Since the algorithm does not know which times are uncorrupted (have $a^*_i = 1$), we cannot hope to exactly match the performance of $\OPT$. However, we use it as a benchmark and bound the amount of excess error our algorithms make compared to this oracle. 

We make the following assumptions which are standard in the filtering/control literature (see e.g. \cite{anderson2007optimal,anderson2012optimal,schick1994robust}):
\begin{assumption}[Complete Observability] \label{assume:obs}
	There exist constants $\alpha,\kappa > 0$ and $s\in\mathbb{N}$ for which \begin{equation}
		\calO_s \triangleq \sum^{s-1}_{i=0} (A^i)^{\top}B^{\top}B A^i
	\end{equation} satisfies $\lambda_{\min}(\calO_s) \ge \kappa s$ and $\lambda_{\max}(\calO_s) \le \alpha s$.
\end{assumption}

\begin{assumption}[Uniform stability] \label{assume:uniform}
	There is a constant $\rho \ge 1$ for which $\norm{A^j} \le \rho$ for all $j\in\mathbb{N}$.
\end{assumption}




The following elementary manipulations will be useful:

\begin{fact}\label{fact:epochs}
	For $t\in\mathbb{N}$ divisible by $s$, $\calO_t = \sum^{t/s - 1}_{r=0} (A^{rs})^{\top}\calO_s A^{rs}$. In particular, \begin{equation}
		\norm{\calO_t} \le t\cdot \alpha \cdot \rho^2. \label{eq:opnormOt}
	\end{equation}
\end{fact}

\begin{fact}\label{fact:unroll}
	$x^*_t - A^t x^*_0 = \sum^t_{j=1} A^{t-j} w^*_j$
\end{fact}

\subsection{Sum-of-Squares Basics}
\label{sec:sos_basics}

Here we give a quick review of the sum-of-squares algorithm; for a more thorough treatment, we refer the reader to \cite{o2013approximability,barak2014sum,hopkins2018statistical}.

\paragraph{Pseudoexpectations.} The sum-of-squares SDP hierarchy is a series of increasingly tight SDP relaxations for solving polynomial systems $\calP \triangleq \{p_i(x) \geq 0\}_{i=1}^N$.  Although it is in general NP-hard to solve polynomial systems, the level-$\ell$ SoS SDP attempts to approximately solve $\calP$ with increasing accuracy as $\ell$ increases by adding more constraints to the SDP.  This improvement in approximation naturally comes at the expense of increasing runtime and space.  

In particular, one can think of the SoS SDP as outputting a "distribution" $\mu$ over solutions to $\calP$.  However, there are two important caveats.  Firstly, one can only access the degree-$\ell$ moments of the "distribution" and secondly there may be no true distribution with the corresponding degree $\ell$ moments.  Thus we refer to $\mu$ as a \emph{pseudodistribution}.

\begin{definition}
A degree $\ell$ pseudoexpectation $\wt{\mathbb{E}} : \R[x]_{\leq \ell} \to \R$ {\it satisfying $\calP$} is a linear functional over polynomials of degree at most $\ell$ satisfying 
\begin{enumerate}
\item(Normalization) $\psE*{1} = 1$, 
\item(Constraints of $\calP$) $\psE*{p(x) a^2(x)} \geq 0$ for all $p \in \calP$ and polynomials $a$ with $\deg(a^2 \cdot p) \le \ell$,
\item(Non-negativity on square polynomials) $\psE*{q(x)^2} \ge 0$ whenever $\deg(q^2) \le \ell$.
\end{enumerate}
\end{definition}

For any fixed $\ell \in \mathbb{N}$, given a polynomial system, one can efficiently compute a degree $\ell$ pseudo-expectation in polynomial time. 
\begin{fact} (\cite{Nesterov00,parrilo2000structured,Lasserre01,Shor87}). For any $n$, $\ell \in \Z^+$, let $\pE_{\zeta}$ be degree $\ell$ pseudoexpectation satisfying a polynomial system $\calP$. Then the following set has a $n^{O(\ell)}$-time weak
separation oracle (in the sense of \cite{GLS1981}):
 \begin{align*}
 & \brc*{\pE_\zeta\left[(1, x_1, x_2, . . . , x_n)^{\otimes \ell}\right]: \text{ degree } \ell \text{ pseudoexpectations } \pE_{\zeta} \text{ satisfying }\calP}
\end{align*}

Using this separation oracle, the ellipsoid algorithm finds a degree $\ell$ pseudoexpectation in time $n^{O(\ell)}$. We call this algorithm the \emph{degree $\ell$ sum-of-squares algorithm}. 
\end{fact}

To reason about the properties of pseudo-expectations, we turn to the dual object of sum-of-squares proofs.

\paragraph{Sum-of-Squares Proofs}
For any nonnegative polynomial $p(x): \R^d \rightarrow \R$, one could hope to prove its nonnegativity by writing $p(x)$ as a sum of squares of polynomials $p(x) = \sum_{i=1}^m q_i(x)^2$ for a collection of polynomials $\{q_i(x)\}_{i=1}^m$.  Unfortunately, there exist nonnegative polynomials with no sum-of-squares proof even for $d = 2$.  Nevertheless, there is a generous class of nonnegative polynomials that admit a proof of positivity via a proof in the form of a sum of squares.  The key insight of the sum-of-squares algorithm is that these sum-of-squares proofs of nonnegativity can be found efficiently provided the degree of the proof is not too large. 

\begin{definition} (Sum-of-Squares Proof)
Let $\mathcal{A}$ be a collection of polynomial inequalities $\{p_i(x) \geq 0\}_{i=1}^m$.  A sum-of-squares proof that a polynomial $q(x) \geq 0$ for any $x$ satisfying the inequalities in $\mathcal{A}$ takes on the form 

\[
    \left(1+ \sum_{k \in [m']} b_k^2(x)\right) \cdot q(x) = \sum_{j\in [m'']} s_j^2(x) + \sum_{i \in [m]} a_i^2(x) \cdot p_i(x) 
\]
where $\{s_j(x)\}_{j \in [m'']},\{a_i(x)\}_{i \in [m]}, \{b_k(x)\}_{i \in [m']}$ are real polynomials.  If such an expression were true, then $q(x) \geq 0$ for any $x$ satisfying $\mathcal{A}$.  We call these identities sum-of-squares proofs, and the degree of the proof is the largest degree of the involved polynomials $\max \{\deg(s_j^2), \deg(a_i^2 p_i)\}_{i,j}$.  Naturally, one can capture polynomial equalities in $\mathcal{A}$ with pairs of inequalities.   We denote a degree $\ell$ sum-of-squares proof of the positivity of $q(x)$ from $\calA$ as $\calA \sststile{\ell}{x} \{q(x) \geq 0\}$ where the superscript over the turnstile denote the formal variable over which the proof is conducted.  This is often unambiguous and we drop the superscript unless otherwise specified.    
\end{definition}

A number of basic inequalities like Cauchy-Schwarz and H\"{o}lder's admit sum-of-squares proofs (see e.g. Appendix A of \cite{hopkins2018mixture}).

Sum-of-squares proofs can also be strung together and composed according to the following convenient rules.  
\begin{fact}
For polynomial systems $\calA$ and $\calB$, if $\calA \sststile{d}{x} \{p(x) \geq 0\}$ and $\calB \sststile{d'}{x} \{q(x)\geq 0\}$ then $\calA \cup \calB \sststile{\max(d,d')}{x}\{p(x) + q(x) \geq 0\}$.  Also $\calA \cup \calB \sststile{dd'}{x} \{p(x)q(x) \geq 0\}$ 
\end{fact}

Sum of squares proofs yield a framework to reason about the properties of pseudo-expectations, that are returned by the SoS SDP hierarchy.  
%
\begin{fact} (Informal Soundness)
If $\calA \sststile{r}{x} \{q(x) \ge 0\}$ and $\psE{\cdot}$ is a degree-$\ell$ pseudoexpectation operator for the polynomial system defined by $\calA$, then $\psE{q(x)}  \ge 0$.
\end{fact}

The following standard fact which follows by ``SoS Cauchy-Schwarz'' (see e.g. Lemma A.5 of \cite{barak2014rounding}) will allow us to convert from pseudodistributions over solutions to a polynomial systems to integral solutions.

\begin{lemma}\label{lem:psE-CS}
	For any vector $w^*$ and degree-2 pseudoexpectation $\psE{\cdot}$ over vector-valued variable $w$, we have that \begin{equation}
		\norm{\psE{w} - w^*}^2 \le \psE{\norm{w - w^*}^2}.\label{eq:soscs}
	\end{equation}
\end{lemma}

\begin{proof}
	By the dual definition of $L_2$ norm, the left-hand side of \eqref{eq:soscs} can be written as \begin{equation}
	    \sup_{v\in\S^{d-1}}\iprod{\Sig v,\psE{w} - w^*}^2.
	\end{equation} For any $v\in\S^{d-1}$, \begin{equation}
		\iprod*{\Sig v,\psE{w} - w^*}^2 = \left(\psE{\iprod{\Sig v,w - w^*}}\right)^2 \le \psE{\iprod{\Sig v,w-w^*}^2} \le \psE{\norm{w - w^*}^2_{\Sig}},
	\end{equation} where the first inequality follows by the pseudoexpectation version of SoS Cauchy-Schwarz (see e.g. Lemma A.5 of \cite{barak2014rounding}). Therefore, taking the maximum over all $v \in S^{d - 1}$ proves the inequality.
\end{proof}

Finally, we will need the following elementary but crucial inequality which admits a degree-2 sum-of-squares proof. Roughly speaking, it captures the fact that if the sum of two vectors is small in norm, then either vector must have norm upper bounded in terms of the norm of the other vector:

\begin{fact}\label{fact:sos_simple}
    Let $v_1,v_2$ be $d$-dimensional vector-valued indeterminates. There is a degree-2 sum-of-squares proof of the inequality $\norm{v_1}^2 \le 4\norm{v_2}^2 + \frac{4}{3}\epsilon$ from the constraint $\norm{v_1 + v_2}^2 \le \epsilon$.
\end{fact}

\begin{proof}
    By expanding out the hypothesis, we have
    \[\|v_1\|^2 + 2\langle v_1, v_2 \rangle + \|v_2\|^2 \leq \epsilon.\] By Cauchy-Schwarz, we also have
    \[-2\langle v_1, v_2 \rangle \leq \frac{1}{4}\|v_1\|^2 + 4\|v_2\|^2.\]
    Adding these two inequalities together and rearranging gives the desired inequality.
\end{proof}

\subsection{Concentration Inequalities}
\begin{lemma}[Matrix Hoeffding, see e.g. Theorem 1.3 in \cite{tropp}]\label{lem:matrixhoeffding}
    For any $\delta > 0$, given symmetric random matrices $M_1,\ldots,M_T \in\R^{d\times d}$ satisfying $\norm{M_t} \le 1$ almost surely for all $t$, if $M \triangleq \sum_t M_t$, then \begin{equation}
        \Pr{\norm{M - \E{M}} \ge \sqrt{8T\log(d/\delta)}} \le \delta.
    \end{equation}
\end{lemma}

\begin{lemma}[see e.g. \cite{vershynin2018high}]\label{lem:normbound}
	If $v\sim\calN(0,\Sig)$ for some $\Sig\in\R^{d\times d}$, then with probability at least $1 - \delta$, \begin{equation}
		\norm{v} \le O\left(\left(\sqrt{d} + \sqrt{\log(1/\delta)}\right)\norm{\Sig}^{1/2}\right)
	\end{equation}
\end{lemma}
We use concentration for Gaussian polynomials, which is a consequence of Gaussian hypercontractivity.
\begin{lemma}[see e.g. \cite{o2014analysis}]\label{lem:polyconc}
	For degree-$d$ polynomial $p: \R^m\to\R$, if $x\sim\calN(0,\Id)$, then \begin{equation}
		\Pr*{\abs*{p(x) - \E{p}} > t\cdot \sqrt{\Var{p}}} \le \exp(-\Omega(t^{2/d})).
	\end{equation}
\end{lemma}




%% file: TlogT.tex
\newcommand{\thres}{\zeta}
\newcommand{\tinynum}{\nu}

\section{Poly-logarithmic Excess Risk and Confidence Band Recovery}
\label{sec:logT}

In this section we show how to achieve excess risk scaling poly-logarithmically in the number of iterations. While this is worse than the final bound we will show in Section~\ref{sec:TloglogT}, it will introduce many of the important steps in the final analysis and also yield a warm start for our estimate of the trajectory which we will subsequently refine in Section~\ref{sec:TloglogT} to get our final bound. Crucially, we show this algorithm can output a \emph{confidence band} which with high probability (over the entire data generating process) contains the true trajectory.

The main result of this section is the following:
\begin{theorem}\label{thm:main_TlogT}
    For any $\eta \le 0.49$, there is a polynomial-time algorithm that, given the corrupted observations $\brc{\wt{y}_i}$, with probability $1 - \delta$ over the randomness of the input, outputs a trajectory $\brc{\wh{x}_i}$ and steps $\brc{\wh{w}_i}$ for which $\wh{x}_i = A\wh{x}_{i - 1} +\wh{w}_i$ for every $i\in [T]$, and for which,
    \begin{multline}
        \frac{1}{T}\left(\sum^{T-1}_{i=0}\left(a^*_i \norm{B\wh{x}_i - y_i}^2/\tau^2 + \norm{\wh{w}_i}^2/\sigma^2\right) + \|\wh x_0\|^2/R^2\right) - \OPT \\
        \lesssim
        \tau^{-2} \eta \cdot \brk*{E_{\mathsf{noise}} + \rho^2\left(\alpha +\norm{B}^2\sqrt{\log(dT/t\delta)/t}\right)\cdot \left(\frac{\rho^6 E_{\mathsf{noise}} t}{\kappa} + \frac{R^2}{T/t}(d + \log(1/\delta))\right)},
    \end{multline}
    where 
    \begin{align}
        E_{\mathsf{noise}} &\triangleq \tau^2\left(m + \log(T/\delta)\right) + t\rho^2\norm{B}^2\sigma^2\left(d + \log(T/\delta)\right) \label{eq:Enoise_def} \\
        t &\triangleq \Max{s}{\wt{\Theta}(\kappa^{-2}\rho^{12}\norm{B}^4\log(dT/\delta)))}. \label{eq:tdef}
    \end{align}
\end{theorem}

\subsection{Sum-of-Squares Relaxation}

We now formulate the sum-of-squares program we work with in this section. We begin by introducing an important parameter, the so-called \emph{window size} $t$. Recall from Assumption~\ref{assume:obs} that we assume that the observability matrix $\calO_s$ is well-conditioned. We will take $t$ to be a sufficiently large multiple of $s$ such that, roughly speaking, the contribution to the observability matrix $\calO_t$ from the uncorrupted time steps is also well-conditioned. We defer the tuning of $t$ to later in the proof of Theorem~\ref{thm:main_TlogT}.
For convenience, given $0\le i < T$, let $\ell(i)\triangleq \floor{i/t}$ denote the index of the window to which iterate $i$ belongs.

At this point, we can define our sum-of-squares relaxation:

\begin{program}\label{program:sos}
	Let $\brc{y_i}$ be the observations we are given, and let window size $t\in\mathbb{N}$ be a parameter to be tuned later. The program variables are $d$-dimensional vector-valued variables $\brc{x_i}$ (trajectory estimates) and $\brc{w_i}$ (process noise estimates), $m$-dimensional vector-valued variables $\brc{v_i}$ (observation noise estimates), and Boolean variables $\brc{a_i}$ (indicators for uncorrupted time steps), and the constraints are that for all $0 \le i < T$, 
	\vspace{0.05cm}
	\noindent\emph{Boolean indicators for uncorrupted steps}
	\begin{enumerate}
		\item $a_i^2 = a_i$ \label{item:boolean}
	\end{enumerate}
	\noindent\emph{Trajectory estimate follow linear dynamics and fit $y_i$'s on uncorrupted steps}
	\begin{enumerate}
	    \setcounter{enumi}{1}
		\item $x_i = Ax_{i-1} + w_i$ \label{item:dynamics}
      	\item $a_i(y_i - Bx_i - v_i) = 0$ \label{item:measurements}
  	\end{enumerate}
  	\noindent\emph{Only $\eta$ fraction of timesteps corrupted}
  	\begin{enumerate}
  	\setcounter{enumi}{3}
      	\item $\sum_{i=0}^{T-1} a_i \ge (1-1.01\eta)T$ \label{item:many_inliers}
    \end{enumerate}
    \noindent\emph{Process and observation noise bounded}
    \begin{enumerate}
        \setcounter{enumi}{4}
        \item $\norm{v_i}^2 \le O(\tau^2(m + \log(T/\delta))$\label{item:measurements_noise}
        \item $\norm{w_i}^2 \le O(\sigma^2(d + \log(T/\delta))$\label{item:dynamics_noise}
    \end{enumerate}
    \noindent\emph{Random corruptions subsample observability matrix in each window}
    \begin{enumerate}
        \setcounter{enumi}{6}
        \item $\sum^{t-1}_{j=0} (1 - a_{\ell t + j}) (A^j)^{\top} B^{\top} B A^j \preceq \eta\cdot \calO_{\t} + O\left(\rho^2\norm{B}^2\sqrt{\t\log(dT/\t\delta)}\right)\cdot \Id$ for all $0\le \ell < T/\t$ \label{item:subsample}
    \end{enumerate}
    \noindent\emph{Initial state bounded}
    \begin{enumerate}
        \setcounter{enumi}{7}
        \item $\norm{x_0}^2 \le R^2(d + O(\log(1/\delta)))$ \label{constraint:bound}
	\end{enumerate}
	The program objective is to minimize
	\begin{equation}
		\min \frac{1}{T} \psE*{ \sum^{T-1}_{i=0} \left(a_i \norm{Bx_i - y_i}^2/\tau^2 + \norm{w_i}^2/\sigma^2\right) + \|x_0\|^2/R^2}\label{eq:obj}
	\end{equation}
	over degree-$4$ pseudoexpectations satisfying the above constraints. 
\end{program}
\begin{remark}[Uncorrupted Case: Equivalence to Kalman Smoother]\label{rmk:uncorrupted-case}
Suppose that we know there are no corruptions: then we can set $\eta = 0$ in the above program and therefore eliminate the variables $a_i$ (they are all equal to $1$). Then, by a well-known folklore argument, the SoS program is equivalent to the corresponding convex program with actual variables $x_i \in \mathbb{R}^d, v_i \in \mathbb{R}^m$, etc. with the same set of constraints. (This is because, by SoS Cauchy Schwarz, replacing the pseudoexpectation $\psE*{\cdot}$ with the delta distribution over $\psE*{x}$ gives a valid pseudoexpectation with equal or better objective value.) Then the objective is the same as the MAP objective, and as argued below the constraints are satisfied with high probability by the unconstrained MAP solution (Kalman smoother), so our algorithm simply outputs the MAP.
\end{remark}

\subsection{Feasibility of Oracle Kalman Smoother}

In the following section, we show that the output of the oracle Kalman filter, i.e. the algorithm which knows precisely which time steps have been corrupted and runs the offline Kalman filter (Kalman smoother) on the uncorrupted steps to optimally estimate the trajectory, satisfies the constraints of the Program with high probability. In the proof of Lemma~\ref{lem:feasible}, we show how to do this by reducing to showing that the ground truth $x^*$ is feasible with high probability, which is more straightforward. The key fact which allows us to do this is knowledge that the posterior is a Gaussian centered at the output of the Kalman filter. 
\begin{lemma}\label{lem:feasible}
    Let $\brc{x_i}$ be the sequence of estimates given by running the Kalman smoother (i.e. offline Kalman filter) on the uncorrupted part of the trajectory, let $a_i = a^*_i$, let $w_i = x_i - A x_{i - 1}$ for all $T$, let $v_i = y_i - B x_i$ when $a^*_i = 1$ and otherwise $v_i = 0$. Let $E[\cdot]$ be the expectation with respect to the delta distribution at this point $(x_i,a_i,v_i,w_i)_{i = 1}^n$. Then $E[\cdot]$ is feasible for Program~\ref{program:sos} with probability at least $1 - \delta$.
\end{lemma}
\begin{proof}
It is immediate that Constraints~\ref{item:boolean}, \ref{item:dynamics}, and \ref{item:measurements} are satisfied.

Constraints~\ref{item:many_inliers} and \ref{item:subsample} only involve $a^*_i$ and we verify them in Lemma~\ref{lem:feasible-gt}. It remains to check Constraints \ref{item:measurements_noise} and \ref{item:dynamics_noise}.

For what follows, suppose $a^*_i$ is fixed. 
We claim the following two distributions on $\brc{x^*_i}$ are equal:
\begin{enumerate}
    \item Sample a trajectory $\brc{x^*_i}$ from the prior.
    \item Sample a trajectory $\brc{x^0_i}$ from the prior, sample observations $y_i$ for times where $a^*_i = 1$ given this trajectory, and sample trajectory $\brc{x^*_i}$ from the resulting posterior on $\brc{x^0_i}$ given $y_i$.
\end{enumerate}
The equivalence of these two follows from the following basic fact: given a pair of random variables $(X,Y)$, it's equivalent to sample $X$ from its marginal law directly, or to first sample $Y$ from its marginal law, and then to sample $X$ conditional on $Y$. In the second case, the observations are the random variable $Y$ and the trajectory is $X$; the fact implies that $Y$ is sampled from its marginal law, which means that the marginal law of $\brc{x^*_i}$ is simply the prior on trajectories. This fact is sometimes called the Nishimori identity.

Recall that the Kalman smoother output is simply the posterior mean $\hat{x}_i = \E{x^*_i \mid \{y_i\}_{i : a^*_i = 1}}$ and that the posterior on trajectories is a multivariate Gaussian distribution. By Lemma~\ref{lem:feasible-gt}, we have that
\begin{align}
    \|y_i - B x^*_i\|^2 &\le O(m\tau^2 + \tau^2\log(T/\delta)) \\
    \|x^*_i - A x^*_{i - 1}\|^2 &\le O(d\sigma^2 + \sigma^2\log(T/\delta)) \label{eqn:calK}\\
    \|x^*_0\|^2 &\le R^2(d + O(\log(1/\delta)))
\end{align} 
uniformly over $i$ with probability at least $1 - \delta$, then by the law of total probability we know that for $\mathcal{K}$ the feasible set defined by the constraints above in \eqref{eqn:calK},
\[ \delta \ge \Pr{(x^*,y) \notin \mathcal{K}} = \E{\Pr{(x^*,y) \notin \mathcal{K} | y}} \]
so by Markov's inequality $\Pr{\Pr{(x^*,y) \notin \mathcal{K} | y} > 1/3} \le 3\delta$, i.e. $\Pr{\Pr{(x^*,y) \in \mathcal{K} | y} \ge 2/3} \ge 1 - 3\delta$, which by Lemma~\ref{lem:gaussian-convex} implies that $\Pr{(\E{x^*|y},y) \in \mathcal{K}} \ge 1 - 3\delta$ as well. 
Adjusting the value of $\delta$ by constants proves the result.
\end{proof}
\begin{lemma}\label{lem:gaussian-convex}
    Suppose that $\mathcal{K}$ is a closed convex set, $Z \sim N(\mu,\Sigma)$ is an arbitrary Gaussian random vector, and $\Pr{Z \in \mathcal K} \ge 0.5$. Then $\mu \in \mathcal{K}$.
\end{lemma}
\begin{proof}
    First we show this when $\mathcal{K}$ is an affine halfspace, i.e. $\mathcal{K} = \{ x : \langle a, x \rangle \ge b \}$ for some $a$ and $b$ arbitrary.
    The assumption gives that $\langle a, Z \rangle \ge b$ with probability greater than 50\%; since the marginal law of $\langle a, Z \rangle$ is $N(\langle a, \mu \rangle, a^T \Sigma a)$, and the Gaussian is symmetrical about its mean, it must be that $\langle a, \mu \rangle \ge b$ and so $\mu \in \mathcal{K}$. Now the result follows for arbitrary convex sets by writing them as intersections of affine halfspaces, since the above argument shows that $\mu$ will lie in each halfspace (since the probability of lying in each halfspace is at least as large as lying in the intersection), hence in the intersection of the halfspaces. 
\end{proof}

\begin{lemma}\label{lem:matrix_conc} 
    For any $\delta > 0$, \begin{equation}
        \norm*{\sum^{\t-1}_{i=0}a^*_i (A^i)^{\top}B^{\top}BA^i - (1 - \eta)\calO_{\t}} \le O(\rho^2\norm{B}^2 \sqrt{ \t \log(d/\delta)}) \label{eq:deviation}
    \end{equation} with probability at least $1 - \delta$.
\end{lemma}

\begin{proof}
    We apply the Matrix Hoeffding inequality (Lemma~\ref{lem:matrixhoeffding}), 
    using that $\norm{(A^i)^{\top} B^{\top}BA^i} \le \rho^2\norm{B}^2$ by uniform stability.
\end{proof}

\begin{lemma}\label{lem:feasible-gt}
With probability at least $1 - \delta$, the ground truth $(x^*_i,w^*_i, v^*_i,a^*_i)$ satisfies the constraints of Program~\ref{program:sos} 
provided $T = \Omega(\log(2/\delta)/\eta)$. 
\end{lemma}
\begin{proof}
Equality constraints~\ref{item:boolean}, \ref{item:dynamics}, and \ref{item:measurements} are satisfied by definition of the process. The remaining inequality constraints follow from a union bound as follows. The bound on Constraint~\ref{item:many_inliers} follows from Bernstein's inequality (see e.g. \cite{vershynin2018high}). 
Constraint~\ref{item:measurements_noise} follows by standard Gaussian concentration with probability at least $1 - \delta$.
The same reasoning applies to Constraints~\ref{item:dynamics_noise} and \ref{constraint:bound}. Constraint 7 follow from Lemma~\ref{lem:matrix_conc} applied to every window $0\le \ell < T/t$.
\end{proof}

\subsection{Outer Argument}

In this section we reduce the problem of competing with $\OPT$ to getting good prediction error on the first iterate of every window.

\begin{lemma} \label{lem:outer}
    Let $\psE{\cdot}$ be the solution to Program~\ref{program:sos}, assuming it is feasible. Let $\wh{x}_i \triangleq \psE{x_i}$ and $\wh{w}_i \triangleq \psE{w_i}$ for every $0\le i < T$. Provided the event of Lemma~\ref{lem:feasible} holds, then \begin{multline}
        \frac{1}{T}\left(\sum^{T-1}_{i=0}\left(a^*_i \norm{B\wh{x}_i - y_i}^2/\tau^2 + \norm{\wh{w}_i}^2/\sigma^2\right) + \|\wh{x}_0\|^2/R^2\right) - \OPT \lesssim \\
        \eta \left(E_{\mathsf{noise}} + \rho^2\left(\alpha +\norm{B}^2\sqrt{\log(dT/t\delta)/t}\right)\cdot \frac{1}{T/t}\sum^{T/t-1}_{\ell=0} \psE*{\norm{x_{\ell t} - x^*_{\ell t}}^2}/\tau^2\right). \label{eq:outer}
    \end{multline}
    where $E_{\mathsf{noise}}$ is defined in \eqref{eq:Enoise_def}.
\end{lemma}

Before proving this, we will need the following helper lemma which we will reuse with minor modifications later in Section~\ref{sec:TloglogT}.

\begin{lemma}\label{lem:kkm}
    Let $\psE{\cdot}$ be the solution to Program~\ref{program:sos}, assuming it is feasible. Let $\wh{x}_i \triangleq \psE{x_i}$ and $\wh{w}_i \triangleq \psE{w_i}$ for every $0\le i < T$. Provided the event of Lemma~\ref{lem:feasible} holds, then
    \begin{multline}
        \frac{1}{T}\left(\sum_{i=0}^{T-1}\left(a^*_i \| B\wh{x}_i - y_i\|^2/\tau^2 +  \norm{\wh{w}_i}^2/\sigma^2\right) + \|\hat{x}_0\|^2/R^2\right) - \OPT \le \\
        \psE*{\frac{1}{T}\sum_{i=0}^{T-1} (1 - a_i)\| B(x_i - x^*_i)\|^2/\tau^2} + O(\eta\cdot (m + \log(T/\delta))).
    \end{multline}
\end{lemma}

\begin{proof}
    By Lemma~\ref{lem:psE-CS}, for any $0\le i < T$, $\norm{B\wh{x}_i - y_i}^2 \le \psE{\norm{Bx_i - y_i}^2}$ and $\norm{\wh{w}_i}^2 \le \psE{\norm{w_i}^2}$, so it suffices to prove that the pseudoexpectation of $\sum^{T-1}_{i=0}\left(a^*_i\norm{Bx_i - y_i}^2/\tau^2 + \norm{w_i}^2/\sigma^2\right) + \|x_0\|^2/R^2$ is sufficiently bounded using the constraints of Program~\ref{program:sos}.
    First, by splitting up $a^*_i = a^*_ia_i + a^*_i(1-a_i)$, we have
    \begin{align}
        & \sum_{i=0}^{T-1} \left(a^*_i\| Bx_i - y_i\|^2/\tau^2 + \norm{w_i}^2/\sigma^2\right)  \label{eq:main_obj}
        \\
        &= \sum_{i = 0}^{T-1}\left(\norm{w_i}^2/\sigma^2 + a^*_ia_i\| Bx_i - y_i\|^2/\tau^2 + a^*_i(1 - a_i)\| Bx_i - y_i\|^2/\tau^2\right) \\
        &\leq \sum_{i = 0}^{T-1}\left(\norm{w_i}^2/\sigma^2 + a_i \| Bx_i - y_i\|^2/\tau^2 + 2a^*_i(1 - a_i)\left( \| B(x_i - x^*_i)\|^2/\tau^2 + \|v_i\|^2/\tau^2\right)\right) \label{eq:rewrite_obj}
    \end{align}
    where in the inequality we used the fact that $a^*_i \le 1$ and that for $i$ satisfying $a^*_i = 1$, $\norm{Bx_i - y_i}^2 = \norm{B(x_i - x^*_i) - v_i}^2 \le 2\norm{B(x_i - x^*_i)}^2 + 2\norm{v_i}^2$. Furthermore, note that
    \begin{equation}
        \sum^{T-1}_{i=0} a^*_i(1 - a_i)\norm{v_i}^2/\tau^2 \lesssim \eta\left(m + \log(T/\delta)\right)T, \label{eq:av}
    \end{equation} by Constraints~\ref{item:many_inliers} and \ref{item:measurements_noise}. Putting \eqref{eq:rewrite_obj} and \eqref{eq:av} together allows us to upper bound the pseudo-expectation of $\sum^{T-1}_{i=0}\left(a^*_i\norm{Bx_i - y_i}^2/\tau^2 + \norm{w_i}^2/\sigma^2\right) + \|x_0\|^2/R^2$ by
    \begin{equation}
        \OPT + \psE*{\frac{1}{T}\sum_{i=0}^{T-1} (1 - a_i)\| B(x_i - x^*_i)\|^2/\tau^2} + O(\eta(m + \log(T/\delta))). \label{eq:firstopt}
    \end{equation}
    where we used the fact that $\psE{\cdot}$ minimizes the objective \eqref{eq:obj}, the fact that the oracle Kalman filter solution is feasible because the event of Lemma~\ref{lem:feasible} holds, as well as the fact that $a^*_i \le 1$.
\end{proof}

We now proceed with the proof of Lemma~\ref{lem:outer}.

\begin{proof}[Proof of Lemma~\ref{lem:outer}]
    Lemma~\ref{lem:kkm} reduces upper bounding the excess risk achieved by $\brc{\wh{x}_i},\brc{\wh{w}_i}$ to bounding the main term $\psE{\frac{1}{T}\sum^{T-1}_{i=0} (1 - a_i) \norm{B(x_i - x^*_i)}^2/\tau^2}$ in \eqref{eq:firstopt}, which we do now. Using Fact~\ref{fact:unroll}, for any $i = \ell t + j$ we can write $B(x_i - x^*_i) = BA^j(x_{\ell t} - x^*_{\ell t}) + \sum^j_{s=0} BA^{j - s}(w_{\ell t + s} - w^*_{\ell t + s})$.
    We thus have \begin{align}
        \MoveEqLeft \frac{1}{T}\sum_{i=0}^{T-1} (1 - a_i)\| B(x_i - x^*_i)\|^2 \\
        &= \frac{1}{T} \sum_{\ell,j} (1 - a_{\ell t + j})\norm*{BA^j(x_{\ell t} - x^*_{\ell t}) + \sum^j_{s=0} BA^{j-s}(w_{\ell t + s} - w^*_{\ell t + s})}^2 \\
        &\le \frac{3}{T} \sum_{\ell, j} (1 - a_{\ell t + j}) \left(\norm*{BA^j(x_{\ell t} - x^*_{\ell t})}^2 + \norm*{\sum^j_{s=0}BA^{j-s}w_{\ell t + s}}^2 + \norm*{\sum^j_{s=0}BA^{j-s}w^*_{\ell t +s}}^2\right) \label{eq:subsample_and_noise}
    \end{align}
    We can control the two noise terms on the right by noting that for any $\ell,j$, \begin{equation}
        \norm*{\sum^j_{s = 0} BA^{j-s} w_{\ell t + s}}^2 \le (j + 1)\sum^j_{s=0} \norm{BA^{j-s} w_{\ell t + s}}^2 \lesssim t\rho^2\norm{B}^2\sigma^2(d + \log(T/\delta)),
    \end{equation} where in the last step we used Constraint~\ref{item:dynamics_noise}.
    Because the true process noise $\brc{w^*_i}$ is part of a feasible solution to Program~\ref{program:sos}, from Constraint~\ref{item:many_inliers} we conclude that
    \begin{equation}
        \frac{1}{T}\sum_{\ell,j}(1 - a_{\ell t + j}) \left(\norm*{\sum_s BA^{j-s}w_{\ell t + s}}^2 + \norm*{\sum_s BA^{j-s}w^*_{\ell t + s}}^2\right) \lesssim \eta t\rho^2\norm{B}^2\sigma^2(d + \log(T/\delta))
    \end{equation}
    For the remaining terms in \eqref{eq:subsample_and_noise}, we invoke Constraint~\ref{item:subsample} and the bound on $\norm{\calO_t}$ in \eqref{eq:opnormOt} to get
    \begin{equation}
        \frac{1}{T}\sum_{\ell,j}(1 - a_{\ell t + j})\norm{BA^j(x_{\ell t} - x^*_{\ell t})}^2 \le 1.01\eta\rho^2\left(\alpha + O\left( \norm{B}^2\sqrt{\log(dT/t\delta)/t}\right)\right) \frac{1}{T/t}\sum^{T/t-1}_{\ell=0} \norm{x_{\ell t} - x^*_{\ell t}}^2
    \end{equation}
    from which the lemma follows by substituting the two estimates above into \eqref{eq:subsample_and_noise}. 
\end{proof}

\subsection{Decay of Unobservable Subspace}
\label{sec:old_decay}

In this section we show how to bound our prediction error on the first iterate of every window. Towards proving this, the main result of this subsection is to show that our error in estimating these first iterates decays exponentially over time provided that a certain matrix concentration event holds in every window.

We begin by describing this event. Let $\Pi$ denote the projection to the \emph{observable subspace}, that is, to the subspace of $v\in\R^d$ for which $v^{\top} \calO_t v \ge \thres$ for $\thres \triangleq \frac{\kappa \t}{40000\rho^4}$, where the window size $t$ will be optimized at the end of this section. The matrix concentration that we need to hold in every window is the following:

\begin{lemma}\label{lem:anticonc}
    Suppose $t= \wt{\Omega}\left(\kappa^{-2} \rho^{12} \norm{B}^4\log(dT/\delta)\right)$. Then with probability at least $1 - \delta$ over the randomness of $\brc{a^*_i}$, we have that for all windows $0\le \ell < T/t$, there is a degree-2 SoS proof of the psd inequality
    \begin{equation}
        \sum^{t-1}_{i = 0} a^*_{\ell t + i} a_{\ell t + i} \Pi (A^i)^{\top} B^{\top} B A^i \Pi \succeq \frac{1}{100}\Pi\calO_t \Pi \label{eq:psdlb}
    \end{equation}
    using the constraints of Program~\ref{program:sos}.
\end{lemma}

\begin{proof}
    We will focus on $\ell = 0$ and apply a union bound over $\ell$ at the end. Recall from Lemma~\ref{lem:matrix_conc} that with probability at least $1 - \delta/(2T/t)$ we have \begin{equation}
        \sum^{t-1}_{i=0} a^*_i \Pi (A^i)^{\top} B^{\top} B A^i \Pi \succeq (1 - \eta)\cdot \Pi\calO_t\Pi - O\left(\rho^2\norm{B}^2\sqrt{t\log(dT/t\delta)}\right)\cdot \Pi.
    \end{equation}
    Condition on this event. Write $a^*_i a_i = a^*_i - a^*_i(1 - a_i) \ge a^*_i - (1 - a_i)$, where in the inequality we use Constraint~\ref{item:boolean}. By Constraint~\ref{item:subsample}, we have \begin{equation}
        \sum^{t-1}_{i=0} (1 - a_i) \Pi (A^i)^{\top} B^{\top} B A^i \Pi \preceq \eta\cdot \Pi\calO_t \Pi + O\left(\rho^2\norm{B}^2\sqrt{t\log(dT/t\delta)}\right) \cdot \Pi,
    \end{equation} 
    so because $\Pi \preceq \thres^{-1} \cdot \Pi\calO_t\Pi$ by definition of $\Pi$, we have a degree-2 SoS proof of the inequality
    \begin{equation}
        \sum^{t-1}_{i = 1} a^*_{\ell t + i} a_{\ell t + i} \Pi (A^i)^{\top} B^{\top} B A^i \Pi \succeq \left(1 - 2\eta - O\left(\thres^{-1}\rho^2\norm{B}^2\sqrt{t\log(dT/t\delta)}\right)\right) \Pi\calO_t \Pi.
    \end{equation}
    We would like $t$ to be large enough that the factor on the right-hand side exceeds $1/100$. As $\eta \le 0.49$, it suffices for $\thres \ge O(\rho^2 \norm{B}^2\sqrt{t\log(dT/t\delta)})$. Recalling that $\thres = \frac{\kappa \t}{40000\rho^4}$, we need to take $t \ge \wt{\Omega}\left(\kappa^{-2} \rho^{12} \norm{B}^4 \log(dT/\delta)\right)$. The proof follows by union bounding over $0 \le \ell < T/t$.
\end{proof}

We now turn to showing the main result of this section, namely that provided the event of Lemma~\ref{lem:anticonc} holds, our prediction error on the first iterate of every window decays exponentially over time.

\begin{lemma}\label{lem:error_decay}
    Let pseudoexpectation $\psE{\cdot}$ be the solution to Program~\ref{program:sos}, assuming it is feasible. Provided the event of Lemma~\ref{lem:anticonc} holds, we have
    \[ \psE{\|x_{\ell t} - x^*_{\ell t}\|^2} \le \frac{1}{2}\psE{\|x_{(\ell - 1) t} - x^*_{(\ell - 1) t}\|^2} + O(\rho^6 E_{\mathsf{noise}} t/\kappa), \]
    where $E_{\mathsf{noise}}$ is defined in \eqref{eq:Enoise_def}.
\end{lemma}

Before proving Lemma~\ref{lem:error_decay}, we first show how to use it to conclude the proof of Theorem~\ref{thm:main_TlogT}.

\begin{proof}[Proof of Theorem~\ref{thm:main_TlogT}]
    Take $t$ as in \eqref{eq:tdef}. The events of Lemma~\ref{lem:feasible} and Lemma~\ref{lem:anticonc} hold with probability at least $1 - 2\delta$. By summing the conclusion of Lemma~\ref{lem:error_decay} over the time windows, we get \begin{equation}
        \frac{1}{T/t}\sum^{T/t-1}_{\ell=0} \psE*{\norm{x_{\ell t} - x^*_{\ell t}}^2} \le \frac{1}{T/t}\psE*{\norm{x_0 - x^*_0}^2} + O(\rho^6 E_{\mathsf{noise}} t/\kappa). \label{eq:corgeo}
    \end{equation} 
    Recall that $x^*_0\sim\calN(0,R^2\cdot\Id)$, so by standard concentration, $\norm{x^*_0}^2 \le R^2(d + O(\log(1/\delta)))$ with probability at least $1 - \delta$. $\norm{x^*_0}^2$ is similarly bounded by Constraint~\ref{constraint:bound}. We can thus bound $\psE*{\norm{x_0 - x^*_0}^2}$ by $2R^2(d + O(\log(1/\delta)))$, so plugging this into \eqref{eq:corgeo} and invoking Lemma~\ref{lem:outer}, we conclude the proof of Theorem~\ref{thm:main_TlogT}.
\end{proof}

We now proceed to the proof of Lemma~\ref{lem:error_decay}. The first step is an averaging argument to show that applying $A^t$ to a vector in the unobservable subspace is guaranteed to decrease its norm. 

\begin{restatable}{lemma}{averaging}\label{lem:unobservable-decay}
	For any vector $x\in\R^d$, $\norm{A^t \Pi^{\perp} x}^2 \le \frac{1}{40000\rho^2}\norm{\Pi^{\perp} x}^2$.
\end{restatable}

\begin{proof}
	We have that \begin{align}
		\frac{1}{\t/s}\sum^{\t/s-1}_{j=0} \norm{A^{j s}\Pi^{\perp} x}^2 &= \frac{1}{\t/s}\sum^{\t/s-1}_{j=0} x^{\top}\Pi^{\perp} (A^{j s})^{\top} A^{j s} \Pi^{\perp} x \\
		&\le \frac{1}{\kappa \t} \sum^{\t/s-1}_{j=0} x^{\top} \Pi^{\perp}(A^{j s})^{\top} \calO_s A^{js}\Pi^{\perp} x \\
		&= \frac{1}{\kappa \t} x^{\top}\Pi^{\perp}\calO_t \Pi^{\perp} x \le \frac{1}{40000\rho^4} \norm{\Pi^{\perp} x}^2,
	\end{align} where the third step follows by the first part of Fact~\ref{fact:epochs} and the last step follows by the definition of $\Pi^{\perp}$ and $\thres = \frac{\kappa \t}{40000\rho^4}$. By averaging, there exists there some $0 \le j < t/s$ for which $\norm{A^{js}\Pi^{\perp} x}^2 \le \frac{1}{40000\rho^4}\norm{\Pi^{\perp}x}^2$. The lemma follows by uniform stability.
\end{proof}

We will eventually take $x$ to be the difference between our estimate of an iterate at the beginning of a window and the ground truth. Informally, this will tell us that over the course of a window of size $t$, the component of the error that started in the unobservable subspace has decayed.

What about the component of the error that started in the \emph{observable} subspace? By uniform stability, it cannot increase by too much, but unlike the unobservable component, it need not decay. This brings us to the win-win argument at the core of the proof of Lemma~\ref{lem:observable-case-analysis}: when the observable component does not decay, we can still relate it to the observable component in the previous time window just by uniform stability, and then bound this by a tiny fraction of the {unobservable component in the previous time window}!

\begin{lemma}\label{lem:pi-piperp-ineq}
    With probability at least $1 - \delta$, the following holds true for for every window $0\le \ell < T/t$, for $q \triangleq x_{\ell t} - x^*_{\ell t}$, all for all $i < t$.
    There is a degree-4 SoS proof from the constraints in Program~\ref{program:sos} that
    \begin{equation}
        a^*_{\ell t + i} a_{\ell t + i}\| BA^i \Pi q\|^2 \leq 4a^*_{\ell t + i} a_{\ell t + i}\| BA^i \Pi^\perp q\|^2 + O(E_{\mathsf{noise}}),
    \end{equation}    
    where recall that $E_{\mathsf{noise}}$ is defined in \eqref{eq:Enoise_def}.
\end{lemma}
\begin{proof}
Without loss of generality we can assume $\ell = 0$.
For any $i < t$, we have the following sequence of inequalities in degree-4 SoS \begin{align}
        a_i a^*_i \norm*{BA^i(\Pi + \Pi^{\perp})q}^2 &= a_i a^*_i \norm*{BA^iq}^2 \\
        &= a_i a^*_i\norm*{(y_i - Bx^*_i) - (y_i - Bx_i) + (Bx^*_i - BA^i x^*_0) - (Bx_i - BA^i x_0)}^2 \\
        &\le 3\norm{v^*_i}^2 + 3\norm{v_i}^2 + 3\norm*{\sum^i_{s=1} BA^{i - s} (w_s - w^*_s)}^2 \le E_{\mathsf{noise}}, \label{eq:pipiperp}
    \end{align}
    where we used the constraints and event of Lemma~\ref{lem:feasible} in the last step (which holds with probability at least $1 - \delta$).
    So the lemma follows by applying Fact~\ref{fact:sos_simple} to $\epsilon \triangleq E_{\mathsf{noise}}$, $v_1 = a^*_i a_i BA^i \Pi q $, and $v_2 = a^*_i a_i BA^i \Pi^\perp q $.
\end{proof}

\begin{lemma}\label{lem:observable-case-analysis}
    Let pseudoexpectation $\psE{\cdot}$ be the solution to Program~\ref{program:sos}, assuming it is feasible. Provided the events of Lemma~\ref{lem:anticonc} and Lemma~\ref{lem:pi-piperp-ineq} hold, then at least one of the following holds for every window $0\le \ell < T/t$ for $q \triangleq x_{\ell t} - x^*_{\ell t}$:
    \begin{enumerate}
        \item (Observable component decays) 
        \begin{equation}
            \psE*{\norm{A^{t}\Pi q}^2} 
            \le \frac{1}{10}\psE*{\| \Pi q \|^2}.
        \end{equation}
        \item (Observable error bounded by unobservable error)
        \begin{equation}
            \psE*{\norm{\Pi q}^2} \le \frac{1}{10\rho^2}\psE*{\norm{\Pi^{\perp} q}^2} + O\left(E_{\mathsf{noise}}\,\rho^4/\kappa\right).
        \end{equation}
    \end{enumerate}
\end{lemma}

\begin{proof}
    Without loss of generality we can assume $\ell = 0$. In addition to the lower bound of \eqref{eq:psdlb}, we also have a degree-2 SoS proof of the upper bound
    \begin{equation}\label{eqn:unobs-psd}
        \sum_{i=0}^{T-1} a^*_i a_i \Pi^\perp (A^i)^T B^T B A^i \Pi^\perp \preceq  \Pi^\perp \mathcal{O}_t \Pi^\perp \preceq \thres \cdot I,
    \end{equation}
    where in the first step we used that $a_i a^*_i \le 1$ by Constraint~\ref{item:boolean} and in the second step we used the definition of $\Pi$.
    
    For convenience, define $q\triangleq x_0 - x^*_0$. We proceed by casework on whether there is a gap between $\psE*{(\Pi q)^{\top} \calO_t (\Pi q)}$ and $\psE{\thres\norm{\Pi q}^2}$:
    
    \noindent{\textbf{Case 1}}: $\psE*{\| \Pi q \|^2} \ge \psE*{\frac{1}{4000\rho^2\thres}\sum_{i=0}^{t-1}  \| BA^i \Pi q\|^2}$.
    
    The analysis for this case is very similar to the analysis in Lemma~\ref{lem:unobservable-decay}. We have 
    \begin{align*}
        \psE*{\| \Pi q \|^2} 
        &\geq \psE*{\frac{1}{4000\rho^2\thres}\sum_{i=0}^{t-1}  \| BA^i \Pi q\|^2} \\
        &= \psE*{\frac{1}{4000\rho^2\thres}\sum_{j = 0}^{t/s-1}  q^{\top}\Pi {A^{js}}^{\top} \calO_s A^{js} \Pi q} \geq \psE*{\frac{10\rho^2 s}{t}\sum_{j = 0}^{t/s-1}  \|A^{js} \Pi q\|^2},
    \end{align*}
    where in the last step we used the definition of $\thres$ and the assumption that $\lambda_{\min}(\calO_s) \ge \kappa s$. Rearranging, we obtain 
    \begin{equation}
        \frac{1}{10\rho^2} \psE{\| \Pi q \|^2} \geq \frac{1}{t/s}\sum_{j=0}^{t/s-1} \psE*{\|A^{js} \Pi q\|^2}.
    \end{equation}
    Therefore, there exists some index $0 \le j < t/s$ for which $\psE*{\norm{A^{js}\Pi q}^2} \le \frac{1}{10\rho^2}\psE*{\norm{\Pi q}^2}$. By uniform stability, we obtain the first desired outcome in the lemma statement.

    \noindent{\textbf{Case 2}}: $ \psE*{\| \Pi q \|^2} \le \psE*{\frac{1}{4000\rho^2\thres}\sum_{i=0}^{t-1}  \| BA^i \Pi q\|^2}$.

    In this case we invoke \eqref{eq:psdlb} to obtain
    \begin{equation} \label{eqn:fix-main}
    \psE*{\| \Pi q \|^2} \leq \psE*{\frac{1}{4000\rho^2\thres}\sum_{i=0}^{t-1}\| BA^i \Pi q\|^2} \leq \psE*{\frac{1}{40\rho^2\thres}\sum_{i=0}^{t-1} a^*_i a_i\| BA^i \Pi q\|^2}.
    \end{equation}
    Recall from Lemma~\ref{lem:pi-piperp-ineq} that we have a degree-4 SoS proof of
    \begin{equation}
        a^*_i a_i\| BA^i \Pi q\|^2 \leq 4a^*_i a_i\| BA^i \Pi^\perp q\|^2 + O(E_{\mathsf{noise}}).
    \end{equation}
    Summing this inequality over $i < t$ and taking pseudo-expectations, we get
    \[\psE*{\sum_{i=0}^{t-1} a^*_i a_i\| BA^i \Pi q\|^2} \le \psE*{4\sum_{i=0}^{t-1}a^*_i a_i\| BA^i \Pi^\perp q\|^2} + O(E_{\mathsf{noise}}t). \]
    Substituting this back into the main bound \eqref{eqn:fix-main}, we get
    \begin{align}
    \psE*{\| \Pi q \|^2} &\leq  \psE*{\frac{1}{10\rho^2\thres}\sum_{i=0}^{t-1} a^*_i a_i\| BA^i \Pi^\perp q\|^2} + O\left(\frac{E_{\mathsf{noise}} t}{30\thres}\right) \\
    &\le \psE*{\frac{1}{10\rho^2}\norm{\Pi^{\perp} q}^2} + O\left(\frac{E_{\mathsf{noise}} t}{30\thres}\right),
    \end{align} where in the last step we used \eqref{eqn:unobs-psd}. Unpacking the definition of $\thres$, we arrive at the second desired bound.
\end{proof}

We are now ready to prove Lemma~\ref{lem:error_decay}:

\begin{proof}[Proof of Lemma~\ref{lem:error_decay}]
    By the SoS triangle inequality,
    \begin{align*}
    \MoveEqLeft\|x_{\ell t} - x^*_{\ell t}\|^2 \\
    &\le 2 \|A^t(x_{(\ell - 1)t} - x^*_{(\ell - 1)t})\|^2 + O(t^2 \sigma^2(d + \log(T/\delta))) \\
    &\le 4 \|A^t\Pi(x_{(\ell - 1)t} - x^*_{(\ell - 1)t})\|^2 + 4\|A^t\Pi^{\perp} (x_{(\ell - 1)t} - x^*_{(\ell - 1)t})\|^2 + O(t^2 \sigma^2(d + \log(T/\delta))) \\
    &\le 4 \|A^t\Pi(x_{(\ell - 1)t} - x^*_{(\ell - 1)t})\|^2  + (1/10000) \|\Pi^{\perp} (x_{(\ell - 1)t} - x^*_{(\ell - 1)t})\|^2 + O(t^2 \sigma^2(d + \log(T/\delta)))
    \end{align*}
    where we used Constraint~\ref{item:dynamics_noise} and triangle inequality in the first inequality, SoS triangle inequality in the second inequality, and Lemma~\ref{lem:unobservable-decay} in the third inequality. 
    
    Now based on Lemma~\ref{lem:observable-case-analysis} applied to $\ell - 1$, we consider the following two cases:
    
    \noindent\textbf{Case 1}: Observable component decays, that is, we have
    \[ \psE*{\norm{A^{t}\Pi (x_{(\ell - 1) t} - x^*_{(\ell -1) t})}^2} \leq \frac{1}{10} \psE*{\| \Pi (x_{(\ell -1) t} - x^*_{(\ell - 1) t}) \|^2}. \]
        Then we can argue that the error at time $\ell t$ is a small fraction of the error at time $(\ell - 1)t$ because both the observable and unobservable components of the error at time $(\ell - 1)t$ have decayed over $t$ steps. Formally:
        \begin{align*}
            \MoveEqLeft\psE{\|x_{\ell t} - x^*_{\ell t}\|^2}  \\
            &\le \psE*{4 \|A^t\Pi(x_{(\ell - 1)t} - x^*_{(\ell - 1)t})\|^2  + (1/10000) \|\Pi^{\perp} (x_{(\ell - 1)t} - x^*_{(\ell - 1)t})\|^2} + O(t^2 \sigma^2(d + \log(T/\delta))) \\
            &\le  \psE*{(2/5)\|\Pi (x_{(\ell - 1)t} - x^*_{(\ell - 1)t})\|^2 + (1/10000) \|\Pi^{\perp} (x_{(\ell - 1)t} - x^*_{(\ell - 1)t})\|^2} + O(t^2 \sigma^2(d + \log(T/\delta))) \\
            &\le (1/2) \psE*{\|x_{(\ell - 1)t} - x^*_{(\ell - 1)t}\|^2} + O(t^2 \sigma^2(d + \log(T/\delta))).
        \end{align*}
        where in the last step we used the Pythagorean Theorem.
    
    \noindent\textbf{Case 2}: Observable error bounded by unobservable error, that is
        \begin{equation}
            \psE*{\norm{\Pi q}^2} \le \frac{1}{10\rho^2}\psE*{\norm{\Pi^{\perp} q}^2} + O(E_{\mathsf{noise}}\rho^4/\kappa)   \label{eq:usecase2}
        \end{equation}
        where $q \triangleq x_{(\ell - 1)t} - x^*_{(\ell - 1)t}$. Then we can argue that the error at time $\ell t$ is a small fraction of the error at time $(\ell - 1)t$ as follows. As in Case 1, the unobservable error at time $(\ell - 1)t$ has decayed. As discussed above, the observable error at time $\ell t$ might even be bigger than the observable error at time $(\ell - 1)t$, but it can't be much bigger because of uniform stability. On the other hand, the latter is bounded by a small fraction of the \emph{unobservable} error at time $(\ell - 1)t$. This lets us conclude that the overall error at time $\ell t$ is bounded even by the unobservable error at time $(\ell - 1)t$. Formally,
        \begin{align*}
            \MoveEqLeft \psE{\|x_{\ell t} - x^*_{\ell t}\|^2} 
            \\
            &\le \psE*{4 \|A^t\Pi(x_{(\ell - 1)t} - x^*_{(\ell - 1)t})\|^2  + (1/10000) \|\Pi^{\perp} (x_{(\ell - 1)t} - x^*_{(\ell - 1)t})\|^2} + O(t^2\sigma^2(d + \log(T/\delta))) \\
            &\le 4 \rho^2 \|\Pi(x_{(\ell - 1)t} - x^*_{(\ell - 1)t})\|^2  + (1/10000) \|\Pi^{\perp} (x_{(\ell - 1)t} - x^*_{(\ell - 1)t})\|^2 + O(t^2\sigma^2(d + \log(T/\delta))) \\
            &\le \frac{1}{2}\psE*{\norm{q}^2} + O(\rho^6 E_{\mathsf{noise}} t/\kappa).
        \end{align*}
        where $C < 1$ is an absolute constant and in the last step we used \eqref{eq:usecase2} and absorbed $O(t^2\sigma^2(d + \log(T/\delta)))$ into $O(\rho^6 E_{\mathsf{noise}}t/\kappa)$.
        Since we showed the desired conclusion in both cases, the proof is complete.
\end{proof}

\subsection{Confidence Band Recovery}

Here we note that as a consequence of Theorem~\ref{thm:main_TlogT}, our estimate $\brc{\psE{x_i}}$ of the trajectory is actually pointwise $O(\log T)$-close to the true trajectory at all time steps, except for a $o(1)$ proportion of time close to time zero. This will be useful in the next section when we use this as a warm start for a second sum-of-squares relaxation that will achieve excess risk \emph{doubly logarithmic} in $T$.

\begin{corollary}\label{cor:warmstart}
    Let pseudoexpectation $\psE{\cdot}$ be the solution to Program~\ref{program:sos}, assuming it is feasible. Then provided the event of Lemma~\ref{lem:anticonc} holds, for all $0\le i < T$ the estimates $\brc{\psE{x_i}}$ satisfy
    \begin{equation}
        \norm{\psE{x_i} - x^*_i} \lesssim \frac{\rho}{2^{\ell(i)/2}}R\left(\sqrt{d} + \sqrt{\log(1/\delta)}\right) + O(\rho^4 E^{1/2}_{\mathsf{noise}}t^{1/2}/\kappa^{1/2})
    \end{equation}
\end{corollary}

\begin{proof}
    By unrolling Lemma~\ref{lem:error_decay} and using the fact that $\psE{\norm{x_0 - x^*_0}^2} \le 2R^2(d + O(\log(1/\delta)))$, we conclude that
    \begin{equation}
        \psE{\|x_{\ell t} - x^*_{\ell t}\|^2} \le \frac{1}{2}\psE{\|x_{(\ell - 1) t} - x^*_{(\ell - 1) t}\|^2} + O(\rho^6 E_{\mathsf{noise}}t/\kappa) \le \frac{1}{2^{\ell - 1}}\cdot 2R^2(d + O(\log(1/\delta))) + O(\rho^6 E_{\mathsf{noise}}t/\kappa),
    \end{equation}
    so by Lemma~\ref{lem:psE-CS}, we have that
    \begin{equation}
        \|\psE{x_{\ell t}} - x^*_{\ell t}\|^2 \le 4R^2(d + O(\log(1/\delta)))/2^{\ell} + O(\rho^6 E_{\mathsf{noise}}t/\kappa).
    \end{equation}
    It remains to bound the error on the iterates \emph{within} each window. For any $0 < i < t$, Constraint~\ref{item:dynamics} and Fact~\ref{fact:unroll} imply that 
    \begin{align}
        \MoveEqLeft\norm*{\psE{x_{\ell t + i}} - x^*_{\ell t +i}} \\
        &= \norm*{A^i (\psE{x_{\ell t + i}} - x^*_{\ell t + i})  + \sum^i_{j=1} A^{i-j} (\psE{w_{\ell t + j}} - w^*_{\ell t + j})} \\
        &\le \rho\norm{\psE{x_{\ell t +i}} - x^*_{\ell t +i}} + \sum^i_{j=1} \rho\norm*{\psE{w_{\ell t + j}} - w^*_{\ell t+j}} \\
        &\lesssim \rho \left(\frac{1}{2^{\ell/2}}R\left(\sqrt{d} + \sqrt{\log(1/\delta)}\right) + O(\rho^3 E^{1/2}_{\mathsf{noise}}/\kappa^{1/2})\right) + t\rho\left(\sigma\sqrt{d} + \sigma\sqrt{\log(T/\delta)}\right). \\
        &\lesssim \rho \left(\frac{1}{2^{\ell/2}}R\left(\sqrt{d} + \sqrt{\log(1/\delta)}\right) + O(\rho^3 E^{1/2}_{\mathsf{noise}}/\kappa^{1/2})\right)
    \end{align}
    as desired.
\end{proof}

%% file: TloglogT.tex
\newcommand{\concprob}{\delta_1}

\section{Achieving \texorpdfstring{$\log\log T$}{loglog T} Excess Risk}
\label{sec:TloglogT}


In this section, we complete the proof of our main theorem by showing how to refine the poly-logarithmic excess risk guarantee of the previous section to get doubly logarithmic excess risk. At a very high level, the bottleneck in the analysis from the previous section arises when we argue that for most windows, the uncorrupted time steps subsample the observability matrix. The issue is that for any window, with probability $\delta$ the error from matrix concentration corresponding to the right-hand side of \eqref{eq:deviation} will scale worse than $\sqrt{\log(d/\delta)}$. In the previous section, we simply took $\delta = O(1/T)$ so that by a union bound, with high probability this will not happen for any window, hence our poly-logarithmic excess risk bound.

In this section, we instead take $\delta = O(1/\polylog T)$ so that matrix concentration will now fail on $O(1/\polylog T)$ fraction of the windows, but over the windows where matrix concentration holds, the risk scales with $\log\log T$ instead of $\log T$. To handle the ``bad'' windows where matrix concentration fails, we will exploit the fact that by Corollary~\ref{cor:warmstart}, the estimate from the previous section is pointwise $O(\polylog T)$-close to the true trajectory (i.e. it forms a high probability confidence band). By folding this into a new sum-of-squares relaxation and introducing additional indicator variables corresponding to the events that matrix concentration holds in particular windows, we are able to bound the contribution of the bad windows to the risk by $O(\polylog(T)/\polylog(T)) = O(1)$. Formally, we show the following bound:

\begin{theorem}\label{thm:main_TloglogT}
    For any $\eta \le 0.49$, there is a polynomial-time algorithm that, given the corrupted observations $\brc{\wt{y}_i}$, outputs a trajectory $\brc{\wh{x}_i}$ and steps $\brc{\wh{w}_i}$ for which $\wh{x}_i = A\wh{x}_{i - 1} +\wh{w}_i$ for every $i\in [T]$, and for which
    \begin{multline}
        \avgT\left(a^*_i \norm{B\wh{x}_i - y_i}^2/\tau^2 + \norm{\wh{w}_i}^2/\sigma^2\right) + \frac{\|\wh{x}_0\|^2}{R^2 T} - \OPT \\
        \lesssim \poly(s,\alpha,1/\kappa,\rho,\|B\|,\log d,\log(1/\delta)) \brk*{ \eta\log(1/\eta)\left(m + d(\sigma^2/\tau^2)\log\log T\right) + \eta^{1/2}(R^2 d/\tau^2)T^{-0.49}}
    \end{multline}
    provided $T \ge \Max{\wt{\Omega}(\eta^{-1}\kappa^{-2}\rho^{12}\norm{B}^4\log d)}{\Omega(t^2\log^k(1/\delta))})$.
\end{theorem}

\subsection{Sum-of-Squares Relaxation}

In this section we define a new sum-of-squares program to exploit Corollary~\ref{cor:warmstart}. Recall that the estimate from the previous section, call it $\brc{x'_i}$, satisfies
\begin{equation}
    \norm{x'_i - x^*_i}^2 \lesssim \frac{\rho^2}{2^{\ell(i)}}R^2(d + \log(1/\delta)) + \epsgeo,
\end{equation}
where $\epsgeo$ satisfies
\begin{equation}
    \epsgeo \triangleq O\left(\frac{\rho^8 t_{\mathsf{pre}}}{\kappa}\cdot \brk*{\tau^2\left(m + \log(T/\delta)\right) + \sigma^2 \left(d +\log(T/\delta)\right)\cdot \rho^2 t_{\mathsf{pre}} \norm{B}^2}\right) \label{eq:eps_geo2}
\end{equation}
for $t_{\mathsf{pre}}\triangleq \Max{s}{\wt{\Theta}(\kappa^{-2}\rho^{12}\norm{B}^4\log(dT/\delta))}$ from \eqref{eq:tdef} denoting the window size parameter from the previous section. 
Note that $t_{\mathsf{pre}}$ scales logarithmically in $T$ and $\epsgeo$ therefore scales as $\log^3 T$. We now define our program:

\begin{program}\label{program:sosv2} 
	Let $\brc{y_i}$ be the observations we are given, and let $\delta_1 > 0$ and window size $t\in\mathbb{N}$ be parameters to be tuned later. The program also takes in as parameters a sequence $\brc{x'_0,\ldots,x'_{T-1}}$ of vectors in $\R^d$, corresponding to an estimate of the trajectory.
	
	The program variables are $\brc{x_i}$, $\brc{w_i}$, $\brc{v_i}$, and $\brc{a_i}$ as before, as well as Boolean variables $\brc{b_{\ell}}^{T/t-1}_{\ell = 0}$ (indicators for windows where matrix concentration holds), and the constraints are that for all $0 \le i < T$ and $0 \le \ell < T/t$:
	
	\vspace{0.3cm}
	\noindent\emph{Constraints from Program~\ref{program:sos}}
    \begin{enumerate}
    	\item $a_i^2 = a_i$\label{constraint2:a_boolean}
    	\item $x_i = Ax_{i-1} + w_i$ \label{constraint2:dynamics}
      	\item $a_i(y_i - Bx_{i-1} - v_i) = 0$ \label{constraint2:measurements}
      	\item $\sum_{i=0}^{T-1} a_i \ge (1-1.01\eta)T$ \label{constraint2:many_inliers}
      	\item $\norm{x_0}^2 \le R^2(d + O(\log(1/\delta)))$ \label{constraint2:bound}
    \end{enumerate}
    
    \noindent\emph{Random windows where subsampling succeeds/fails}
    \begin{enumerate}
  	      \setcounter{enumi}{5}
      	\item $b_{\ell}^2 = b_{\ell}$ \label{constraint2:b_boolean}
      	\item $\sum_{\ell=0}^{T/t-1} b_{\ell} \ge (1 - \concprob)\cdot T/t$ \label{constraint2:many_b}
      	\item $\sum_{i=0}^{T-1} (1 - b_{\ell(i)})(1 - a_i) \le \eta\delta_1$\label{constraint2:many_ab}
    \end{enumerate}
    \noindent\emph{Confidence band given by $\brc{x'_i}$}
    \begin{enumerate}
        \setcounter{enumi}{8}
      	\item $\norm{x_i - x'_i}^2 \le \epsgeo + O(\rho^2 R^2(d + \log(1/\delta))/2^{\ell(i)})$ \label{constraint2:funnel}
    \end{enumerate}
    \noindent\emph{Process and observation noise bounded on average}
    \begin{enumerate}
        \setcounter{enumi}{9}
        \item $\frac{1}{T}\sum_{\ell<T/\t, j<\t} (1 - a_{\ell t + j}) \norm{\sum^j_{i=1}BA^{j - i}w_{\ell\t+i}}^2 \le O\left(\eta^{1-2/k}\t\alpha\sigma^2\rho^2 m k\right)$ \label{constraint2:totalnoise}
      	\item $\frac{1}{T}\sum_{\ell<T/\t} \norm{\sum^t_{i=1}A^{t-i}w_{\ell\t+i}}^2 \le O\left(\sigma^2\rho^2 d\right)$ \label{constraint2:totalnoise_dumb} 
        \item $\frac{1}{T}\sum_{i = 0}^{T-1} \norm{v_i}^2 = O(\tau^2(m + \log(2/\delta)/T))$\label{constraint2:measurements_noise}
        \item $\sum_{i=0}^{T-1} (1 - a_i)\norm{v_i}^2 \le O(mk\tau^2\cdot\eta^{1-2/k})$ \label{constraint2:measurements_noise_corrupted}
    \end{enumerate}
    \noindent\emph{Subsampling with confidence $1 - \delta_1$ in each window}
    \begin{enumerate}
        \setcounter{enumi}{13}
        \item $b_{\ell} \sum^{\t- 1}_{j=0} (1 - a_{\ell t + j}) (A^j)^{\top} B^{\top} B A^j \preceq b_{\ell}\cdot \left(\eta\cdot \calO_{\t} + O\left(\rho^2\norm{B}^2\sqrt{\t\log(d/\delta_1)}\right)\cdot \Id\right)$ \label{constraint2:subsample}
	\end{enumerate}
	The program objective is to minimize
	\begin{equation}
		\min\psE*{\frac{1}{T} \sum^{T-1}_{i=0} \left(a_i \norm{Bx_i - y_i}^2 / \tau^2 + \norm{w_i}^2 / \sigma^2\right)}\label{eq:obj2}
	\end{equation}
	over degree-$4$ pseudoexpectations satisfying the above constraints. 
\end{program}

As Program~\ref{program:sosv2} is considerably more involved than Program~\ref{program:sos} (note in particular the new Constraints~\ref{constraint2:b_boolean}-\ref{constraint2:subsample}), here we spend some time to highlight the key differences.
First notice that we now have an additional set of Boolean variables $\brc{b_{\ell}}$ indicating the windows whose uncorrupted time steps the program deems to have properly subsampled the observability matrix. In Constraint~\ref{constraint2:subsample} we choose the error in subsampling to scale with $\sqrt{\log(d/\delta_1)}$ for a parameter $\delta_1$ which we will take to scale logarithmically in $T$ so that $1 - \delta_1$ fraction of windows will incur at most this level of subsampling error (Constraint~\ref{constraint2:many_b}). We will additionally insist that the fraction of time steps which are simultaneously corrupted and belong to a window where matrix concentration fails is at most roughly $\eta\delta_1$ (Constraint~\ref{constraint2:many_ab}).

Also note that we have introduced an additional set of parameters, namely the sequence $\brc{x'_0,\ldots,x'_{T-1}}$ which we will ultimately take to be the trajectory estimate from the algorithm in Section~\ref{sec:logT}. In Constraint~\ref{constraint2:funnel} we encode the fact that this estimate is pointwise $\log T$-close to the true trajectory by Corollary~\ref{cor:warmstart}.

Finally, in lieu of constraining $\norm{v_i}^2,\norm{w_i}^2$ to be bounded for all $i$, we simply constrain them to be small on average in an appropriate sense captured by Constraints~\ref{constraint2:totalnoise}, \ref{constraint2:totalnoise_dumb}, \ref{constraint2:measurements_noise}, \ref{constraint2:measurements_noise_corrupted}. These four constraints may appear somewhat \emph{ad hoc}, but at a high level the reason for their inclusion is that in various parts in the proof of Theorem~\ref{thm:main_TloglogT}, we need to upper bound some sum over noise terms across all time steps or all time windows. In Section~\ref{sec:logT} we handled these sums by constraining their individual summands to be bounded, but this incurred a factor of $\log T$ through the union bound. Here we are being more careful and simply constraining the full sums to be bounded, leading to constraints in Program~\ref{program:sosv2} that are somewhat harder to parse but yield better error bounds.

We give a full description of our algorithm in Algorithm~\ref{alg:sos} below.

\begin{algorithm2e}
\DontPrintSemicolon
\caption{\textsc{SoSKalman}($\brc{y_i},\delta$)}
\label{alg:sos}
	\KwIn{Observations $y_0,\ldots,y_{T-1}$, failure probability $\delta$}
	\KwOut{Estimate $\wh{x}_0,\ldots,\wh{x}_{T-1}$ for the trajectory}
	    $\delta_1 \gets \Theta(\log(1/\delta)/\log^3 T)$.\;
	    $t_{\mathsf{pre}} \gets \Max{s}{\wt{\Theta}(\kappa^{-2}\rho^{12}\norm{B}^4\log(dT/\delta)))}$.\;
	    $t \gets \Max{s}{\wt{\Theta}(\kappa^{-2}\rho^{12}\norm{B}^4\log(d/\delta_1)))}$.\;
	    Let $\wt{\mathbb{E}}_{\mathsf{pre}}[\cdot]$ be the pseudoexpectation solving Program~\ref{program:sos} with window size $t_{\mathsf{pre}}$.\;
	    $x'_i \gets \wt{\mathbb{E}}_{\mathsf{pre}}[x_i]$ for every $0 \le i < T$.\;
        Let $\psE{\cdot}$ be the pseudoexpectation solving Program~\ref{program:sosv2} with window size $t$ and sequence $\brc{x'_0,\ldots,x'_{T-1}}$.\;
        $\wh{x}_i\gets \psE{x_i}$ for every $0 \le i < T$.\;
        Output $\brc{\wh{x}_0,\ldots,\wh{x}_{T-1}}$.\;
\end{algorithm2e}

\subsection{Feasibility of Oracle Kalman Smoother}

The goal for the this subsection will be to prove that Program~\ref{program:sosv2} has a feasible solution corresponding the output of the oracle Kalman smoother. 
To show this, it suffices by the argument of Lemma~\ref{lem:feasible} to show that the ground truth $x^*$ is feasible with high probability:

\begin{lemma}\label{lem:feasible2}
    With probability at least $1 - \delta$, the ground truth $(x^*_i, w^*_i, v^*_i, a^*_i, b^*_{\ell})$ is feasible for Program~\ref{program:sosv2}
    provided $T \ge \Omega(\Max{\brc*{(t/\eta\delta_1)\log(1/\delta)}}{\brc*{t^2\log^k(1/\delta)}})$.
\end{lemma}

To prove this, we first collect a number of tail bounds. We introduce random variables $b^*_{\ell}$ which measure failure of concentration for the ground truth. Formally, we define
\[ b^*_{\ell} = \bone*{\sum^{\t - 1}_{j=0} (1 - a^*_{\ell t + j}) (A^j)^{\top} B^{\top} B A^j \preceq \eta\cdot \calO_{\t} + O\left(\rho^2\norm{B}^2 \sqrt{ \t \log(d/\concprob)}\right)\cdot \Id}. \]
Next we establish concentration for the sum of the $b^*_{\ell}$, showing that there are many ``good'' windows with high probability. Note that in the following Lemma, we require $\delta_1$ is at least size $\Omega((t/T) \log(1/\delta))$ for this concentration argument to work, but we actually will apply the Lemma with a larger value of  $\delta_1$ in our analysis. Informally, picking a relatively large $\delta_1$ means that concentration fails for more windows (which we can handle using our confidence band technique) but we get tighter concentration bounds for the remaining ``good'' windows. 

\begin{lemma}\label{lem:bstar_conc}
With probability at least $1 - \delta$,
\[ \sum_{\ell = 0}^{T/t-1} b^*_{\ell} \ge (1 - \concprob) (T/t) \]
provided $T = \Omega((t/\delta_1) \log(1/\delta))$.
\end{lemma}

\begin{proof}
    First, observe that the $b^*_{\ell}$ are i.i.d. Bernoulli random variables, since the only randomness in their definition is the $a^*_j$ which are i.i.d. and the windows do not overlap. Next, observe that
    \begin{equation}
        \E{b^*_{\ell}} = \Pr*{\sum^{\t - 1}_{j=0} (1 - a^*_{\ell t + j}) (A^j)^{\top} B^{\top} B A^j \preceq \eta\cdot \calO_{\t} + O\left(O(\rho^2\norm{B}^2 \sqrt{ \t \log(d/\concprob)})\right)\cdot \Id} \ge 1 - \concprob/2 \label{eq:bexp}
    \end{equation} for any $\ell < T/t$
    by Lemma~\ref{lem:matrix_conc}. Therefore by Bernstein's inequality,
    \[ \Pr*{\sum_{\ell = 0}^{T/t - 1} b^*_{\ell} < (1 - \concprob) (T/t)} \le \exp\left(-\Omega\left(\frac{\concprob^2 (T/t)^2}{\concprob (T/t) + (1/3) \concprob (T/t)}\right)\right) = \exp\left(-\Omega(\concprob T/t)\right) \le \delta, \]
    provided that $\concprob = \Omega((t/T) \log(1/\delta))$.
\end{proof}

\begin{lemma}\label{lem:many_ab}
    With probability at least $1 - \delta$,
    \begin{equation}
        \frac{1}{T}\sum^{T-1}_{i=0} (1 - b^*_{\ell(i)})(1 - a^*_i) \le \eta\delta \label{eq:avab}
    \end{equation}
    provided $T = \Omega((t/\eta\concprob)\log(1/\delta))$.
\end{lemma}

\begin{proof}
    For any $\ell < T/t$, consider the random variable $Z_{\ell} \triangleq \frac{1}{t}\cdot (1-b^*_{\ell})\sum^{t-1}_{j=0} (1 - a^*_{\ell t + j})$. Note that $|Z_{\ell}| \le 1$ with probability 1, and $\E{Z_{\ell}} \le \eta\delta_1/2$ by \eqref{eq:bexp}. Furthermore \begin{equation}
        \Var{Z_{\ell}} \le \E{Z_{\ell}^2} \le \left(\eta + 1/t\right)\eta\delta_1/2.
    \end{equation}
    As the left-hand side of \eqref{eq:avab} is $Z\triangleq \frac{1}{T/t}\sum^{T/t-1}_{\ell = 0} Z_{\ell}$, by Bernstein's we have
    \begin{equation}
        \Pr*{\frac{1}{T/t}\sum^{T/t-1}_{\ell = 0} Z_{\ell} > \eta\delta} \le \exp\left(-\Omega\left(\frac{\eta^2\delta_1^2(T/t)}{\eta\delta_1 + \eta\delta_1(\eta + 1/t)} \right)\right) \le \exp\left(-\Omega(\eta\delta_1T/t)\right) \le \delta,
    \end{equation}
    provided $T = \Omega((t/\eta \concprob)\log(1/\delta))$.
\end{proof}

The next three lemmas will allow us to show that the ground truth trajectory $\brc{x^*_i}$ and noise $\brc{w^*_i, v^*_i}$ satisfy Constraints~\ref{constraint2:totalnoise} and \ref{constraint2:totalnoise_dumb}.

\begin{lemma}\label{lem:triplesum}
	For any $\delta > 0$, if $T \ge \t^2 \log^k(1/\delta)$ then 
	\begin{equation}
		\left(\frac{1}{T}\sum^{T/\t - 1}_{\ell = 0}\sum^{t-1}_{j = 0} \norm*{\sum^j_{i=1}BA^{j-i} w^*_{\ell\t+i}}^k\right)^{2/k} \le O\left(\t\alpha\sigma^2\rho^2 m k\right)
	\end{equation} with probability at least $1 - \delta$.
\end{lemma}

\begin{proof}
	 We first bound the expectation \begin{align}
		\E*{\norm*{\sum^j_{i=1} BA^{j-i} w^*_{\ell\t + i}}^k} &= \sigma^k \cdot \E[g\sim\calN(0,\calO_{j-1})]{\norm{g}^k} \\
		&\le \sigma^k \cdot \E[g\sim\calN(0,\calO_{\t})]{\norm{g}^k}\\
		&\le (\t\alpha\sigma^2\rho^2)^{k/2}\cdot\E[g\sim\calN(0,\Id_m)]{\norm{g}^k} \\
		&\le O(\t\alpha\sigma^2\rho^2 m k)^{k/2} \label{eq:firstmoment},
	\end{align} where the second and third steps follow by the fact that $\calO_{j-1} \preceq \calO_{\t}\preceq \t\alpha\rho^2$ (by Fact~\ref{fact:epochs}) so that the norm of a sample from $\calN(0,\calO_{j-1})$ is stochastically dominated by that of one from $\calN(0,\t\alpha\sigma^2\rho^2\cdot \Id)$.

	Next for any $\ell$ we crudely upper bound the variance of $\sum^{\t-1}_{j=0}\norm{\sum^j_{i=1}BA^{j-i}w^*_{\ell\t+i}}^k$ by its second moment:
	\begin{equation}
		\E*{\left(\sum^{\t-1}_{j=0}\norm*{\sum^j_{i=1} BA^{j-i} w^*_{\ell\t + i}}^k\right)^2} \le \sigma^{2k}\t\sum^{\t-1}_{j=0}\E*{\norm*{\sum^j_{i=1}BA^{j-i}w^*_{\ell\t + i}}^{2k}} \le t^2 \cdot O(\t\alpha\sigma^2\rho^2 m k)^k, \label{eq:secondmoment}
	\end{equation} where in the last step we used \eqref{eq:firstmoment}.

	Putting \eqref{eq:firstmoment} and \eqref{eq:secondmoment}, we conclude that for the random variable $Z \triangleq \frac{1}{T}\sum_{\ell ,j}\norm*{\sum^j_{i=1} BA^{j-i} w^*_{\ell\t + i}}^k$, we have \begin{equation}
		\E{Z} \le O(\t\alpha\sigma^2\rho^2 dk)^{k/2} \qquad \Var{Z} \le \frac{\t^2}{T}\cdot O(\t\alpha\sigma^2 \rho^2 dk)^k,
	\end{equation}
	so by Lemma~\ref{lem:polyconc} we conclude that $Z \le (1 + \t\log^{k/2}(1/\delta)/\sqrt{T})\cdot O(\t\alpha \sigma^2 \rho^2 d k)^{k/2}$ with probability at least $1 - \delta$. The desired bound follows from the assumed bound on $T$.
\end{proof}

We now use Lemma~\ref{lem:triplesum} to verify that the ground truth satisfies Constraint~\ref{constraint2:totalnoise} with high probability.

\begin{corollary}\label{cor:totalnoise_feasible}
    For any $\delta > 0$, if $T \ge \Omega(\Max{\brc*{ \Omega(\log(2/\delta)/\eta)}}{\brc*{t^2\log^k(1/\delta)}})$, then
	\begin{equation}
		\frac{1}{T}\sum^{T/\t - 1}_{\ell = 0}\sum^{t-1}_{j = 0} a^*_{\ell t + j} \norm*{\sum^j_{i=1}BA^{j-i} w^*_{\ell\t+i}}^2 \le O\left(\eta^{1-2/k}\t\alpha\sigma^2\rho^2 m k\right)
	\end{equation} with probability at least $1 - \delta$.
\end{corollary}

\begin{proof}
    By H\"{o}lder's, we can bound the left-hand side by
    \begin{equation}
        \left(\frac{1}{T}\sum^{T/t-1}_{\ell=0}\sum^{t-1}_{j=0}a^*_{\ell t + j}\right)^{1 - 2/k}\left(\sum^{T/t-1}_{\ell = 0}\sum^{t-1}_{j = 0}\norm*{\sum^j_{i=1} BA^{j-i}w^*_{\ell t + i}}^k\right)^{2/k} \le O\left(\eta^{1-2/k}\t\alpha\sigma^2\rho^2 m k\right),
    \end{equation}
    where the last step follows by Bernstein's applied to $\brc{a_i}$ and by Lemma~\ref{lem:triplesum}.
\end{proof}

Using a proof similar to that of Lemma~\ref{lem:triplesum}, we now verify that the ground truth satisfies Constraint~\ref{constraint2:totalnoise_dumb} with high probability.

\begin{lemma}\label{lem:triplesum_dumb}
	For any $\delta > 0$, if $T \ge \log^2(1/\delta)$ then \begin{equation}
		\frac{1}{T}\sum^{T/\t - 1}_{\ell = 0} \norm*{\sum^t_{i=1}A^{t-i} w^*_{\ell\t+i}}^2 \le O\left(\sigma^2\rho^2 d\right) 
	\end{equation} with probability at least $1 - \delta$.
\end{lemma}

\begin{proof}
	 Let $M \triangleq \sum^t_{i = 1} (A^{t - i})^{\top} A^{t - i}$. We first bound the expectation \begin{equation}
		\E*{\norm*{\sum^t_{i=1} A^{t-i} w^*_{\ell\t + i}}^2} = \sigma^2 \cdot \E[g\sim\calN(0,M)]{\norm{g}^2} \le \t\sigma^2\rho^2 d \label{eq:firstmoment2},
	\end{equation} where the second step follows by the fact that $\norm{M} \le \sum^t_{i = 1}\norm{(A^{t - i})^{\top} A^{t - i}} \le t\cdot\rho^2$
	so that the norm of a sample from $\calN(0,M)$ is dominated by that of one from $\calN(0,\t\rho^2\cdot \Id_d)$.

	Next for any $\ell$ we upper bound the variance of $\norm{\sum^b_{i=1}A^{b-i}w^*_{\ell\t+i}}^2$ by its second moment:
	\begin{equation}
		\E*{\norm*{\sum^t_{i=1} A^{t-i} w^*_{\ell\t + i}}^4} = \sigma^4 \cdot \E[g\sim\calN(0,M)]{\norm{g}^4} O(\t\sigma^2\rho^2 d)^2, \label{eq:secondmoment2}
	\end{equation} where in the last step we used \eqref{eq:firstmoment2}.

	Putting \eqref{eq:firstmoment2} and \eqref{eq:secondmoment2}, we conclude that for the random variable $Z \triangleq \frac{1}{T}\sum_{\ell}\norm*{\sum^t_{i=1} BA^{t-i} w^*_{\ell\t + i}}^2$, we have \begin{equation}
		\E{Z} \le O(\sigma^2\rho^2 d) \qquad \Var{Z} \le \frac{1}{T}\cdot O(\sigma^2 \rho^2 d)^2,
	\end{equation}
	so by Lemma~\ref{lem:polyconc} we conclude that $Z \le (1 + \log(1/\delta)/\sqrt{T})\cdot O(\sigma^2 \rho^2 d)$ with probability at least $1 - \delta$. The desired bound follows from the assumed bound on $T$.
\end{proof}

\begin{lemma}\label{lem:feasible_corrupt_measurement_noise}
    For any even $k > 0$, if $T \ge \Omega(\Max{\brc*{ \Omega(\log(2/\delta)/\eta)}}{\brc*{\log^k(1/\delta)}})$, then
    \begin{equation}
        \sum^{T-1}_{i = 0} (1 - a^*_i)\norm{v^*_i}^2 \le O(mk\tau^2\cdot \eta^{1-2/k}).
    \end{equation} with probability at least $1 - \delta$.
\end{lemma}

\begin{proof}
    By H\"{o}lder's, \begin{equation}
        \frac{1}{T} \sum^{T-1}_{i=0} (1 - a^*_i) \norm{v^*_i}^2 \le \left(\frac{1}{T}\sum^{T-1}_{i=0} (1-a^*_i)\right)^{1 - 2/k}\left(\frac{1}{T}\sum^{T-1}_{i=0} \norm{v^*_i}^k\right)^{2/k}. \label{eq:holdersav}
    \end{equation}
    By Bernstein's, we have that $\avgT (1 - a^*_i) \le 1.01\eta$. To control the random variable $Z\triangleq \avgT \norm{v^*_i}^k$, note that this is a degree-$k$ polynomial in an $mT$-dimensional Gaussian with covariance $\tau^2\cdot\Id$. We have $\E{Z} \le (mk\tau^2)^{k/2}$ and $\Var{Z}\le \E{Z^2} \le \frac{1}{T}(2mk\tau^2)^k$, so by Lemma~\ref{lem:polyconc},
    \begin{equation}
        \Pr*{Z > O(mk\tau^2)^{k/2}(1 + \log^{k/2}(1/\delta)/\sqrt{T})} \le \delta.
    \end{equation}
    By the assumed bound on $T$, we conclude that $Z \le O(mk\tau^2)^{k/2}$ with probability at least $1 - \delta$. The lemma follows by \eqref{eq:holdersav}.
\end{proof}

We can finally complete the proof that Program~\ref{program:sosv2} is feasible:

\begin{proof}[Proof of Lemma~\ref{lem:feasible2}]
    The new constraints that we need to verify are Constraints~\ref{constraint2:many_b}, \ref{constraint2:funnel}, \ref{constraint2:totalnoise}, \ref{constraint2:totalnoise_dumb}, \ref{constraint2:measurements_noise}, \ref{constraint2:subsample}. Constraints~\ref{constraint2:many_b} and \ref{constraint2:subsample} follow from Lemma~\ref{lem:bstar_conc}, while Constraint~\ref{constraint2:many_ab} follows from Lemma~\ref{lem:many_ab}. Constraint~\ref{constraint2:funnel} follows from Corollary~\ref{cor:warmstart}. Constraint~\ref{constraint2:totalnoise} follows Corollary~\ref{cor:totalnoise_feasible}. Constraint~\ref{constraint2:totalnoise_dumb} follows from 
    Lemma~\ref{lem:triplesum_dumb}. Constraint~\ref{constraint2:measurements_noise} follows from Lemma~\ref{lem:normbound}. Constraint~\ref{constraint2:measurements_noise_corrupted} follows from Lemma~\ref{lem:feasible_corrupt_measurement_noise}.
\end{proof}





\subsection{Outer Argument}

In this section we show, analogously to Lemma~\ref{lem:outer} in Section~\ref{sec:logT}, that in order to compete with $\mathsf{OPT}$, it suffices for our estimate of the trajectory to be sufficiently accurate on the first iterate of every window. The difference between the following lemma and Lemma~\ref{lem:outer} is that here, we do not assume that matrix concentration holds in every window, but we do exploit the fact that the algorithm from the previous section yields a confidence band around the true trajectory.

\begin{lemma} \label{lem:outer2}
    Let $\psE{\cdot}$ be the solution to Program~\ref{program:sosv2}, assuming it is feasible. Let $\wh{x}_i \triangleq \psE{x_i}$ and $\wh{w}_i \triangleq \psE{w_i}$ for every $0 \le i < T$. Provided the event of Lemma~\ref{lem:feasible2} holds, then
    \begin{align*}
        \MoveEqLeft\frac{1}{T}\sum^{T-1}_{i=0}\left(a^*_i \norm{B\wh{x}_i - y_i}^2/\tau^2 + \norm{\wh{w}_i}^2/\sigma^2\right) + \frac{\|\wh{x}_0\|^2}{R^2 T} - \OPT \\
        &\lesssim \tau^{-2}\left\{\eta \rho^2 \left(\alpha +\norm{B}^2\sqrt{\log(t/\delta_1)/t}\right)\cdot \frac{1}{T/t}\sum^{T/t - 1}_{\ell=0} \psE*{\norm{x_{\ell t} - x^*_{\ell t}}^2} \right.\\
        &\quad \left. +  k\eta^{1-2/k}\cdot \left(\t\alpha\sigma^2\rho^2 m  + m\tau^2\right) + \norm{B}^2\left(\epsgeo\eta\delta_1 + \rho^2R^2(d+\log(1/\delta))\cdot \left(\frac{\eta\delta_1}{T/t}\right)^{1/2}\right)\right\}
    \end{align*}
\end{lemma}

If we ignore the $\frac{\eta\delta_1}{T/t}$ term as it is vanishing in $T$, we see that the excess risk can be upper bounded in terms of the error on the first iterates of every window, some term that vanishes in $\eta$, and the quantity $\norm{B}^2\epsgeo\eta\delta_1$. Recall that $\epsgeo$ scales polylogarithmically in $T$, so because we will ultimately take $1/\delta_1$ polylogarithmic in $T$, the contribution of this term will be negligible.

To prove Lemma~\ref{lem:outer2}, we will use the following analogue of Lemma~\ref{lem:kkm}.

\begin{lemma}\label{lem:kkm2}
    Let $\psE{\cdot}$ be the solution to Program~\ref{program:sosv2}, assuming it is feasible. Let $\wh{x}_i \triangleq \psE{x_i}$ and $\wh{w}_i \triangleq \psE{w_i}$ for every $0\le i < T$. Provided the event of Lemma~\ref{lem:feasible2} holds, then
    \begin{multline}
        \frac{1}{T}\left(\sum_{i=0}^{T-1}\left(a^*_i \| B\wh{x}_i - y_i\|^2/\tau^2 +  \norm{\wh{w}_i}^2/\sigma^2\right) + \|\hat{x}_0\|^2/R^2\right) - \OPT \le \\
        \psE*{\frac{1}{T}\sum_{i=0}^{T-1} (1 - a_i)\| B(x_i - x^*_i)\|^2/\tau^2} + O(mk\cdot \eta^{1-2/k}). \label{eq:firstopt2}
    \end{multline}
\end{lemma}

\begin{proof}
    The proof is essentially identical to that of Lemma~\ref{lem:kkm}. The only difference is in the step where we carry out the bound in \eqref{eq:av}. Previously in the proof of Lemma~\ref{lem:kkm}, we used Constraints~\ref{item:many_inliers} and \ref{item:measurements_noise} from Program~\ref{program:sos}. Instead, in Program~\ref{program:sosv2} we have directly included the desired bound on $\sum^{T-1}_{i=0} (1-a_i)\norm{v_i}^2$ via Constraint~\ref{constraint2:measurements_noise_corrupted}.
\end{proof}

We now proceed with the proof of Lemma~\ref{lem:outer2}.

\begin{proof}[Proof of Lemma~\ref{lem:outer2}]
    Lemma~\ref{lem:kkm2} reduces upper bounding the excess risk achieved by $\brc{\wh{x}_i}$, $\brc{\wh{w}_i}$ to bounding the main term $\psE{\frac{1}{T}\sum^{T-1}_{i=0}(1 - a_i)\norm{B(x_i - x^*_i)}^2/\tau^2}$ in \eqref{eq:firstopt2}, which we do now. Writing $1 - a_i = b_{\ell(i)}\cdot (1 - a_i) + (1 - b_{\ell(i)})\cdot (1 - a_i)$, we have \begin{equation}
        \sum^{T-1}_{i=0} (1 - a_i)\norm{B(x_i - x^*_i)}^2 = \sum^{T-1}_{i=0} \brk*{b_{\ell(i)}(1 - a_i)\norm{B(x_i - x^*_i)}^2 + (1 - b_{\ell(i)})(1-a_i)\norm{B(x_i - x^*_i)}^2}. \label{eq:breakup_bs}
    \end{equation}
    We first control the latter sum on the right-hand side of \eqref{eq:breakup_bs} using Constraint~\ref{constraint2:funnel} and the confidence band guarantee of Corollary~\ref{cor:warmstart}.
    For every $0\le i < T$, we have
    \begin{equation}
        \norm{x_i - x^*_i}^2 \le 2\norm{x_i - x'_i}^2 + 2\norm{x'_i - x^*_i}^2 \lesssim \epsgeo + \rho^2 R^2(d + \log(1/\delta))/2^{\ell(i)},
    \end{equation}
    so
    \begin{multline}
        \sum^{T-1}_{i=0} (1 - b_{\ell(i)})(1-a_i)\norm{B(x_i - x^*_i)}^2 \le \\ \norm{B}^2\epsgeo\sum^{T-1}_{i=0}(1 - b_{\ell(i)})(1 - a_i) + \norm{B}^2\rho^2 R^2(d +\log(1/\delta))\sum^{T-1}_{i=0} (1 - b_{\ell(i)})(1 - a_i)/2^{\ell(i)/2}. \label{eq:abB}
    \end{multline}
    By Constraint~\ref{constraint2:many_ab}, we can bound the first term on the right-hand side by $\norm{B}^2\epsgeo\eta\delta_1 T$. For the second term, note that
    \begin{equation}
        \avgT (1 - b_{\ell(i)})(1 - a_i)/2^{\ell(i)/2} = (\eta\delta_1)^{1/2}\cdot \left(\avgT 2^{-\ell(i)}\right)^{1/2} \le \left(\frac{\eta\delta_1}{T/t}\right)^{1/2}.
    \end{equation}
    We may thus upper bound \eqref{eq:abB} by
    \begin{equation}
        \norm{B}^2\epsgeo \eta\delta_1 + \norm{B}^2\rho^2 R^2(d + \log(1/\delta))\cdot \left(\frac{\eta\delta_1}{T/t}\right)^{1/2}.
    \end{equation}
    
    We now turn to the first sum on the right-hand side of in \eqref{eq:breakup_bs}. For these terms, where $b_{\ell(i)} = 1$, we will eschew the guarantee of Corollary~\ref{cor:warmstart} in favor of a finer estimate.
    Using Fact~\ref{fact:unroll}, for any $i = \ell t + j$ we can write $B(x_i - x^*_i) = BA^j(x_{\ell t} - x^*_{\ell t}) + \sum^j_{s=1} BA^{j - s}(w_{\ell t + s} - w^*_{\ell t + s})$, so by triangle inequality we can upper bound the first sum on the right-hand side of \eqref{eq:breakup_bs} by \begin{align}
        3\sum_{\ell=0}^{T/t-1}b_{\ell} \sum_{j=0}^{t-1} (1 - a_{\ell t + j}) \left(\norm*{BA^j(x_{\ell t} - x^*_{\ell t})}^2 + \norm*{\sum^j_{s=0}BA^{j-s}w_{\ell t + s}}^2 + \norm*{\sum^j_{s=0}BA^{j-s}w^*_{\ell t +s}}^2\right) \label{eq:subsample_and_noise2}
    \end{align}
    
    We can control the two noise terms on the right by using $b_{\ell} \le 1$, Constraint~\ref{constraint2:totalnoise}, and Corollary~\ref{cor:totalnoise_feasible} to get 
    \begin{multline}
        \frac{1}{T}\sum^{T/t - 1}_{\ell = 0}b_{\ell} \sum^{t-1}_{j=0}(1 - a_{\ell t + j})\left(\norm*{\sum^j_{s=0}BA^{j-s}w_{\ell t + s}}^2 + \norm*{\sum^j_{s=0}BA^{j-s}w^*_{\ell t +s}}^2\right) \\
        \le O\left(\eta^{1-2/k}\cdot \t\alpha\sigma^2\rho^2 m k \right). \label{eq:noise_term}
    \end{multline}
    For the remaining sum in \eqref{eq:subsample_and_noise2}, we invoke Constraint~\ref{constraint2:subsample} and the bound on $\norm{\calO_t}$ in \eqref{eq:opnormOt} to get
    \begin{equation}
        \frac{1}{T}\sum^{T/t-1}_{\ell=0}b_{\ell}\sum^{t-1}_{j=0}(1 - a_{\ell t + j})\norm{BA^j(x_{\ell t} - x^*_{\ell t})}^2 \le \eta\left(\alpha\rho^2 + O\left( \rho^2\norm{B}^2\sqrt{\log(t/\delta_1)/t}\right)\right)\cdot \frac{1}{T/t}\sum^{T/t-1}_{\ell=0} b_{\ell} \norm{x_{\ell t} - x^*_{\ell t}}^2,
    \end{equation}
    from which the lemma follows by substituting this and \eqref{eq:noise_term} into \eqref{eq:subsample_and_noise2}.
\end{proof}

\subsection{Decay of Unobservable Subspace}
\label{sec:decay}

In this section we revisit the analysis from Section~\ref{sec:decay}. It turns out that essentially all of the steps from that section carry over to our setting once we incorporate the $b_{\ell}$ variables. As the proofs here are otherwise essentially identical to their counterparts in Section~\ref{sec:decay}, we will only state the relevant lemmas and defer the proofs to Appendix~\ref{app:defer}.

As in Section~\ref{sec:old_decay}, let $\Pi$ denote the projection to the \emph{observable subspace}, that is, to the subspace of $v\in\R^d$ for which $v^{\top} \calO_t v \ge \thres$ for $\thres \triangleq \frac{\kappa \t}{40000\rho^4}$.

We will use the following matrix concentration bound which is the analogue of Lemma~\ref{lem:anticonc}:

\begin{lemma}\label{lem:anticonc_2}
    Suppose $\eta < 0.49$ and $t\ge \wt{\Omega}\left(\kappa^{-2} \rho^{12} \norm{B}^4\log(d/\delta_1)\right)$. Then with probability at least $1 - \delta$ over the randomness of $\brc{a^*_i}$, we have that for all windows $0 \le \ell < T/t$, there is a degree-4 SoS proof of the psd inequality
    \begin{equation}
        b_\ell b^*_{\ell} \sum^{t-1}_{i = 0} a^*_{\ell t + i} a_{\ell t + i} \Pi (A^i)^{\top} B^{\top} B A^i \Pi \succeq b_\ell b^*_\ell \frac{1}{100}\Pi\calO_t \Pi. \label{eq:psdlb2}
    \end{equation}
    using the constraints of Program~\ref{program:sosv2}.
\end{lemma}

\noindent The main result of this section, analogous to Lemma~\ref{lem:error_decay}, says that provided the event of Lemma~\ref{lem:anticonc} holds, our prediction error on the first iterate of every window decays exponentially over time:

\begin{lemma}\label{lem:error_decay2}
    Let pseudoexpectation $\psE{\cdot}$ be the solution to Program~\ref{program:sosv2}, assuming it is feasible. Provided the event of Lemma~\ref{lem:anticonc_2} holds, we have
    \[ \psE*{b_{\ell - 1} b^*_{\ell - 1} \|x_{\ell t} - x^*_{\ell t}\|^2} \le \psE*{b_{\ell - 1} b^*_{\ell - 1} \cdot \frac{1}{2}\norm{\Pi^{\perp} q}^2 + 4\gamma_{\ell} + 4\gamma^*_{\ell} + \frac{2}{15\rho^2\zeta}\sum^t_{i=1}(\epsilon_{\ell,i} + \epsilon^*_{\ell,i})}, \]
    where 
    \begin{equation}
        \gamma_{\ell} \triangleq \norm*{\sum^t_{i=1} A^{t-i} w_{\ell t + i}}^2, \qquad \gamma^*_{\ell} \triangleq \norm*{\sum^t_{i=1} A^{t-i} w^*_{\ell t + i}}^2,
    \end{equation}
    \begin{equation}
        \epsilon_{\ell,i} \triangleq \left(3\norm{v_{\ell t +i}}^2 + 6\norm{\sum^i_{s=1}BA^{i-s}w_{\ell t + s}}^2\right), \qquad \epsilon^*_{\ell,i} \triangleq \left(3\norm{v^*_{\ell t +i}}^2 + 6\norm{\sum^i_{s=1}BA^{i-s}w^*_{\ell t + s}}^2\right).
    \end{equation}
\end{lemma}

To prove this, we will use Lemma~\ref{lem:unobservable-decay} from Section~\ref{sec:logT} which is just a statement about vectors and the observable subspace. The other key ingredient, analogous to Lemma~\ref{lem:observable-case-analysis}, is the following:

\begin{lemma}\label{lem:observable-case-analysis2}
    Let pseudoexpectation $\psE{\cdot}$ be the solution to Program~\ref{program:sosv2}, assuming it is feasible. Provided the event of Lemma~\ref{lem:anticonc_2} holds, then at least one of the following holds for every window $0 \le \ell < T/t$ for $q \triangleq x_{\ell t} - x^*_{\ell t}$:
    \begin{enumerate}
        \item (Observable component decays) 
        \begin{equation}
            \psE*{b_{\ell} b^*_{\ell}\norm{A^{t}\Pi q}^2} 
            \le \frac{1}{10}\psE*{b_{\ell} b^*_{\ell} \| \Pi q) \|^2}.
        \end{equation}
        \item (Observable error bounded by unobservable error)
        \begin{equation}
            \psE*{b_{\ell} b^*_{\ell} \norm{\Pi q}^2} \le \psE*{\frac{1}{10\rho^2}b_{\ell} b^*_{\ell} \norm{\Pi^{\perp} q}^2 + \frac{1}{30\rho^2\zeta}\sum^t_{i=1}(\epsilon_{\ell,i} + \epsilon^*_{\ell,i})}.
        \end{equation}
    \end{enumerate}
\end{lemma}



\subsection{Error on First Iterates in Each Window}
\label{sec:first_iters_loglogT}

Recall from Lemma~\ref{lem:outer2} that we need to upper bound the pseudoexpectation of the average error on the first iterate in each window. In this section we prove the following bound:

\begin{lemma}\label{lem:error-first-iterate-sos2}
Provided $T \ge \Omega(\Max{\brc*{(t/\eta\delta_1)\log(1/\delta)}}{\brc*{t^2\log^k(1/\delta)}})$, with probability at least $1 - \delta$ we have
    \begin{align*}
        \frac{1}{T/t}\sum^{T/t - 1}_{\ell = 0} \psE*{\norm{x_{\ell t} - x^*_{\ell t}}^2} \lesssim 
        \kappa^{-1} \t \alpha \sigma^2 \rho^4m + \kappa^{-1}\tau^2\rho^2 m + \t\sigma^2\rho^2 d + 
        \concprob\epsgeo +\frac{\rho^2 R^2(d + \log(1/\delta))}{T/t}.\label{eq:firstiters_2}
    \end{align*}
\end{lemma}

\begin{proof}

We will use a naive bound on the first summand $\frac{1}{T/t}\psE{\norm{x_0 - x^*_0}^2}$ by using the fact that $x^*_0 \sim\calN(0,R^2\cdot\Id_d)$ and proceed to upper bound the contribution of windows $1$ to $T/t-1$.

For these terms, we will use the elementary inequality
\begin{equation}
    1 \le b_{\ell - 1}b^*_{\ell-1} + \frac{3}{2}(1 - b_{\ell - 1}) + \frac{3}{2}(1 - b^*_{\ell - 1})
\end{equation} and upper bound the left hand side of our desired inequality by three sums.

For the first sum, we recall from Lemma~\ref{lem:error_decay2}
that for any $\ell > 1$,
\[ \psE*{b_{\ell - 1} b^*_{\ell - 1} \|x_{\ell t} - x^*_{\ell t}\|^2} \le \psE*{b_{\ell - 1} b^*_{\ell - 1} \cdot \frac{1}{2}\norm{\Pi^{\perp} q}^2 + 4\gamma_{\ell} + 4\gamma^*_{\ell} + \frac{2}{15\rho^2\zeta}\sum^t_{i=1}(\epsilon_{\ell,i} + \epsilon^*_{\ell,i})}, \]
where $q = x_{(\ell - 1)t} - x^*_{(\ell - 1)t}$, to get
\begin{align}
    \MoveEqLeft \psE*{\frac{1}{T/t}\sum^{T/t-1}_{\ell =1} b_{\ell - 1} b^*_{\ell - 1} \norm{x_{\ell t} - x^*_{\ell t}}^2} \\
    &\le \psE*{\frac{1}{T/t}\sum^{T/t-1}_{\ell =1} b_{\ell - 1} b^*_{\ell - 1} \cdot \frac{1}{2}\norm{\Pi^{\perp} q}^2 + \frac{1}{T/t}\sum^{T/t-1}_{\ell =1} \left(4\gamma_{\ell} + 4\gamma^*_{\ell} + \frac{2}{15\rho^2\zeta}\sum^t_{i=1}(\epsilon_{\ell,i} + \epsilon^*_{\ell,i})\right)} \\
    &\le  \psE*{\frac{1}{T/t}\sum^{T/t-1}_{\ell =2} \frac{1}{2}\norm{q}^2 + \frac{1}{T/t}\sum^{T/t-1}_{\ell =1} \left(4\gamma_{\ell} + 4\gamma^*_{\ell} + \frac{2}{15\rho^2\zeta}\sum^t_{i=1}(\epsilon_{\ell,i} + \epsilon^*_{\ell,i})\right)}.\label{eq:bb}
\end{align}
For the second sum, we have
\begin{align}
    \MoveEqLeft\psE*{\frac{1}{T/t}\sum^{T/t-1}_{\ell =1} (1 - b_{\ell - 1}) \norm{x_{\ell t} - x^*_{\ell t}}^2} \\
    &\le \psE*{\frac{2}{T/t}\sum^{T/t-1}_{\ell =1}(1 - b_{\ell - 1})\left(\norm{x_{\ell t} - x'_{\ell t}}^2 + \norm{x'_{\ell t} - x^*_{\ell t}}^2\right)} \\
    &\le \psE*{\frac{2}{T/t}\sum^{T/t-1}_{\ell =1}(1 - b_{\ell - 1})\cdot 2(\epsgeo + O(\rho^2R^2(d+\log(1/\delta))/2^{\ell(i)}))} \\
    &\lesssim \concprob \epsgeo + \frac{\rho^2 R^2(d + \log(1/\delta))}{T/t},\label{eq:oneminusb}
\end{align} where in the last step we used Constraint~\ref{constraint2:funnel} and Corollary~\ref{cor:warmstart}. We can bound the third sum in an identical fashion, as $\frac{1}{T/t}\sum^{T/t-1}_{\ell =1}(1 - b^*_{\ell - 1}) \le \concprob$ by Bernstein's. Specifically, we get 
\begin{equation}
    \psE*{\frac{1}{T/t}\sum^{T/t-1}_{\ell =1} (1 - b^*_{\ell - 1}) \norm{x_{\ell t} - x^*_{\ell t}}^2} \lesssim \concprob \epsgeo + \frac{\rho^2 R^2(d + \log(1/\delta))}{T/t} \label{eq:oneminusbstar}
\end{equation}

Putting \eqref{eq:bb}, \eqref{eq:oneminusb}, and \eqref{eq:oneminusbstar} together, we conclude that 
\begin{multline}
    \frac{1}{T/t}\sum^{T/t-1}_{\ell =1} \psE*{\norm{x_{\ell t} - x^*_{\ell t}}^2} \le O\left(\concprob \epsgeo + \frac{\rho^2 R^2(d + \log(1/\delta))}{T/t}\right) + \psE*{\frac{1}{T/t}\sum^{T/t-1}_{\ell =1} \frac{1}{2} \norm{x_{(\ell - 1)t} - x^*_{(\ell - 1)t}}^2} \\
        + \psE*{ \frac{1}{T/t}\sum^{T/t-1}_{\ell =0} \left(4\gamma_{\ell} + 4\gamma^*_{\ell} + \frac{2}{15\rho^2\zeta}\sum^t_{i=1}(\epsilon_{\ell,i} + \epsilon^*_{\ell,i})\right)},
\end{multline}
so rearranging we get
\begin{multline}
    \frac{1}{T/t}\sum^{T/t-1}_{\ell =1}\frac{1}{2}\norm{x_{\ell t} - x^*_{\ell t}}^2 \le \frac{1}{2}\cdot\frac{1}{T/t}\norm{x_0 - x^*_0}^2 +  O\left(\concprob \epsgeo + \frac{\rho^2 R^2(d + \log(1/\delta))}{T/t}\right) \\
    + \psE*{ \frac{1}{T/t}\sum^{T/t-1}_{\ell =0} \left(4\gamma_{\ell} + 4\gamma^*_{\ell} + \frac{2}{15\rho^2\zeta}\sum^t_{i=1}(\epsilon_{\ell,i} + \epsilon^*_{\ell,i})\right)} \label{eq:firstiters_intermed}
\end{multline}

It remains to control the $\epsilon$'s and $\gamma$'s. Taking $k = 2$ in Lemma~\ref{lem:triplesum}, we get
\[		\frac{1}{T}\sum^{T/\t - 1}_{\ell = 0}\sum^{t-1}_{b = 0} \norm*{\sum^b_{i=1}BA^{b-i} w^*_{\ell\t+i}}^2 \le O\left(\t\alpha\sigma^2\rho^2 m\right), \]
so we have from this, Constraint~\ref{constraint2:totalnoise}, Constraint~\ref{constraint2:measurements_noise}, and Lemma~\ref{lem:normbound} that
\begin{align}
    \MoveEqLeft\frac{1}{T} \sum_{\ell = 0}^{T/t - 1} \sum^t_{i = 1} (\epsilon_{\ell,i} + \epsilon^*_{\ell,i}) \\
    &=  \frac{1}{T} \sum_{\ell = 0}^{T/t - 1} \sum^t_{i = 1} \left(3\norm{v_{\ell t +i}}^2 + 6\norm*{\sum^i_{s=1}BA^{i-s}w_{\ell t + s}}^2 + 3\norm*{v^*_{\ell t +i}}^2 + 6\norm{\sum^i_{s=1}BA^{i-s}w^*_{\ell t + s}}^2\right) \\
    &\le O\left(\t\alpha\sigma^2\rho^2 m + \tau^2(m + \log(2/\delta)/T)\right)\label{eq:epsbound}
\end{align}
To bound the $\gamma$'s, we invoke Lemma~\ref{lem:triplesum_dumb} 
to get that
\begin{equation}
	\frac{1}{T}\sum^{T/\t - 1}_{\ell = 0} \gamma^*_{\ell} \le O\left(\sigma^2\rho^2 d \right), \label{eq:upsbound}
\end{equation}
and by Constraint~\ref{constraint2:totalnoise_dumb}, we get the same bound for $\frac{1}{T}\sum^{T/\t - 1}_{\ell = 0} \gamma_{\ell}$. 

Putting together \eqref{eq:epsbound} and \eqref{eq:upsbound} we get
\begin{multline}
\psE*{ \frac{1}{T/t}\sum^{T/t-1}_{\ell =0} \left(4\gamma_{\ell} + 4\gamma^*_{\ell} + \frac{2}{15\rho^2\zeta}\sum^t_{i=1}(\epsilon_{\ell,i} + \epsilon^*_{\ell,i})\right)} \\
\lesssim \frac{t}{\rho^2\zeta}\cdot(\t\alpha \sigma^2 \rho^2m + \tau^2(m + \log(2/\delta)/T)) + \t\sigma^2\rho^2 d \lesssim \kappa^{-1} \t \alpha \sigma^2 \rho^4m + \kappa^{-1}\tau^2\rho^2 m + \t\sigma^2\rho^2 d
\end{multline}
where we recalled the definition of $\thres \triangleq \frac{\kappa \t}{40000\rho^4}$ and used the assumed lower bound on $T$.

Combining with \eqref{eq:firstiters_intermed},  we obtain the claimed bound.
\end{proof}

\subsection{Putting Everything Together}

We can now complete the proof of Theorem~\ref{thm:main_TloglogT} by applying the bounds in Section~\ref{sec:first_iters_loglogT} to Lemma~\ref{lem:outer2}.
\begin{proof}[Proof of Theorem~\ref{thm:main_TloglogT}]
Let \begin{equation}
    t = \Max{s}{\wt{\Theta}\left(\kappa^{-2} \rho^{12} \norm{B}^4\log(d/\delta_1)\right)} \label{eq:loglogTtdef}
\end{equation} as in Lemma~\ref{lem:anticonc_2}.

First, recall from Lemma~\ref{lem:outer2} that
    \begin{align*}
        \MoveEqLeft\frac{1}{T}\sum^{T-1}_{i=0}\left(a^*_i \norm{B\wh{x}_i - y_i}^2/\tau^2 + \norm{\wh{w}_i}^2/\sigma^2\right) + \frac{\|\wh{x}_0\|^2}{R^2 T} - \OPT \\
        &\lesssim \tau^{-2}\left\{\eta \rho^2 \left(\alpha +\norm{B}^2\sqrt{\log(t/\delta_1)/t}\right)\cdot \frac{1}{T/t}\sum^{T/t - 1}_{\ell=0} \psE*{\norm{x_{\ell t} - x^*_{\ell t}}^2} \right.\\
        &\quad \left. +  k\eta^{1-2/k}\cdot \left(\t\alpha\sigma^2\rho^2 m  + m\tau^2\right) + \norm{B}^2\left(\epsgeo\eta\delta_1 + \rho^2R^2(d+\log(1/\delta))\cdot \left(\frac{\eta\delta_1}{T/t}\right)^{1/2}\right)\right\}
    \end{align*}
and recall from Lemma~\ref{lem:error-first-iterate-sos2} that
    \begin{equation*}
        \frac{1}{T/t}\sum^{T/t-1}_{\ell =0} \psE*{\norm{x_{\ell t} - x^*_{\ell t}}^2} \lesssim  \kappa^{-1} \t \alpha \sigma^2 \rho^4m + \kappa^{-1}\tau^2\rho^2 m + \t\sigma^2\rho^2 d + \concprob\epsgeo +\frac{\rho^2 R^2(d + \log(1/\delta))}{T/t}
    \end{equation*}
so combining the two bounds gives 
\begin{multline}
        \avgT\left(a^*_i \norm{B\wh{x}_i - y_i}^2/\tau^2 + \norm{\wh{w}_i}^2/\sigma^2\right) + \frac{\|\wh{x}_0\|^2}{R^2 T} - \OPT \\
        \lesssim \tau^{-2}\Bigg\{\eta \rho^2 \left(\alpha +\norm{B}^2\sqrt{\log(t/\delta_1)/t}\right)\cdot \Bigg[ \kappa^{-1} \t \alpha \sigma^2 \rho^4m + \kappa^{-1}\tau^2\rho^2 m + \t\sigma^2\rho^2 d \\
        +\concprob\epsgeo +\frac{\rho^2 R^2(d + \log(1/\delta))}{T/t}\Bigg] +  k\eta^{1-2/k}\cdot \left(\t\alpha\sigma^2\rho^2 m  + m\tau^2\right)  \\
        + \norm{B}^2\left(\epsgeo\eta\delta_1 + \rho^2R^2(d+\log(1/\delta))\cdot \left(\frac{\eta\delta_1}{T/t}\right)^{1/2}\right)\Bigg\} \label{eq:ugly}
\end{multline}

Note that for the value of $\t$ we chose in \eqref{eq:loglogTtdef}, we have
\begin{align} 
\log(t/\delta_1)/t 
&= \kappa^2 \rho^{-12} \norm{B}^{-4} \frac{\log(\kappa^{-2} \rho^{12} \norm{B}^4) + \log \log(d/\delta_1) + \log(1/\delta_1)}{\log(d/\delta_1)} \\
&\le \kappa^2 \rho^{-12} \norm{B}^{-4}(2 + \log(\kappa^{-2} \rho^{12} \norm{B}^4)/\log(d/\delta_1)) \le O(\kappa^2\rho^{-12}\norm{B}^{-4}).\label{eq:logtovertbound}
\end{align}

The requirement for Lemmas~\ref{lem:bstar_conc} and \ref{lem:many_ab} is that $\concprob =\Omega((t/\eta T) \log(1/\delta))$, which by our choice of $t$ translates to $\delta_1 = \Omega(\kappa^{-2} \rho^{12} \|B\|^4 \log(d/\delta_1)\log(1/\delta)/\eta T)$, or equivalently
\begin{equation}
    \delta_1/\log(d/\delta_1) = \Omega(\kappa^{-2} \rho^{12} \|B\|^4 \log(1/\delta)/\eta T) \label{eq:delta1def}
\end{equation} 
and then taking \begin{equation}
    \delta_1 = \Theta(\log(1/\delta)/\log^3 T)
\end{equation} satisfies this condition provided 
$T = \wt{\Omega}(\eta^{-1}\kappa^{-2}\rho^{12}\norm{B}^4\log d)$. Note that this choice of $\delta_1$ yields \begin{equation}
    \t = \Max{s}{\wt{\Theta}(\kappa^{-2} \rho^{12} \|B\|^4 \log(d\log(T)/\log(1/\delta)))}. \label{eq:finaltbound}
\end{equation} 
In particular, $\t$ is \emph{doubly logarithmic} in $T$.

Plugging in \eqref{eq:logtovertbound} and \eqref{eq:delta1def} into \eqref{eq:ugly} gives
\begin{align*}
        &\avgT\left(a^*_i \norm{B\wh{x}_i - y_i}^2/\tau^2 + \norm{\wh{w}_i}^2/\sigma^2\right) + \frac{\|\wh{x}_0\|^2}{R^2 T} - \OPT \\
        &\lesssim \tau^{-2}\left\{\eta \rho^2 \left(\alpha + \kappa \rho^{-6} \right)\cdot \left( \kappa^{-1} \t \alpha \sigma^2 \rho^4m + \kappa^{-1}\tau^2\rho^2 m + \t\sigma^2\rho^2 d + \frac{\log(1/\delta)}{\log^3 T} \epsgeo +\frac{\rho^2 R^2(d + \log(1/\delta))}{T/t}\right)\right.\\
        &\quad \left. +  k\eta^{1-2/k}\cdot \left(\t\alpha\sigma^2\rho^2 m  + m\tau^2\right) + \norm{B}^2\left(\epsgeo\eta\frac{\log(1/\delta)}{\log^3 T} + \rho^2R^2(d+\log(1/\delta))\cdot \left(\frac{\eta\delta_1}{T/t}\right)^{1/2}\right)\right\}.
\end{align*}
Recall from  \eqref{eq:eps_geo2} that
\begin{equation}
    \epsgeo = O\left(\frac{\rho^8 t_{\mathsf{pre}}}{\kappa}\cdot \brk*{\tau^2\left(m + \log(T/\delta)\right) + \sigma^2 \left(d +\log(T/\delta)\right)\cdot \rho^2 t_{\mathsf{pre}} \norm{B}^2}\right)
\end{equation}
and $t_{\mathsf{pre}} = \Max{s}{\wt{\Theta}(\kappa^{-2}\rho^{12}\norm{B}^4\log(dT/\delta))}$. Crucially, we see that the factor of $\log^3 T$ in $\epsgeo$ is canceled out by the $1/\log^3 T$ in our bound coming from the choice of $\delta_1$. We see that
\begin{multline}
    \avgT\left(a^*_i \norm{B\wh{x}_i - y_i}^2/\tau^2 + \norm{\wh{w}_i}^2/\sigma^2\right) + \frac{\|\wh{x}_0\|^2}{R^2 T} - \OPT \\
    \lesssim \poly(s,\alpha,1/\kappa,\rho,\|B\|,\log d,\log(1/\delta)) \brk*{ k\eta^{1-2/k}\left(m + d(\sigma^2/\tau^2)\log\log T\right) + \eta^{1/2}(R^2 d/\tau^2)T^{-0.49}}
\end{multline}
By taking $k = \floor{2\log(1/\eta)}$ so that $k\eta^{1-2/k} = 2e\eta\log(1/\eta)$, we obtain the claimed bound.
\end{proof}

%% file: twostage.tex
\section{Online Filtering}
\label{sec:online}

The main technical results proved earlier in this paper are for the problem of robust \emph{offline}/noncausal filtering, where we are allowed to look at future observations when estimating the state at time $t$. Of course, in many settings we would also like the ability to make online predictions. 
In this section, we describe a two-stage filter which combines a first-stage filter based on offline filtering described above with a second stage of Kalman filtering. The main result of this section is that combining these two approaches gives good results for online robust filtering. The strategy is to use the first (offline) filter as a way to get a rough estimate of the state, which allows us to correct for large outliers in the sequence of observations. Given these corrected values, the success of the Kalman filter follows from classical stability results, which show that small deviations from the true response model do not cause the filter to explode. We note that the idea of using the stability of the Kalman filter as a tool in the analysis of robust filtering appeared in the work \cite{schick1994robust} referenced in the Introduction.

\subsection{Background: Stability of the Kalman Filter}
We start by stating the Kalman filter recursion in our notation. Recall from \eqref{eq:dynamics} and \eqref{eq:measure} that the uncorrupted version of our model is
\begin{align}
	x^*_t &= Ax^*_{t-1} + w^*_t  \\
	y^*_t &= Bx^*_t + v^*_t
\end{align}
where $w^*_t\sim\calN(0,\sigma^2 \cdot \Id_d), v^*_t\sim\calN(0,\tau^2 \cdot \Id_m)$. Given this notation, the Kalman filter equations are
\begin{align*}
    \hat{x}_{t | t } &= \hat{x}_{t | t - 1} + K_{t} \tilde y_{t} \\
    \hat{x}_{t | t - 1} &= A \hat x_{t - 1 | t - 1} \\
    \tilde y_{t} &= y^*_t - B \hat{x}_{t | t - 1} \\
    S_t &= B P_{t | t - 1}B^T + \tau^2 \Id_m \\    
    K_{t} &= P_{t | t - 1} B^T S_t^{-1} \\
    P_{t | t} &= (I - K_t B) P_{t | t - 1} \\
    P_{t | t - 1} &= A P_{t - 1 | t - 1} A^T + \sigma^2 \Id_d
\end{align*}
where $\hat{x}_{t | t}$ is the \emph{a posteriori} state estimate, $\hat{x}_{t | t - 1}$ is the \emph{a priori} state estimate, $\tilde y_t$ is the \emph{innovation}, $S_t$ the covariance of the innovation, $K_t$ is the optimal Kalman gain, $P_{t | t}$ is the \emph{a posteriori} state estimate covariance, and $P_{t | t - 1}$ is the \emph{a priori} state estimate covariance. 
\begin{theorem}[Exponential Asymptotic Stability of the Kalman Filter, \cite{moore1980coping}, Theorem 4.1 of \cite{schick1994robust} and \cite{schick1989robust}]\label{thm:kalman-stability}
Suppose that the matrices $A,B$ satisfy observability and $\sigma > 0$.
Then if
\[ \tilde x_{t | t} = (I - K_t B)A\tilde \tilde x_{t - 1 | t - 1}, \]
there exists $\lambda > 0$ and $\delta \in (0,1)$ depending on the system parameters such that
\[ \|\tilde x_{t | t}\| < \lambda \delta^t \|\tilde x_{0 | 0}\| \]
for any $t \ge 1$.
\end{theorem}
This result is shown using Lyapunov's method and more precise dependence of $\lambda,\delta$ on system parameters is given in \cite{schick1989robust}, 
based on the argument in \cite{moore1980coping}. We note that in the more general case where the system noise covariance is rank-degenerate, the result needs an additional (standard) assumption called \emph{controllability}. 
The significance of $\tilde x_t$ can be seen from the equation
\[ \hat{x}_{t | t} = A \hat{x}_{t - 1 | t - 1} + K_t (y^*_t - B A\hat{x}_{t - 1 | t - 1}) = (I - K_t B) A\hat{x}_{t - 1 | t - 1} + K_t y^*_t = \tilde{x}_{t | t} + K_t y^*_t \]
which follows by expanding the Kalman filter equations; the Theorem shows that if we add a perturbation $\tilde x_{0 | 0}$ to the true value of $\hat{x}_{0 | 0}$, the effect of this perturbation disappears at an exponential rate. As a consequence we obtain the following result.
\begin{corollary}[\cite{schick1989robust,schick1994robust,moore1980coping}]\label{corr:stability}
Suppose that the assumptions of Theorem~\ref{thm:kalman-stability} are satisfied, and let $(y_t)_{t = 1}^T$ be an arbitrary sequence of random variables. Let $\hat{x}_{t | t}$ be defined by the Kalman filter equations above with arbitrary initialization, and let $\hat{ \hat{x}}_{t | t}$ 
be defined by the Kalman filter equations with input $y^*_t$ replaced by $y_t$ and the same initialization. Then for any $t$,
\[ \|\hat{\hat{x}}_{t | t} - \hat{x}_{t | t}\| \le \lambda \sum_{s = 1}^t \delta^{t - s} \|K_s(y_s - y^*_s)\|. \]
\end{corollary}
\subsection{Two-Stage Filter}

The following Lemma describes the two-stage filter we propose. Given access to a sequence of $y'_1,\ldots,y'_T$ which is close to the true (noisy) observations $y^*_1,\ldots,y^*_T$, the Kalman filter is applied to the sequence $y''_1,\ldots,y''_T$ which corresponds to $\tilde{y}_1,\ldots,\tilde{y}_T$ with all gross outliers removed. Note that the sequence $y'_1,\ldots,y'_T$ is allowed to dependence arbitrarily on $y_1,\ldots,y_T$, and so can be produced by an offline filter, while the final prediction for time $T + 1$ is produced online. The result has accuracy close to that of the Kalman filter on the clean data for predicting (in an online/causal fashion) the next state $x^*_{T + 1}$. The first Lemma below states this result for a single timestep, and the Theorem gives a version averaged over time steps.
\begin{lemma}\label{lem:filter1}
Let $\mathcal{F}_T$ is the filtration generated by all random variables from the generative model up to time $T$ (i.e. $a^*_t,y^*_t,x^*_t$).
Suppose that the collection of random variables $y'_1,\ldots,y'_T$ are measurable with respect to $\mathcal{F}_T$ 
Furthermore, suppose that there exists $r \ge 0$ such that
\begin{equation}\label{eqn:guaranteed-corruption}
 \max_{t = 1}^T |y^*_t - y'_t| \le r.    
\end{equation}
For $1 \le t \le T$, define $y''_t = \tilde{y}_t$ if $|\tilde{y}_t - y'_t| \le r$ and $y''_t = y'_t$ otherwise. Let $\hat{\hat{x}}_{T + 1 | T}$ denote the output of the Kalman filter applied to input $y''_t$ and $\hat{x}_{T + 1 | T}$ denote the output of the Kalman filter applied to the (unknown) ground truth $y^*_1,\ldots,y^*_{T}$, then
\[ \E{\|\hat{\hat{x}}_{T + 1 | T} - x^*_{T + 1}\|^2 \mid \mathcal{F}_T} = \E{\|\hat{x}_{T + 1 | T} - x^*_{T + 1}\|^2 \mid \mathcal{F}_T} + \|A\| \lambda r \sum_{s = 1}^T \|K_s\| \delta^{T - s} (1 - a^*_s).  \]
\end{lemma}
\begin{proof}
Recall that $\hat{\hat{x}}_{T + 1 | T} = A \hat{\hat{x}}_{T | T}$ and likewise for $\hat{x}_T$. It follows that
\begin{align*} 
\|\hat{\hat{x}}_{T + 1 | T} - \hat{x}_{T + 1 |T}\| 
&\le \|A\|\|\hat{\hat{x}}_{T | T} - \hat{x}_{T | T}\| \\
&\le \|A\| \lambda \sum_{s = 1}^T \delta^{T - s} \|K_s(y^*_s - y''_s)\| \\
&\le \|A\| \lambda \sum_{s = 1}^T \delta^{T - s} (1 - a^*_s) \|K_s(y^*_s - y'_s)\| \\
&\le \|A\| \lambda \sum_{s = 1}^T \delta^{T - s} (1 - a^*_s) \|K_s\| r
\end{align*}
where the second inequality follows from Corollary~\ref{corr:stability} and the last inequality follows from the fact that $y''_t = y^*_t$ unless $\tilde y_t$ is guaranteed to be a corruption, i.e. $a^*_s = 0$, based on \eqref{eqn:guaranteed-corruption}.
\end{proof}
\begin{theorem}
Let $\delta > 0$ be arbitrary.
Let $\mathcal{F}_i$ is the filtration generated by all random variables from the generative model up to time $i$ (i.e. $a^*_t,y^*_t,x^*_t$).
Suppose that the collection of random variables $y'_{i1},\ldots,y'_{ii}$ are measurable with respect to $\mathcal{F}_i$. 
Suppose\footnote{Such an upper bound can be derived from the other assumptions we make, see e.g. \cite{schick1994robust,moore1980coping}.} that $\|K_i\| \le K$ for all $i$.
Furthermore, suppose that there exists $r \ge 0$ such that
\begin{equation}
 \max_{t = 1}^T |y^*_t - y'_t| \le r.    
\end{equation}
with probability at least $1 - \delta/2$. Then with $\hat{\hat{x}}_{i + 1 | i}$ defined as in Lemma~\ref{lem:filter1} (in terms of $y'_{i1},\ldots,y'_{ii}$), we have
\begin{align*} 
&\frac{1}{T}\sum_{i = 1}^T \E{\|\hat{\hat{x}}_{i + 1 | i} - x^*_{i + 1}\|^2 \mid \mathcal{F}_i}
-  \frac{1}{T} \sum_{i = 1}^T \E{\|\hat{x}_{i + 1 | i} - x^*_{i + 1}\|^2 \mid \mathcal{F}_i} \\
&\qquad \le \frac{\|A\| \lambda r K}{1 - \delta} (\eta + O(\sqrt{\eta \log(2/\delta)/T} + \log(2/\delta)/T)). 
\end{align*}
with probability at least $1-\delta$.
\end{theorem}
\begin{proof}
Observe that
\begin{align*} 
\sum_{i = 1}^T \sum_{s = 1}^i \|K_i\| \delta^{i - s} (1 - a^*_s) 
&\le K \sum_{i = 1}^T \sum_{s = 1}^i \delta^{i - s} (1 - a^*_s) \\
&= K \sum_{s = 1}^T (1 - a^*_s) \sum_{i = s}^T \delta^{i - s} \\
&\le \frac{K}{1 - \delta} \sum_{s = 1}^T (1 - a^*_s)
\end{align*}
and
\[ \sum_{s = 1}^T (1 - a^*_s) \le \eta T + O(\sqrt{\eta T \log(2/\delta)} + \log(2/\delta)) \]
with probability at least $1 - \delta/2$ by Bernstein's inequality. Therefore the result follows from Lemma~\ref{lem:filter1} and the union bound.
\end{proof}
Now we can apply the two-stage approach by combining it with the SoS filter. 
In this case, $y'_{i1},\ldots,y'_{ii}$ are simply the output of the offline SoS run at time $i$. Because the key input to the previous Theorem is the confidence band assumption, we only need to run the first (simpler) SoS program in \textsc{SoSKalman}, although running the second step won't hurt. 
\begin{theorem}\label{thm:online}
Suppose as described above that $y'_{i1},\ldots,y'_{ii}$ are the output of rounding Program~\ref{program:sos} as in \textsc{SoSKalman}, and the assumptions of Corollary~\ref{cor:warmstart} are satisfied. Suppose (for simplicity) that the initialization radius $R = O(r''/\sqrt{d})$ where $r''$ is as defined below.
With probability at least $1 - \delta$, 
\begin{align*} 
&\frac{1}{T}\sum_{i = 1}^T \E{\|\hat{\hat{x}}_{i + 1 | i} - x^*_{i + 1}\|^2 \mid \mathcal{F}_i}
-  \frac{1}{T} \sum_{i = 1}^T \E{\|\hat{x}_{i + 1 | i} - x^*_{i + 1}\|^2 \mid \mathcal{F}_i} \\
&\qquad \le \frac{\|A\| \lambda r'' K}{1 - \delta} (\eta + O(\sqrt{\eta \log(2/\delta)/T} + \log(2/\delta)/T))
\end{align*}
where $r'' = O(\|B\|\rho \epsilon_{\mathsf{geo}}^{1/2})$ as in Corollary~\ref{cor:warmstart}.  
\end{theorem}
\begin{remark}
The above version of the result is stated in terms of conditional expectations, which means we integrate over the conditional randomness of $x^*_{i + 1}$ (i.e. over the Gaussian posterior) when evaluating the prediction. This is analogous to looking at pseudoregret vs. regret in bandits, see e.g. \cite{bubeck2012regret}.  By a standard martingale concentration argument, this can be converted into a high probability upper bound on $\frac{1}{T}\sum_{i = 1}^T \|\hat{\hat{x}}_{i + 1 | i} - x^*_{i + 1}\|^2
-  \frac{1}{T} \sum_{i = 1}^T \|\hat{x}_{i + 1 | i} - x^*_{i + 1}\|^2$ if desired. 
\end{remark}
Note that for simplicity, we assumed that the width $R$ of the prior in the initialization $x_0 \sim N(0,R^2 I_d)$ was not too large, to simplify the final bound. It is possible to modify the above in a straightforward way and get a more complex statement showing the dependence on $R$; just as in the SoS filter, the  the term involving $R$ will decay exponentially fast as time goes on. 

%% file: wiener.tex
\section{Stationary Case: Robust Wiener Filter}\label{apdx:wiener}
In this Appendix, we consider the case where the matrix $A$ in the linear dynamical system is \emph{stable}, i.e. has all eigenvalues strictly within the unit ball of radius one. Equivalently \cite{kailath2000linear} the linear dynamical system is \emph{exponentially stable}:
\begin{equation}\label{eqn:exp-stable}
\|A^t\| \le \rho e^{-\gamma t} 
\end{equation}
for some positive constants $\rho,\gamma$ dependent on $A$. 
In this case, the state of the dynamical system, stared at any initialization, will converge to a stationary measure $N(0, \Sigma)$
with
\[ \Sigma = \sigma^2(I + AA^T + (A^2)(A^2)^T + \cdots) = \sigma^2 \sum_{k = 0}^{\infty} (A^k)(A^k)^T \]
and the convergence of the infinite matrix sum follows from \eqref{eqn:exp-stable}. 

Because the system will converge to the stationary measure from any initialization, it makes sense (and is conventional) to focus on the case of a \emph{stationary} system where time ranges from $-\infty$ to $\infty$. In this case the Kalman filter simplifies to the \emph{Wiener filter}; we recall the relevant facts here. (The precise statement of Theorem 8.4.1 in \cite{kailath2000linear} is more detailed and explains how to solve for $K$ as the unique stabilizing solution of a Discrete-time Algebraic Ricatti Equation.) 
\begin{theorem}[Theorem 8.4.1 of \cite{kailath2000linear}]
In the time-invariant stationary state space model described above,
the minimal-squared error predictor  $\hat{x}_{t + 1}$ of the next state $x_{t + 1}$
satisfies the recursion
\begin{align*}
    \hat{x}_{t + 1} := \E{x_{t + 1} \mid (y_j)_{t = -\infty}^j} = A \hat{x}_t + K(y_t - B \hat{x}_t) = (A - KB) \hat{x_{t}} + K y_t
\end{align*}
for some gain matrix $K$ which is determined by the system parameters but does not depend on $t$, and such that $A - KB$ is stable (i.e. has all eigenvalues within the open unit disk). 
\end{theorem}

Equivalently, we can write the Wiener filter as
\[ \hat{x}_{t} = \sum_{s = 1}^{\infty} (A - KB)^{s - 1} K y_{t - s}. \]

We now show that in this stationary setting, a truncated version of the Wiener has reasonable robustness guarantees. This contrasts with the much more difficult  setting which we focus on in the main text (uniformly but not strictly stable dynamics), where naively truncating the Kalman filter makes total error $\Omega(T^2)$, see Technical Overview.
Given parameters $h \ge 1, \tau \ge 0$ we formally define the truncated filter by
\[ \hat{x}_{t,h,\tau} := f_{\tau}\left(\sum_{s = 1}^{h} (A - KB)^{s - 1} K y_{t - s}\right) \]
where
\[ f_{\tau}(x) = x \cdot \bone{\|x\| \le \tau} \]
zeroes out its input if its input $x$ has norm larger than $\tau$. 
\begin{lemma}\label{lem:wiener-truncation}
With the above notation,
\begin{align*} 
\MoveEqLeft \sqrt{\E{\|\hat x_t - \hat x_{t,h,\tau}\|^2}} \\
&\le \|A - KB\|^h \left(\sum_{s = 1}^{\infty} \|A - K B\|^{s - 1}\right) \sqrt{\E{\|K y_0\|^2}} + \frac{1}{\tau} \left(\sum_{s = 1}^{\infty} \|A - K B\|^{s - 1}\right)^2 \sqrt{\E{\|K y_0\|^4}}.  
\end{align*}
\end{lemma}
\begin{proof}
By the $L_2$ triangle inequality,
\begin{equation}\label{eqn:truncation-error}
\sqrt{\E{\|\hat x_t - \hat x_{t,h,\tau}\|^2}} \le \sqrt{\E{\|\hat x_t - \sum_{s = 1}^{h} (A - KB)^{s - 1} K y_{t - s}\|^2}} + \sqrt{\E{\|\hat x_{t,h,\tau} - \sum_{s = 1}^{h} (A - KB)^{s - 1} K y_{t - s}\|^2}}. 
\end{equation}
To bound the first term on the right hand side, observe that
\begin{align*} 
\E*{\|\hat x_t - \sum_{s = 1}^{h} (A - KB)^{s - 1} K y_{t - s}\|^2} 
&= \E*{\left\|\sum_{s = h + 1}^{\infty} (A - KB)^{s - 1} K y_{t - s}\right\|^2} \\
&\le \E*{\left(\sum_{s = h + 1}^{\infty} \|A - KB\|^{s - 1} \|K y_{t - s}\|\right)^2} \\
&\le \left(\sum_{s = h + 1}^{\infty} \|A - K B\|^{s - 1}\right) \E*{\sum_{s = h + 1}^{\infty} \|A - K B\|^{s - 1} \|K y_{t - s}\|^2} \\
&\le \|A - K B\|^{2h} \left(\sum_{s = 1}^{\infty} \|A - K B\|^{s - 1}\right)^2 \E{\|K y_0\|^2}
\end{align*}
where the second inequality is Cauchy-Schwarz and in the last step we used stationarity.
To upper bound the second term on the right hand side of \eqref{eqn:truncation-error}, observe
\begin{align*}
\MoveEqLeft \E*{\|\sum_{s = 1}^{h} (A - KB)^{s - 1} K y_{t - s}\|^2 \bone*{\|\sum_{s = 1}^{h} (A - KB)^{s - 1} K y_{t - s}\|^2 > \tau^2}} \\
&\le \sqrt{\E*{\|\sum_{s = 1}^{h} (A - KB)^{s - 1} K y_{t - s}\|^4} \Pr*{\|\sum_{s = 1}^{h} (A - KB)^{s - 1} K y_{t - s}\|^2 > \tau^2}} \\
&\le \frac{1}{\tau^2} \E*{\|\sum_{s = 1}^{h} (A - KB)^{s - 1} K y_{t - s}\|^4}
\end{align*}
where the first inequality is Cauchy-Schwarz and the second inequality is Markov's inequality. Observe by Holder's inequality that 
\begin{align*}
\|\sum_{s = 1}^{h} (A - KB)^{s - 1} K y_{t - s}\|^4
&\le \left(\sum_{s = 1}^h \|A - K B\|^{s - 1} \|K y_{t - s}\|\right)^4 \\
&\le \left(\sum_{s = 1}^h \|A - K B\|^{s - 1}\right)^3\sum_{s = 1}^h \|A - K B\|^{s - 1} \|K y_{t - s}\|^4
\end{align*}
so using stationarity and linearity of expectation
\[
\E*{\|\sum_{s = 1}^{h} (A - KB)^{s - 1} K y_{t - s}\|^4} \le \left(\sum_{s = 1}^{\infty} \|A - K B\|^{s - 1}\right)^4 \E{\|K y_0\|^4}.
\]
\end{proof}

\begin{theorem}
Suppose that $\eta$ fraction of responses are arbitrarily corrupted, let $\hat x_t$ denote the Wiener filter run on the uncorrupted responses and let $x'_{t,h,\tau}$ denote the truncated Wiener filter run on the corrupted responses. Then taking $h = \Theta_{A,K,B}(\log(1/\eta))$ and $\tau = \Theta_{A,K,B}([(1/\eta)\log(1/\eta)]^{1/3})$,
we have that
\[ \E{\|\hat{x}_0 - x'_{0,h,\tau}\|^2} \le C \eta^{1/3} \log(1/\eta)^{1/3} \]
where $C$ is a constant independent of $\eta$ but allowed to depend on all other system constants.
\end{theorem}
\begin{proof}
By the union bound, there is a probability of at most $\eta h$ that one of observations in the $h$ previous time steps is corrupted. Since the probability of corruption is independent of the process, the expected error conditional on a corruption occuring is at most $O(\tau^2 + E \|x_0\|^2)$. Conditional on no corruption in the last $h$ timesteps, the error is upper bounded by the previous Lemma~\ref{lem:wiener-truncation}. Therefore the error is upper bounded by
\begin{multline} O\Bigg(\eta h (\tau^2 + E \|x_0\|^2) + \|A - KB\|^h \left(\sum_{s = 1}^{\infty} \|A - K B\|^{s - 1}\right) \sqrt{\E{\|K y_0\|^2}} \\
+ \frac{1}{\tau} \left(\sum_{s = 1}^{\infty} \|A - K B\|^{s - 1}\right)^2 \sqrt{\E{\|K y_0\|^4}}\Bigg) 
\end{multline}
Note that provided $\eta h \tau^2 \to 0$, $h \to \infty$, and $1/\tau \to 0$ the right hand side goes to zero as $\eta \to 0$. Taking $h = \Theta_{A,K,B}(\log(1/\eta))$ and $\tau = \Theta_{A,K,B}([(1/\eta)\log(1/\eta)]^{1/3})$ gives the result. 
\end{proof}
We note that both $C$ and the dependence on $\eta$ can easily be further optimized, but we have not done this for simplicity's sake. Also, note that this argument indeed relied strongly on the fact that the matrix $A$ is (strictly) stable; when $A$ has eigenvalues on the unit circle, there is no constant threshold $\tau$ which will work in the argument, since the dynamical system generally won't stay near the origin. 

%% file: examples.tex
\section{Some Instructive Examples}\label{apdx:examples}

In this section we present several examples that illustrate important subtleties in the model we consider, and which drove many of the insights underlying our analysis.

\subsection{Noisily Observed Simple Random Walk}
\label{subsec:srw}
Let $m = d = 1$ so that $A$ and $B$ are scalars. Let $A = 1$, $B = 1$, and $\sigma^2 = \tau^2 = 1$ so that the trajectory $\brc{x_i}$ is given by a random walk over $\R$ with Gaussian increments. Note that this model is completely observable with $s = 1$ and uniformly stable with $\rho = 1$. Moreover, observe that the final iterate $x_{T-1}$ will typically satisfy $\|x_{T-1}\| \approx \sqrt{T}$. This illustrates two important phenomena: 

\paragraph{Oblivious outlier removal fails.} The most intuitively obvious way to deal with large corruptions is to perform some kind of simple outlier removal, like deleting all observations $y_t$ with $\|y_t\| \ge R$ for some threshold $R$, and then running a standard algorithm like Kalman filtering on whatever remains. In this example, $R$ would have to be at least $\Omega(\sqrt{T})$ in order to avoid removing a large portion of the ground truth observations. It can be straightforwardly checked, via the Kalman filter recursion, that an adversary making $\eta T$ corruptions of size $\Theta(\sqrt{T})$ can cause the Kalman filter to incur excess risk scaling as $\Omega(\eta T)$, which is much larger than our guarantee of $\wt{O}(\eta)\cdot\log\log T$. 

The fact that simple outlier removal procedures fail is the key difficulty with dealing with marginally stable dynamical systems like the above. In contrast, in Appendix~\ref{apdx:wiener} we show that if the system is \emph{strictly} stable (i.e. all eigenvalues of $A$ are bounded away from one) then a fairly simple truncated filter will actually have reasonable robustness guaranteees.

\paragraph{Impossibility of handling adversarial corruptions in adversarial locations.} 
Suppose that instead of the adversary controlling a random fraction of the timesteps, the adversary is allowed to choose which $\eta T$ timesteps it gets to corrupt. Now suppose that for the timesteps $T - 2\eta T$ up to $T - 1$, the adversary samples an independent random walk $x'_{T - 2\eta T},\ldots,x'_{T-1}$ with Gaussian increments conditioned on $x_{T - 2\eta T} = x'_{T - 2\eta T}$, chooses a random subset of $\eta T$ coordinates from $[T - 2\eta T, T]$ to corrupt, and for each corrupted step $i$ sets $y_i = x'_i + \calN(0,1)$ instead of $y^*_i = x_i + \calN(0,1)$. By construction, it is information-theoretically impossible to distinguish from its observations which of the following is the ground truth: the actual true trajectory $(x_0,\ldots,x_{T-1})$, or the ``fake'' trajectory $(x_0,\ldots,x_{T - 2\eta T}, x'_{T - 2 \eta T + 1},\ldots, x'_{T-1})$. Note that any choice of estimate $\hat{x}_i$ will be distance at least $|x_i - x'_i|/2$ from one of these two trajectories.

By linearity of expectation, the algorithm will pay mean squared error at least $$\frac{1}{T}\sum_{i = T - 2\eta T}^T \frac{1}{2} \mathbb{E} \left(\frac{x_i - x'_i}{2}\right)^2 = \Omega(\eta\cdot \eta T) = \Omega(\eta^2 T)$$ to estimate the trajectory. This is again much larger than the $\wt{O}(\eta)\cdot\log\log T$ guarantee we give in the setting with randomly located corruptions.
Based on this example, we see that it is information-theoretically necessary for the corruption to be \emph{spread out randomly}, and otherwise it is not possible to obtain strong estimation guarantees. It also demonstrates that the assumption $\eta < 1/2$ is required, because the attack above succeeds even with access to random locations when $\eta = 1/2$.

\subsection{Coordinate Cycling Dynamics}
\label{subsec:switching}

The definition of complete observability ensures that with $s$ consecutive uncorrupted observations $y_0,\ldots,y_{s-1}$, we can estimate $x_0$ nontrivially well. Here we give an example showing that if some of these observations are corrupted, one needs to go out to $\omega(s)$ time steps to be able to estimate $x_0$ nontrivially well.

Consider a dynamics where $A(x_1,\ldots,x_d) = (x_d,x_1,\ldots,x_2)$, i.e. $A$ cyclically permutes the coordinates of the state, and the observations are given by the 1-dimensional projection $B(x_1,\ldots,x_d) = x_1$. It is easy to see that this model is observable after $s = d$ steps.

However, in the presence of corruptions with corruption fraction $\eta > 1/d$, we generally cannot estimate the state from just $s = d$ observations. This is because if $\eta > 1/d$, there is a significant probability that one of the first $s$ observations was corrupted, and so we do not have an accurate estimate of $x_0$ in the coordinate corresponding to that timestep. In fact, for any coordinate $j\in[d]$, the distribution over the earliest cycle of $d$ timesteps in which we get a clean observation of $j$ is geometric with parameter $1 - \eta$, so if $\eta$ is a constant, we need to make $\Theta(d\log d)$ observations in order to guarantee that we make a clean observation of every coordinate at least once with probability $2/3$.

\subsection{Subspace That Is Hard to Robustly Observe}
\label{sec:3dexample}

We consider a variant of the previous example which illustrates some additional properties relevant to our analysis. Let
\[ A = \begin{pmatrix}
0 & 1 & 0\\
1 & 0 & 0\\
1 & 1 & 1/2
\end{pmatrix}\]
and define $B:\R^3\to\R^2$ by $B(x_1,x_2,x_3) = (x_1,x_3)$. Observe that
$BA^k (x_1,x_2,x_3)$ is either $(x_1, (1 + 1/2 + \cdots)(x_1 + x_2) + x_3/2^k)$ or $(x_2, (1 + 1/2 + \cdots)(x_1 + x_2) + x_3/2^k)$ depending on the parity of $k$. In particular, observe that $\|BA^k e_3\|^2 = 1/2^{2k}$ which shrinks exponentially fast with $k$, so $\sum_{k = 0}^{\infty} \|BA^k e_3\|^2 = \Theta(1)$ and the sum is dominated by its first few terms. The significance of this is if $x^*_0 = (x^*_{01},x^*_{02},x^*_{03})$ 
and the first few observations $\tilde{y}_0,\tilde{y}_1,\tilde{y}_2,...$ are corrupted (which can happen with decent probability if $\eta$ is a constant), then almost all of the information about $x^*_{03}$ is lost, even though this is system is observable with parameter $s=2$.

Without delving into more details, we think of the $z$ coordinate as an unobservable subspace i.e a subspace of the state space that can't be estimated once a few timesteps are corrupted.

\paragraph{Unobservable Subspace Can Contain Unbounded Error.}
\label{subsec:unobservableunbounded}
Using the same dynamics $A$ and observation $B$, we have established that starting from $(x_1,x_2,x_3)$ that if the first few observations are corrupted then measurement of $x_3$ is lost at an exponential rate.  This would not be a problem if we anticipated the $z$ coordinate to converge quickly to zero so that over a long trajectory it suffices to estimate the $x$ and $y$ coordinates alone.  However, this is not true.  If $(x_{1}, x_{2}, x_{3}) = (R,R,R)$ the $z$ coordinate can have size scaling with $R$ over the entire length of the trajectory.  That is to say, the state vector can have an unbounded component in the unobservable subspace.      

\paragraph{Estimating Unobservable Error of the Present from Observable Error of the Past.}
\label{subsec:past2present}

Of course, this does not appear to be an intractable problem for the specific dynamics $A$ and measurement $B$ that we have chosen here. We can give up on estimating the state in the $z$ coordinate directly.  Instead, if we can estimate the $x$ and $y$ coordinates at time $t$ we can build a good estimate of the $z$ coordinate at the next timestep $t+1$.  In this manner, it is possible to estimate the state in the unobservable subspace, $z$, using estimates of the observable subspace $(x,y)$ of the past.  

At first glance, the property that enables us to estimate the state of the unobservable subspace of the present from the observable subspace of the past seems to impose significant constraints on the dynamics $A$ and measurement $B$.  Surprisingly, this is not the case.  In fact, as our analysis in Sections~\ref{sec:old_decay} and \ref{sec:decay} show, this analysis can be carried out for \emph{any} linear dynamical system that is uniformly stable and completely observable.




\subsection{Lower Bound When \texorpdfstring{$\tau\to 0$}{tau->0}}
\label{sec:sigbytau}
Recall that our guarantee in Theorem~\ref{thm:sos_informal} has a term depending on the ratio $\sigma^2/\tau^2$, suggesting that the excess risk achieved by our estimator blows up as $\tau\to 0$ when the other parameters are fixed. This turns out to be unavoidable. Consider the dynamics given by $A(x,y) = (0,x)$ with $B = \Id$. We may observe conflicting information about the first coordinate of initial state $x_0$ from the first two observations $y_0$ and $y_1$ (and no information about the first coordinate of $x_0$ from subsequent observations). We thus have no way of knowing which of the two observations was corrupted, and so an error of size $\Omega(R^2/\tau^2)$ is unavoidable. For a similar reason, the guarantee needs to blow up if $\tau \to 0$ with $\sigma,\eta > 0$ fixed.

\subsection{Dimension Dependence is Necessary}
\label{app:ddepend}

In this section we describe a simple example showing that the dependence of the excess risk on the ambient dimension $d$ in our main guarantee is unavoidable.

\begin{lemma}
    Consider the linear dynamical system where $A = 0$, $B = \Id$, and $\sigma^2 = \tau^2 = R^2 = 1$. If the corrupted observations are given by drawing an independent sample from $\calN(0,2\Id)$, then for any algorithm that takes as input the corrupted observations $\brc{y_i}$ and outputs a trajectory $\brc{\wh{x}_i}$, we have
    \begin{equation}
        \E{L(\wh{x})}\ge \E{\mathsf{OPT}} + \Omega(\eta d),
    \end{equation}
    where the expectation is over the randomness of the algorithm, which observations were corrupted, and the process and observation noise.
\end{lemma}

\begin{proof}
    In this case, every observation $y_i$ is just a fresh draw from $\calN(0,2\Id)$ obtained either from sampling $x^*_i\sim\calN(0,\Id)$ and adding $v^*_i\sim\calN(0,\Id)$, or from sampling independently of the trajectory. Note that $L(\wh{x})$ can be expressed as
    \begin{equation}
        L(\wh{x}) = \frac{1}{T}\left(\sum^{T-1}_{i=0}\left(a^*_i\norm{\wh{x}_i - y_i}^2 + \norm{\wh{x}_i}^2\right) + \norm{\wh{x}_0}^2\right).
    \end{equation}
    We first compute what $\mathsf{OPT}$ would be, i.e. what value one could achieve if the indices $i$ for which $a^*_i = 1$ were known. In that case we would take
    \begin{equation} 
        \wh{x}_i = \begin{cases}
            y_i/3 & a^*_i = 1 \ \text{and} \ i = 0 \\
            y_i/2 & a^*_i = 1 \ \text{and} \ i > 0 \\
            0 & a^*_i = 0
        \end{cases}
    \end{equation} and conclude that in expectation over the randomness of which observations were corrupted,
    \begin{equation}
        \E{\mathsf{OPT}} = \E*{\frac{1}{T}\left( \frac{1}{6}a^*_0\norm{y_0}^2 + \sum^{T-1}_{i=0}\frac{1}{2}a^*_i\norm{y_i}^2\right)} = \frac{1-\eta}{6T}\norm{y_0}^2 + \frac{1-\eta}{2}\cdot \avgT\norm{y_i}^2. \label{eq:eOPT}
    \end{equation}
    Now consider any algorithm that doesn't know the indices of the corrupted observations. From their perspective, $\brc{y_i}$ is an i.i.d. sequence of draws from $\calN(0,2\Id)$, so we may without loss of generality assume that it forms each $\wh{x}_i$ independently of all other timesteps. Suppose that in timestep $i$, it uses randomized algorithm $\calA_i:\R^d\to\R^d$ to take in observation $y_i$ and outputs $\wh{x}_i$, where $\calA_i$ is oblivious to $a^*_i$. As $\norm{\wh{x}_i - y_i}^2$ and $\norm{\wh{x}_i}^2$ are both minimized when $\wh{x}_i$ has no component orthogonal to $y_i$, we may assume without loss of generality that $\calA(y)_i = \zeta_i\cdot y_i$ for some random $\zeta_i$. Then the sequence $\brc{\calA(y_i)}$ incurs clean posterior negative log likelihood at least
    \begin{equation}
        \frac{1}{T}\sum^{T-1}_{i=0} \left(a^*_i(\zeta_i - 1)^2 + \zeta_i^2\right)\norm{y_i}^2, \label{eq:subopt}
    \end{equation}
    where we have lower bounded $\frac{1}{T}\norm{\wh{x}_0}^2$ by zero. Now because the $\zeta_i$'s are independent of the $a^*_i$'s, the expectation of \eqref{eq:subopt} over the randomness of the algorithm and the $a^*_i$'s is
    \begin{equation}
        \avgT \E[\zeta_i]{(1 - \eta)(\zeta_i - 1)^2 + \zeta_i^2}\cdot\norm{y_i}^2.
    \end{equation} The coefficient for each summand achieves its minimum of $\frac{1 - \eta}{2 - \eta}$ when $\zeta_i$ is deterministically $\frac{1 - \eta}{2-\eta}$, so any algorithm oblivious to the $a^*_i$'s incurs clean posterior negative log likelihood $L(\wh{x})$ at least $\frac{1-\eta}{2-\eta}\cdot\avgT \norm{y_i}^2$ in expectation. Contrasting this with \eqref{eq:eOPT} and noting that $\frac{1 - \eta}{2 - \eta} - \frac{1 - \eta}{2} = \Omega(\eta)$ for $0\le \eta < 1/2$, we conclude that in expectation over the randomness of the algorithm and the $a^*_i$'s,
    \begin{equation}
        \E{L(\wh{x})}\ge \E{\mathsf{OPT}} + \Omega(\eta)\avgT \norm{y_i}^2 - \frac{1 - \eta}{6T}\norm{y_0}^2.
    \end{equation}
    Taking a further expectation over the $y_i$'s concludes the proof.
\end{proof}

%% file: defer.tex
\section{Deferred Proofs from Section~\ref{sec:TloglogT}}
\label{app:defer}

In this section we collect deferred proofs from Section~\ref{sec:TloglogT}.

\subsection{Proof of Lemma~\ref{lem:anticonc_2}}

\begin{proof}
    Without loss of generality we can assume $\ell = 0$. By definition of $b^*_0$, we have \begin{equation}
        b_0 b^*_0 \sum^t_{i = 1} a^*_i \Pi (A^i)^{\top} B^{\top} B A^i \Pi \succeq b_0 b^*_0 \left((1 - \eta)\cdot \Pi\calO_t\Pi - O\left(\rho^2\norm{B}^2\sqrt{t\log(d/\delta_1)}\right)\cdot \Pi\right),
    \end{equation} and similarly by Constraint~\ref{item:subsample}, we have that \begin{equation}
        b_0 b^*_0\sum^t_{i = 1} (1 - a_i) \Pi (A^i)^{\top} B^{\top} B A^i \Pi \preceq b_0 b^*_0\left(\eta\cdot \Pi\calO_t \Pi + O\left(\rho^2\norm{B}^2\sqrt{t\log(t/\delta_1)}\right) \cdot \Pi\right).
    \end{equation} Writing $a^*_i a_i = a^*_i - a^*_i(1 - a_i) \ge a^*_i - (1 - a_i)$ and subtracting the above two equations gives
    \begin{equation}
        b_0 b^*_0 \sum^t_{i = 1} a^*_ia_i \Pi (A^i)^{\top} B^{\top} B A^i \Pi \succeq b_0 b^*_0 \left((1 - 2\eta)\cdot \Pi\calO_t\Pi - O\left(\rho^2\norm{B}^2\sqrt{t\log(t/\delta_1)}\right)\cdot \Pi\right).
    \end{equation}
    Recalling that $\Pi \preceq \thres^{-1} \cdot \Pi\calO_t\Pi$ by definition of $\Pi$, we would like $t$ to be large enough that \begin{equation}
        1 - 2\eta - O\left(\thres^{-1}\rho^2\norm{B}^2\sqrt{t\log(t/\delta_1)}\right) \ge 1/3.
    \end{equation} 
    As $\eta < 0.49$, it suffices for $\thres \ge O(\rho^2 \norm{B}^2\sqrt{t\log(t/\delta_1)})$. 
    Recalling that $\thres = \frac{\kappa \t}{40000\rho^4}$, it's sufficient that $t \ge \wt{\Omega}\left(\kappa^{-2} \rho^{12} \norm{B}^4 \log(d/\delta1)\right)$.
\end{proof}

\subsection{Proof of Lemma~\ref{lem:observable-case-analysis2}}

\begin{proof}
    Without loss of generality we can assume $\ell = 0$. In addition to the lower bound of \eqref{eq:psdlb2}, we also have a degree-2 SoS proof of the upper bound
    \begin{equation}\label{eqn:unobs-psd2}
        \sum_{i=1}^t a^*_i a_i \Pi^\perp (A^i)^T B^T B A^i \Pi^\perp \preceq  \Pi^\perp \mathcal{O}_t \Pi^\perp \preceq \thres \cdot I,
    \end{equation}
    where in the first step we used that $a_i a^*_i \le 1$ by Constraint~\ref{item:boolean} and in the second step we used the definition of $\Pi$.
    
    For convenience, define $q\triangleq x_0 - x^*_0$. We proceed by casework on whether there is a gap between $\psE*{(\Pi q)^{\top} \calO_t (\Pi q)}$ and $\psE{\thres\norm{\Pi q}^2}$:
    
    \noindent{\textbf{Case 1}}: $\psE*{b_0 b^*_0\| \Pi q \|^2} \ge \psE*{\frac{1}{4000\rho^2\thres}b_0 b^*_0\sum_{i=1}^t  \| BA^i \Pi q\|^2}$.
    
    The analysis for this case is very similar to the analysis in Lemma~\ref{lem:unobservable-decay}. We have 
    \begin{align*}
        \psE*{b_0b^*_0\| \Pi q \|^2} 
        &\geq \psE*{b_0b^*_0\frac{1}{4000\rho^2\thres}\sum_{i=1}^t  \| BA^i \Pi q\|^2} \\
        &= \psE*{b_0b^*_0\frac{1}{4000\rho^2\thres}\sum_{j = 1}^{t/s}  q^T\Pi {A^{js}}^T \calO_s A^{js} \Pi q} \\
        &\geq \psE*{b_0b^*_0\frac{10\rho^2 s}{t}\sum_{j = 1}^{t/s}  \|A^{js} \Pi q\|^2},
    \end{align*}
    where in the last step we used the definition of $\thres$ and the assumption that $\sigma_{\min}(\calO_s) \ge \kappa s$. Rearranging, we obtain 
    \begin{equation}
        \frac{1}{10\rho^2} \psE{b_0 b^*_0\| \Pi q \|^2} \geq  \frac{1}{t/s}\sum_{j=1}^{t/s} \psE*{b_0 b^*_0 \|A^{js} \Pi q\|^2}.
    \end{equation}
    Therefore, there exists some index $j\in[t/s]$ for which $\psE*{b_0 b^*_0\norm{A^{js}\Pi q}^2} \le \frac{1}{10\rho^2}\psE*{b_0 b^*_0 \norm{\Pi q}^2}$. By uniform stability, we obtain the first desired outcome in the lemma statement.

    \noindent{\textbf{Case 2}}: $ \psE*{b_0 b^*_0 \| \Pi q \|^2} \le \psE*{\frac{1}{4000\rho^2\thres} b_0 b^*_0\sum_{i=1}^t  \| BA^i \Pi q\|^2}$.

    We have
    \begin{equation} \label{eqn:fix-main2}
    \psE*{b_0b^*_0\| \Pi q \|^2} \leq \psE*{\frac{1}{4000\rho^2\thres}b_0 b^*_0\sum_{i=1}^t\| BA^i \Pi q\|^2} \leq \psE*{b_0b^*_0\frac{1}{40\rho^2\thres}\sum_{i=1}^t a^*_i a_i\| BA^i \Pi q\|^2}.
    \end{equation}
    where in the last step we invoked Lemma~\ref{lem:anticonc_2}.
    We would like to upper bound $b_0b^*_0\sum^t_{i=1} a^*_i a_i \norm{BA^i\Pi q}^2$ in terms of $b_0 b^*_0\sum^t_{i = 1} a^*_i a_i \norm{BA^i \Pi^{\perp} q}^2$. 
    For any $i \in[t]$, we have the following sequence of inequalities in degree-4 SoS \begin{align}
        a_i a^*_i \norm*{BA^i(\Pi + \Pi^{\perp})q}^2 &= a_i a^*_i \norm*{BA^iq}^2 \\
        &= a_i a^*_i\norm*{(y_i - Bx^*_i) - (y_i - Bx_i) + (Bx^*_i - BA^i x^*_0) - (Bx_i - BA^i x_0)}^2 \\
        &\le 3\norm{v^*_i}^2 + 3\norm{v_i}^2 + 3\norm*{\sum^i_{s=1} BA^{i - s} (w_s - w^*_s)}^2\\
        &\le \left(3\norm{v^*_i}^2 + 6\norm*{\sum^i_{s=1}BA^{i-s}w^*_s}^2\right) + \left(3\norm{v_i}^2 + 6\norm*{\sum^i_{s=1}BA^{i-s}w_s}^2\right)
    \end{align}
    so by applying Fact~\ref{fact:sos_simple} to $\epsilon = \epsilon_{0,i} + \epsilon^*_{0,i}$, $v_1 = a^*_i a_i BA^i \Pi q $, and $v_2 = a^*_i a_i BA^i \Pi^\perp q $, we have a degree-4 SoS proof of
    \begin{equation}
        a^*_i a_i\| BA^i \Pi q\|^2 \leq 4a^*_i a_i\| BA^i \Pi^\perp q\|^2 + \frac{4}{3}(\epsilon_{0,i} + \epsilon^*_{0,i}).
    \end{equation}
    Summing this inequality over $i\in[t]$ and taking pseudo-expectations, we get
    \[\psE*{\sum_{i=1}^t a^*_i a_i\| BA^i \Pi q\|^2} \le \psE*{4\sum_{i=1}^ta^*_i a_i\| BA^i \Pi^\perp q\|^2 + \frac{4}{3}\sum^t_{i=1}(\epsilon_{0,i} + \epsilon^*_{0,i})}. \]
    Substituting this back into the main bound \eqref{eqn:fix-main2}, we get
    \begin{align}
    \psE*{b_0 b^*_0 \| \Pi q \|^2} &\leq  \psE*{\frac{1}{10\rho^2\thres}b_0b^*_0 \sum_{i=1}^t a^*_i a_i\| BA^i \Pi^\perp q\|^2 + \frac{1}{30\rho^2\thres}b_0 b^*_0\sum^t_{i=1}(\epsilon_{0,i} +\epsilon^*_{0,i})} \\
    &\le \psE*{\frac{1}{10\rho^2}b_0b^*_0\norm{\Pi^{\perp} q}^2 + \frac{1}{30\rho^2\thres}\sum^t_{i=1} (\epsilon_{0,i} +\epsilon^*_{0,i})},
    \end{align} where in the last step we used \eqref{eqn:unobs-psd2} and also the fact that $b_0 b^*_0 \le 1$. Unpacking the definition of $\epsilon,\thres$, we arrive at the second desired bound.
\end{proof}

\subsection{Proof of Lemma~\ref{lem:error_decay2}}

\begin{proof}
    Define $q \triangleq x_{(\ell - 1)t} - x^*_{(\ell - 1)t}$. Recalling Fact~\ref{fact:unroll}, we have by SoS triangle inequality that:
    \begin{align*}
        \|x_{\ell t} - x^*_{\ell t}\|^2
        &\le 4 \|A^t\Pi q\|^2 + 4\|A^t\Pi^{\perp}  q\|^2 + 4\gamma_{\ell} + 4\gamma^*_{\ell} \\
        &\le 4 \|A^t\Pi q\|^2  + (1/10000) \|\Pi^{\perp}  q\|^2 + 4\gamma_{\ell} + 4\gamma^*_{\ell}
    \end{align*}
    where we used triangle inequality in the first inequality, SoS triangle inequality in the second inequality, and Lemma~\ref{lem:unobservable-decay} in the third inequality. 
    
    Based on Lemma~\ref{lem:observable-case-analysis2} applied to $\ell - 1$ we either have that:
    \begin{enumerate}
        \item (Observable component decays.) We have
    \[ \psE*{b_{\ell - 1}b^*_{\ell - 1}\norm{A^{t}\Pi q}^2} \leq \frac{1}{10} \psE*{b_{\ell - 1}b^*_{\ell - 1} \| \Pi q \|^2}. \]
        Then 
        \begin{align*}
            \MoveEqLeft \psE{b_{\ell - 1} b^*_{\ell - 1}\|x_{\ell t} - x^*_{\ell t}\|^2} \\
            &\le \psE*{b_{\ell - 1} b^*_{\ell - 1}\cdot\left(4 \|A^t\Pi q\|^2  + (1/10000) \|\Pi^{\perp} q\|^2 \right) + 4\gamma_{\ell} + 4\gamma^*_{\ell}} \\
            &\le  \psE*{b_{\ell - 1} b^*_{\ell - 1}\cdot \left(\frac{2}{5}\|\Pi q\|^2 + (1/10000) \|\Pi^{\perp} q\|^2\right) + 4\gamma_{\ell} + 4\gamma^*_{\ell}} \\
            &\le \psE*{b_{\ell - 1} b^*_{\ell - 1}\cdot \frac{2}{5}\|q\|^2 + 4\gamma_{\ell} + 4\gamma^*_{\ell}}
        \end{align*}
        where in the last step we used the Pythagorean Theorem.
        \item (Observable error bounded by unobservable error.) We have
        \[  \psE*{b_{\ell - 1}b^*_{\ell - 1} \norm{\Pi q}^2} \le \psE*{b_{\ell - 1}b^*_{\ell - 1}\cdot \frac{1}{10\rho^2}\norm{\Pi^{\perp} q}^2 + \frac{1}{30\rho^2\zeta}\sum^t_{i=1}(\epsilon_{\ell,i} + \epsilon^*_{\ell,i})} \]
        Then
        \begin{align*}
            \psE{b_{\ell-1} b^*_{\ell-1}\|x_{\ell t} - x^*_{\ell t}\|^2} 
            &\le \psE*{b_{\ell - 1} b^*_{\ell - 1} \left( 4\|A^t\Pi q\|^2  + (1/10000) \|\Pi^{\perp} q\|^2\right) + 4\gamma_{\ell} + 4\gamma^*_{\ell}} \\
            &\le \psE*{b_{\ell - 1} b^*_{\ell - 1}\left(4\rho^2\|\Pi q\|^2  + (1/10000) \|\Pi^{\perp} q\|^2\right) + 4\gamma_{\ell} + 4\gamma^*_{\ell}} \\
            &\le \psE*{b_{\ell - 1} b^*_{\ell - 1} \left(\frac{1}{2}\norm{\Pi^{\perp} q}^2 + 4\gamma_{\ell} + 4\gamma^*_{\ell} + \frac{2}{15\rho^2\zeta}\sum^t_{i=1}(\epsilon_{\ell,i} + \epsilon^*_{\ell,i})\right)}.
        \end{align*}
        Since we showed the desired conclusion in both cases, the proof is complete.
    \end{enumerate}
\end{proof}

%% file: draft.bbl
\newcommand{\etalchar}[1]{$^{#1}$}
\begin{thebibliography}{CKMY20}

\bibitem[AM07]{anderson2007optimal}
Brian~DO Anderson and John~B Moore.
\newblock {\em Optimal control: linear quadratic methods}.
\newblock Courier Corporation, 2007.

\bibitem[AM12]{anderson2012optimal}
Brian~DO Anderson and John~B Moore.
\newblock {\em Optimal filtering}.
\newblock Courier Corporation, 2012.

\bibitem[BCB12]{bubeck2012regret}
S{\'e}bastien Bubeck and Nicolo Cesa-Bianchi.
\newblock Regret analysis of stochastic and nonstochastic multi-armed bandit
  problems.
\newblock {\em arXiv preprint arXiv:1204.5721}, 2012.

\bibitem[BDJ{\etalchar{+}}20]{bakshi2020robustly}
Ainesh Bakshi, Ilias Diakonikolas, He~Jia, Daniel~M. Kane, Pravesh~K. Kothari,
  and Santosh~S. Vempala.
\newblock Robustly learning mixtures of $k$ arbitrary gaussians, 2020.

\bibitem[BK20]{bakshi2020outlier}
Ainesh Bakshi and Pravesh Kothari.
\newblock Outlier-robust clustering of non-spherical mixtures.
\newblock {\em arXiv preprint arXiv:2005.02970}, 2020.

\bibitem[BKS14]{barak2014rounding}
Boaz Barak, Jonathan~A Kelner, and David Steurer.
\newblock Rounding sum-of-squares relaxations.
\newblock In {\em Proceedings of the forty-sixth annual ACM symposium on Theory
  of computing}, pages 31--40, 2014.

\bibitem[BMR21]{banks2021local}
Jess Banks, Sidhanth Mohanty, and Prasad Raghavendra.
\newblock Local statistics, semidefinite programming, and community detection.
\newblock In {\em Proceedings of the 2021 ACM-SIAM Symposium on Discrete
  Algorithms (SODA)}, pages 1298--1316. SIAM, 2021.

\bibitem[BP21]{bakshi2021robust}
Ainesh Bakshi and Adarsh Prasad.
\newblock Robust linear regression: Optimal rates in polynomial time.
\newblock In {\em Proceedings of the 53rd Annual ACM SIGACT Symposium on Theory
  of Computing}, pages 102--115, 2021.

\bibitem[BS14]{barak2014sum}
Boaz Barak and David Steurer.
\newblock Sum-of-squares proofs and the quest toward optimal algorithms.
\newblock {\em arXiv preprint arXiv:1404.5236}, 2014.

\bibitem[C{\etalchar{+}}20]{chinot2020erm}
Geoffrey Chinot et~al.
\newblock Erm and rerm are optimal estimators for regression problems when
  malicious outliers corrupt the labels.
\newblock {\em Electronic Journal of Statistics}, 14(2):3563--3605, 2020.

\bibitem[CAT{\etalchar{+}}20]{cherapanamjeri2020optimal}
Yeshwanth Cherapanamjeri, Efe Aras, Nilesh Tripuraneni, Michael~I Jordan,
  Nicolas Flammarion, and Peter~L Bartlett.
\newblock Optimal robust linear regression in nearly linear time.
\newblock {\em arXiv preprint arXiv:2007.08137}, 2020.

\bibitem[CKMY20]{chen2020online}
Sitan Chen, Frederic Koehler, Ankur Moitra, and Morris Yau.
\newblock Online and distribution-free robustness: Regression and contextual
  bandits with huber contamination.
\newblock {\em arXiv preprint arXiv:2010.04157}, 2020.

\bibitem[CSV17]{charikar2017learning}
Moses Charikar, Jacob Steinhardt, and Gregory Valiant.
\newblock Learning from untrusted data.
\newblock In {\em Proceedings of the 49th Annual ACM SIGACT Symposium on Theory
  of Computing}, pages 47--60, 2017.

\bibitem[DHKK20]{diakonikolas2020robustly}
Ilias Diakonikolas, Samuel~B Hopkins, Daniel Kane, and Sushrut Karmalkar.
\newblock Robustly learning any clusterable mixture of gaussians.
\newblock {\em arXiv preprint arXiv:2005.06417}, 2020.

\bibitem[DKK{\etalchar{+}}17]{diakonikolas2017being}
Ilias Diakonikolas, Gautam Kamath, Daniel~M Kane, Jerry Li, Ankur Moitra, and
  Alistair Stewart.
\newblock Being robust (in high dimensions) can be practical.
\newblock In {\em Proceedings of the 34th International Conference on Machine
  Learning-Volume 70}, pages 999--1008. JMLR. org, 2017.

\bibitem[DKK{\etalchar{+}}18]{diakonikolas2018robustly}
Ilias Diakonikolas, Gautam Kamath, Daniel~M Kane, Jerry Li, Ankur Moitra, and
  Alistair Stewart.
\newblock Robustly learning a gaussian: Getting optimal error, efficiently.
\newblock In {\em Proceedings of the Twenty-Ninth Annual ACM-SIAM Symposium on
  Discrete Algorithms}, pages 2683--2702. SIAM, 2018.

\bibitem[DKS18]{diakonikolas2018list}
Ilias Diakonikolas, Daniel~M Kane, and Alistair Stewart.
\newblock List-decodable robust mean estimation and learning mixtures of
  spherical gaussians.
\newblock In {\em Proceedings of the 50th Annual ACM SIGACT Symposium on Theory
  of Computing}, pages 1047--1060. ACM, 2018.

\bibitem[DT19]{dalalyan2019outlier}
Arnak Dalalyan and Philip Thompson.
\newblock Outlier-robust estimation of a sparse linear model using l1-penalized
  huber's m-estimator.
\newblock In {\em Advances in Neural Information Processing Systems}, pages
  13188--13198, 2019.

\bibitem[FA75]{fetzer1975observability}
Erwin~Enrique Fetzer and PM~Anderson.
\newblock Observability in the state estimation of power systems.
\newblock {\em IEEE transactions on power Apparatus and Systems},
  94(6):1981--1988, 1975.

\bibitem[GLS81]{GLS1981}
M.~Gr{\"o}tschel, L.~Lov{\'a}sz, and A.~Schrijver.
\newblock The ellipsoid method and its consequences in combinatorial
  optimization.
\newblock {\em Combinatorica}, 1(2):169--197, Jun 1981.

\bibitem[HG05]{henrion2005positive}
Didier Henrion and Andrea Garulli.
\newblock {\em Positive polynomials in control}, volume 312.
\newblock Springer Science \& Business Media, 2005.

\bibitem[HL18]{hopkins2018mixture}
Samuel~B Hopkins and Jerry Li.
\newblock Mixture models, robustness, and sum of squares proofs.
\newblock In {\em Proceedings of the 50th Annual ACM SIGACT Symposium on Theory
  of Computing}, pages 1021--1034. ACM, 2018.

\bibitem[Hop18]{hopkins2018statistical}
Samuel Hopkins.
\newblock {\em STATISTICAL INFERENCE AND THE SUM OF SQUARES METHOD}.
\newblock PhD thesis, Cornell University, 2018.

\bibitem[Hub64]{huber1964robust}
Peter~J Huber.
\newblock Robust estimation of a location parameter.
\newblock {\em The Annals of Mathematical Statistics}, pages 73--101, 1964.

\bibitem[Hub73]{huber1973robust}
Peter~J Huber.
\newblock Robust regression: Asymptotics, conjectures and monte carlo.
\newblock {\em The Annals of Statistics}, pages 799--821, 1973.

\bibitem[Kal59]{kalman1959general}
Rudolf Kalman.
\newblock On the general theory of control systems.
\newblock {\em IRE Transactions on Automatic Control}, 4(3):110--110, 1959.

\bibitem[Kar15]{karlgaard2015nonlinear}
Christopher~D Karlgaard.
\newblock Nonlinear regression huber--kalman filtering and fixed-interval
  smoothing.
\newblock {\em Journal of guidance, control, and dynamics}, 38(2):322--330,
  2015.

\bibitem[KB61]{kalman1961new}
Rudolf~E. K{\'a}lm{\'a}n and Richard~S. Bucy.
\newblock New results in linear filtering and prediction theory.
\newblock {\em Journal of Basic Engineering}, 83:95--108, 1961.

\bibitem[KKM18]{klivans2018efficient}
Adam Klivans, Pravesh~K Kothari, and Raghu Meka.
\newblock Efficient algorithms for outlier-robust regression.
\newblock In {\em Conference On Learning Theory}, pages 1420--1430, 2018.

\bibitem[KSH00]{kailath2000linear}
T.~Kailath, A.H. Sayed, and B.~Hassibi.
\newblock {\em Linear Estimation}.
\newblock Prentice-Hall information and system sciences series. Prentice Hall,
  2000.

\bibitem[KSS18]{kothari2018robust}
Pravesh~K Kothari, Jacob Steinhardt, and David Steurer.
\newblock Robust moment estimation and improved clustering via sum of squares.
\newblock In {\em Proceedings of the 50th Annual ACM SIGACT Symposium on Theory
  of Computing}, pages 1035--1046. ACM, 2018.

\bibitem[Las01]{Lasserre01}
Jean~B. Lasserre.
\newblock {\em New Positive Semidefinite Relaxations for Nonconvex Quadratic
  Programs}, pages 319--331.
\newblock Springer US, Boston, MA, 2001.

\bibitem[LM21]{liu2021settling}
Allen Liu and Ankur Moitra.
\newblock Settling the robust learnability of mixtures of gaussians.
\newblock In {\em Proceedings of the 53rd Annual ACM SIGACT Symposium on Theory
  of Computing}, pages 518--531, 2021.

\bibitem[LRV16]{lai2016agnostic}
Kevin~A Lai, Anup~B Rao, and Santosh Vempala.
\newblock Agnostic estimation of mean and covariance.
\newblock In {\em 2016 IEEE 57th Annual Symposium on Foundations of Computer
  Science (FOCS)}, pages 665--674. IEEE, 2016.

\bibitem[MA80]{moore1980coping}
JB~Moore and Brian~DO Anderson.
\newblock Coping with singular transition matrices in estimation and control
  stability theory.
\newblock {\em International Journal of Control}, 31(3):571--586, 1980.

\bibitem[MB10]{mattingley2010real}
John Mattingley and Stephen Boyd.
\newblock Real-time convex optimization in signal processing.
\newblock {\em IEEE Signal processing magazine}, 27(3):50--61, 2010.

\bibitem[Men18]{mendelson2018learning}
Shahar Mendelson.
\newblock Learning without concentration for general loss functions.
\newblock {\em Probability Theory and Related Fields}, 171(1):459--502, 2018.

\bibitem[MS08]{mo2008characterization}
Yilin Mo and Bruno Sinopoli.
\newblock A characterization of the critical value for kalman filtering with
  intermittent observations.
\newblock In {\em 2008 47th IEEE Conference on Decision and Control}, pages
  2692--2697. IEEE, 2008.

\bibitem[MS11]{mo2011kalman}
Yilin Mo and Bruno Sinopoli.
\newblock Kalman filtering with intermittent observations: Tail distribution
  and critical value.
\newblock {\em IEEE Transactions on Automatic Control}, 57(3):677--689, 2011.

\bibitem[MW72]{muller1972analysis}
PC~M{\"u}ller and HI~Weber.
\newblock Analysis and optimization of certain qualities of controllability and
  observability for linear dynamical systems.
\newblock {\em Automatica}, 8(3):237--246, 1972.

\bibitem[Nes00]{Nesterov00}
Yurii Nesterov.
\newblock {\em Squared Functional Systems and Optimization Problems}, pages
  405--440.
\newblock Springer US, Boston, MA, 2000.

\bibitem[O'D14]{o2014analysis}
Ryan O'Donnell.
\newblock {\em Analysis of boolean functions}.
\newblock Cambridge University Press, 2014.

\bibitem[OZ13]{o2013approximability}
Ryan O'Donnell and Yuan Zhou.
\newblock Approximability and proof complexity.
\newblock In {\em Proceedings of the twenty-fourth annual ACM-SIAM symposium on
  Discrete algorithms}, pages 1537--1556. SIAM, 2013.

\bibitem[Par00]{parrilo2000structured}
Pablo~A Parrilo.
\newblock {\em Structured semidefinite programs and semialgebraic geometry
  methods in robustness and optimization}.
\newblock PhD thesis, California Institute of Technology, 2000.

\bibitem[Poo80]{poor1980robust}
H~Poor.
\newblock On robust wiener filtering.
\newblock {\em IEEE Transactions on Automatic Control}, 25(3):531--536, 1980.

\bibitem[PPSP05]{prajna2005sostools}
Stephen Prajna, Antonis Papachristodoulou, Peter Seiler, and Pablo~A Parrilo.
\newblock Sostools and its control applications.
\newblock In {\em Positive polynomials in control}, pages 273--292. Springer,
  2005.

\bibitem[R{\"O}G13]{roth2013student}
Michael Roth, Emre {\"O}zkan, and Fredrik Gustafsson.
\newblock A student's t filter for heavy tailed process and measurement noise.
\newblock In {\em 2013 IEEE International Conference on Acoustics, Speech and
  Signal Processing}, pages 5770--5774. IEEE, 2013.

\bibitem[Sch89]{schick1989robust}
Irvin~C Schick.
\newblock {\em Robust recursive estimation of the state of a discrete-time
  stochastic linear dynamic system in the presence of heavy-tailed observation
  noise}.
\newblock PhD thesis, Massachusetts Institute of Technology, 1989.

\bibitem[Sho87]{Shor87}
N.Z. Shor.
\newblock Quadratic optimization problems.
\newblock {\em Soviet Journal of Computer and Systems Sciences}, 25, 11 1987.

\bibitem[SI01]{sornette2001kalman}
Didier Sornette and Kayo Ide.
\newblock The kalman--l{\'e}vy filter.
\newblock {\em Physica D: Nonlinear Phenomena}, 151(2-4):142--174, 2001.

\bibitem[SM94]{schick1994robust}
Irvin~C Schick and Sanjoy~K Mitter.
\newblock Robust recursive estimation in the presence of heavy-tailed
  observation noise.
\newblock {\em The Annals of Statistics}, pages 1045--1080, 1994.

\bibitem[SMT{\etalchar{+}}18]{simchowitz2018learning}
Max Simchowitz, Horia Mania, Stephen Tu, Michael~I Jordan, and Benjamin Recht.
\newblock Learning without mixing: Towards a sharp analysis of linear system
  identification.
\newblock In {\em Conference On Learning Theory}, pages 439--473. PMLR, 2018.

\bibitem[SSF{\etalchar{+}}04]{sinopoli2004kalman}
Bruno Sinopoli, Luca Schenato, Massimo Franceschetti, Kameshwar Poolla,
  Michael~I Jordan, and Shankar~S Sastry.
\newblock Kalman filtering with intermittent observations.
\newblock {\em IEEE transactions on Automatic Control}, 49(9):1453--1464, 2004.

\bibitem[SSH20]{simchowitz2020improper}
Max Simchowitz, Karan Singh, and Elad Hazan.
\newblock Improper learning for non-stochastic control.
\newblock In {\em Conference on Learning Theory}, pages 3320--3436. PMLR, 2020.

\bibitem[TP20]{tsiamis2020online}
Anastasios Tsiamis and George Pappas.
\newblock Online learning of the kalman filter with logarithmic regret.
\newblock {\em arXiv preprint arXiv:2002.05141}, 2020.

\bibitem[Tro12]{tropp}
Joel~A Tropp.
\newblock User-friendly tail bounds for sums of random matrices.
\newblock {\em Foundations of computational mathematics}, 12(4):389--434, 2012.

\bibitem[Ver18]{vershynin2018high}
Roman Vershynin.
\newblock {\em High-dimensional probability: An introduction with applications
  in data science}, volume~47.
\newblock Cambridge university press, 2018.

\end{thebibliography}
